\documentclass{article} 
\usepackage{PRIMEarxiv}

\usepackage{amsmath}
\usepackage{amsthm}
\usepackage{amsopn}
\usepackage{amsfonts}
\usepackage{amssymb}
\usepackage{mathrsfs}
\usepackage{bbm}
\usepackage{enumitem}
\usepackage{hyperref}
\usepackage{url}
\usepackage{graphicx}
\usepackage{subcaption}
\usepackage{algorithm}
\usepackage{algorithmic}
\usepackage{natbib}
\usepackage{booktabs}
\usepackage{cleveref}
\usepackage{todonotes}
\usepackage{pifont}
\usepackage{dsfont}

\usepackage{localmacros}

\title{Riemannian Archetypal Analysis:\\ Interpretable non-linear data analysis \\ on deformed star distributions}
\author{Willem Diepeveen \\
Department of Mathematics\\
University of California, Los Angeles\\
Los Angeles, CA 90095, USA \\
\texttt{wdiepeveen@math.ucla.edu} \\
\And
Deanna Needell \\
Department of Mathematics\\
University of California, Los Angeles\\
Los Angeles, CA 90095, USA \\
\texttt{deanna@math.ucla.edu}
}

\allowdisplaybreaks

\begin{document}

\maketitle

\begin{abstract}
    Classical archetypal analysis is appealing for its interpretability, but its linear geometry can limit performance on data with strongly non-linear structure; at the same time, existing neural extensions improve flexibility while often weakening the geometric meaning of archetypes and interpolations. In this work, we develop a Riemannian version of archetypal analysis based on data-driven pullback geometry for real-valued data, with the goal of combining the interpretability of classical archetypal analysis with the expressive power of modern non-linear models. We introduce a class of deformed star distributions together with associated pullback Riemannian geometry to provide a statistical interpretation of the resulting manifold mappings, define the Riemannian archetypal mapping (RAM) as a projection onto the manifold of geodesically convex combinations of archetypes, and propose a practical optimization scheme based on convex relaxation followed by non-convex refinement. We further propose a learning scheme that yields reasonable, albeit generally suboptimal, deformed star distributions from data. Experiments on synthetic examples and MNIST show that the resulting framework produces meaningful geodesics, useful denoising projections, and geometry-aware classifications, while also clarifying where current optimization limitations remain.
\end{abstract}

\blfootnote{Our code is available at \href{https://github.com/wdiepeveen/Riemannian-Archetypal-Analysis}{https://github.com/wdiepeveen/Riemannian-Archetypal-Analysis}.}

\section{Introduction}

Extracting interpretable features from data is a central task in exploratory analysis, understanding constituent components, and enabling downstream classification and decision-making. The usefulness of any such method is largely determined by how well its underlying geometric model fits the data. For example, principal component analysis (PCA) \cite{pearson1901liii} assumes that the data lie near a low-dimensional linear subspace fitted in a least-squares sense; when this assumption is violated, the resulting components can be difficult to interpret and of limited practical value. In response to this limitation, classical but still widely used methods have adopted richer geometric models: independent component analysis (ICA) \cite{comon1994independent} models data as a linear mixture of statistically independent components, non-negative matrix factorization (NMF) \cite{paatero1994positive} represents data as lying in the conical hull of non-negative basis vectors in the positive orthant, and archetypal analysis (AA) \cite{cutler1994archetypal} describes data as residing in a polytope whose vertices are themselves data points. For different data types and end goals, particular combinations and refinements of these basic ideas are often preferred \cite{abdolali2021simplex,ding2008convex,kleverov2026non,li2008minimum,lin2018maximum,miao2007endmember,nascimento2005vertex}.

In applications where the goal is to identify representative data points (archetypes), to quantify how other observations relate to them, and to leverage these relationships in classification or other downstream tasks, AA enjoys several advantages. Highlighting its advantages over ICA and NMF, AA yields interpretable factors in the form of extremal data points (unlike ICA, whose statistically independent components are linear directions that need not correspond to specific representative observations) and applies to general real-valued data without non-negativity constraints (unlike NMF, which requires non-negative data and typically yields basis elements that are not actual observations). For a comprehensive overview of archetypal analysis and its variants, see \cite{alcacer2025survey}.

\paragraph{Towards non-linear archetypal analysis}

However, AA is inherently linear: it operates in an ambient Euclidean space and encodes geometric assumptions via linear combinations or convex mixtures of data points. In many modern applications, data concentrate near highly non-linear manifolds \cite{fefferman2016testing,whiteley2022statistical}, and naive linear or conical models fail to capture the true geometry and to provide meaningful notions of distance or interpolation between archetypes and observations.

Motivated by this limitation, several non-linear, neural network-based extensions of AA have been proposed. Models such as AAnet \cite{van2019finding} replace the linear archetypal map by a learned encoder-decoder architecture while retaining an archetypal structure in the latent space, and subsequent adaptations such as Deep AA \cite{keller2021learning}, which ground archetypal representations in probabilistic generative models -- variational autoencoders (VAEs) \cite{kingma2013auto} -- rather than standard autoencoders, further extend this idea. Despite the empirical success of these models in applications such as single-cell analysis \cite{venkat2025aanet} and related extensions \cite{tasissa2023k,wieser2025revisiting}, current non-linear AA variants still treat the latent space as an ad hoc Euclidean space, without guarantees that distances or interpolations there reflect meaningful geometric relationships between archetypes and observations -- even though such interpretability is a primary motivation for these methods.

This lack of rigorous geometric interpretability is a common critique of non-linear dimension reduction methods \cite{chari2023specious}, and it has motivated a growing body of work on endowing (variational) autoencoder latent spaces with meaningful geometric structure \cite{bhasker2025uncovering,hartwig2026geodesic,kohli2021ldle,lee2025geometry,moon2019visualizing,psenka2024representation,vigouroux2026discovering}. However, to the best of our knowledge, these efforts have so far not been extended to the non-linear AA setting.

\paragraph{Towards Riemannian archetypal analysis}

Our goal in this work is to develop provably geometrically meaningful non-linear AA methodologies by leveraging Riemannian geometry, rather than generalizing the above-mentioned (still heuristic) approaches to endowing the non-linear AA latent space with meaningful geometric structure. Importantly, this should not be confused with simply choosing a non-Euclidean latent manifold \cite{cho2023hyperbolic,davidson2018hyperspherical}, which typically inherits the same interpretability issues as standard autoencoders when the latent geometry is not explicitly tied to the data distribution and archetypal structure.

At a high level, the aim of a Riemannian reformulation of a machine learning method is to learn a Riemannian structure on the ambient space such that the data form (or are well-approximated by) a low-dimensional, totally geodesic submanifold. The machine learning task is then phrased as an optimization problem over this data manifold, and solved using specialized Riemannian optimization techniques; see \Cref{fig:river-ram-schematic} for an illustration in the context of archetypal analysis.

\begin{figure}[h!]
    \centering
    \begin{subfigure}{0.48\linewidth}
        \centering
        \includegraphics[width=\linewidth]{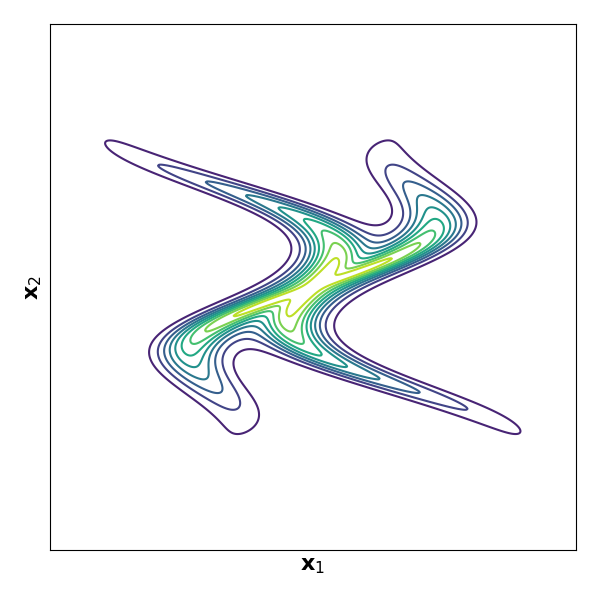}
        \caption{A deformed star distribution visualized by its level sets.}
        \label{fig:river-ram-schematic-a}
    \end{subfigure}
    \hfill
    \begin{subfigure}{0.48\linewidth}
        \centering
        \includegraphics[width=\linewidth]{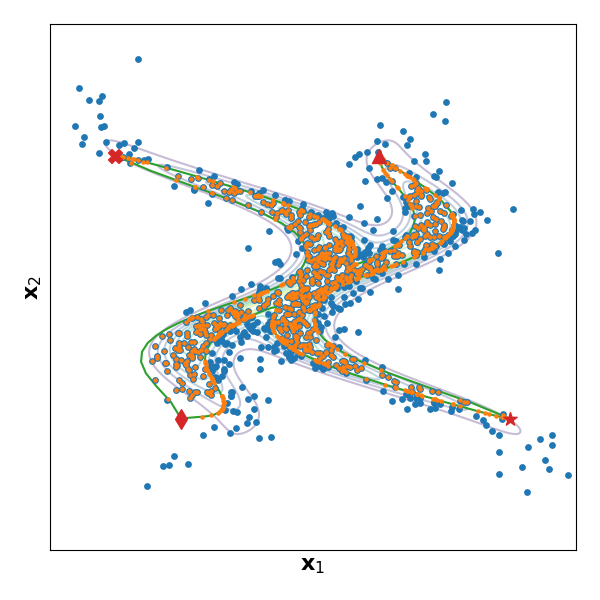}
        \caption{An illustration of Riemannian archetypal analysis.}
        \label{fig:river-ram-schematic-b}
    \end{subfigure}
    \caption{Data from the deformed star distribution (blue) are projected in the $\ell^2$-sense (orange) onto the manifold bounded by geodesics (green) connecting the four archetypes (red) via the Riemannian archetypal mapping (RAM). Each projected point admits a geodesic barycentric representation whose weight vector is sparse for points projected onto a boundary or corner, and points that already lie in the manifold approximation of the data set remain unchanged.}
    \label{fig:river-ram-schematic}
\end{figure}


Learning such Riemannian structures and efficiently evaluating manifold maps (geodesics, exponential, and logarithmic mappings) from data has been an active research area \cite{arvanitidis2016locally,diepeveen2024pulling,hauberg2012geometric,peltonen2004improved,Scarvelis2023,sorrenson2025learning,sun2024geometryaware}; see also \cite{gruffaz2025riemannian} for a recent overview. Only recently have both ingredients been realized in a unified, data-driven way \cite{diepeveen2025scorebased}, in which volume-preserving normalizing flows are trained to “unroll’’ the data manifold, yielding a pullback Riemannian metric with closed-form manifold mappings that provably follow high-likelihood regions of the underlying deformed Gaussian density. On top of such a geometry, Riemannian neural networks become practical. In particular, Riemannian autoencoders (RAEs) \cite{diepeveen2024pulling}, which can be viewed as non-linear generalizations of PCA, implement an $\ell^2$-projection onto a low-dimensional submanifold, where the projection can often be approximated in closed form for sufficiently regular diffeomorphisms \cite{diepeveen2024pulling} or efficiently solved via iso-Riemannian optimization otherwise \cite{diepeveen2025isoriemanopt}.

However, just as classical PCA is well suited for (approximately) Gaussian data in Euclidean space, RAEs and more generally the framework proposed in \cite{diepeveen2025scorebased} are primarily tailored to settings where the underlying geometric model is a deformed Gaussian distribution obtained via a volume-preserving diffeomorphism (or, more generally, a diffeomorphism with constant Jacobian determinant). This is restrictive: not every data distribution can be pushed forward to a Gaussian under such diffeomorphisms, and in particular deformed star distributions -- the natural geometric model underlying archetypal analysis -- do not fit this paradigm. In practice, the existing Riemannian geometry workflow has therefore only reached non-linear PCA-type methods and remains unable to capture the polytope- and star-like geometries (even in linear settings) that AA is designed for.


\subsection{Contributions}

This gap motivates the development of Riemannian archetypal analysis, which is the focus of this work. Our main contributions are as follows:

\begin{itemize}
\item \textbf{Pullback geometries for deformed star distributions.} We introduce a broad class of deformed star distributions and define a corresponding family of pullback Riemannian geometries with diffeomorphisms that do not have a constant Jacobian determinant. We show that pullback geodesics remain within high-likelihood regions of the distribution and that, by choosing a non-trivial pullback metric within the proposed family, we obtain more stable and meaningful geodesics.
\item \textbf{Riemannian archetypal mappings (RAMs).} Given a pullback structure induced by a deformed star distribution and a collection of archetypes, we define the Riemannian archetype mapping (RAM) as the projection of data points onto the manifold of geodesically convex combinations of the archetypes. We reformulate this projection as a non-convex constrained optimization problem and propose an algorithm that initializes from a convex relaxation to obtain reliable solutions, whose distances to each archetype can be directly used for classification and other downstream tasks.
\item \textbf{Learning archetypes and star distributions.} We explain why standard negative log-likelihood training is ill-suited for reliably recovering deformed star distributions, and instead propose a constructive learning scheme combined with classical normalizing flow training. This scheme is motivated by the observation that multiscale normalizing flows with constant Jacobian determinant naturally induce star-shaped latent distributions, which we exploit to learn both archetypes and their associated deformed star geometry.
\end{itemize}

\subsection{Related work and broader impact}

\paragraph{Feature extraction on abstract manifolds}

RAEs are built on ideas originally developed in the abstract manifold setting, both for constructing low-dimensional data manifolds \cite{fletcher2004principal} and for understanding their geometric behavior \cite{diepeveen2025curvature}. Although RAEs ultimately use a (non-intrinsic) $\ell^2$-projection in the ambient space, this distinction from the fully intrinsic setting can to a large extent be addressed by both learning the manifold and performing data analysis under a connection different from the Levi-Civita connection called the iso-connection \cite{diepeveen2025manifold}, chosen so that geodesics preserve their shape while acquiring constant $\ell^2$-speed and inducing notions of distance that align better with the ambient Euclidean geometry.

From this perspective, it would be natural to first look for a generalization of AA in the abstract manifold setting and then specialize to data-driven geometries. However, to the best of our knowledge there is currently no abstract manifold analogue of AA, in contrast to ICA and NMF, for which manifold-based generalizations have been proposed \cite{ho2013nonlinear,chew2025curvature}. Having said that, these manifold ICA and NMF methods have not yet been adapted to the mixed (intrinsic-extrinsic) metric setting that arises in machine learning applications. All this suggests that looking for AA outside the purely intrinsic setting could indicate how to construct it for the intrinsic setting and provide a roadmap for extending ICA, NMF, and a broad family of manifold-based methods \cite{miolane2020geomstats} to a substantially wider range of applications.


\paragraph{Signal processing on learned data manifolds}

For RAEs induced by diffeomorphisms that are not highly regular (in particular, that do not have constant Jacobian determinant as in our deformed star setting), one generally has to resort to optimization rather than closed-form approximations of projections. Classical Riemannian optimization methods \cite{absil2008optimization,boumal2023introduction} rely on geodesic convexity to obtain efficient algorithms, but this assumption is typically too strong in the mixed-metric setting. One can still formally apply Riemannian gradient descent; however, for pullback geometries this is effectively equivalent to performing gradient descent in a chart.

Such chart-based optimization can be made to work \cite{diepeveen2026riemannian} and underlies several non-pullback, latent-space or chart-based approaches \cite{alberti2024manifold,daras2021intermediate,hand2018phase,hand2018global,lei2019inverting,robinett2025manifold,shustin2022manifold}, but strong convexity and Lipschitz properties of the objective often deteriorate under the chart, forcing very small step sizes in practice. Very recently, geodesic convexity on the data manifold and Euclidean convexity of the objective were combined more effectively \cite{diepeveen2025isoriemanopt} by switching the iso-connection \cite{diepeveen2025manifold}, in a way analogous to the fix used for mixed-metric dimension reduction mentioned before. This leads to an iso-Riemannian optimization framework that identifies the “right” vector field for descent and permits larger step sizes, yielding faster convergence than standard chart-based methods.

At present, however, only the iso-Riemannian analogue of gradient descent has been systematically developed \cite{diepeveen2025isoriemanopt}, and there is no general iso-Riemannian framework for handling constraints -- precisely what is needed to accelerate evaluations of RAMs in our setting, where the underlying diffeomorphisms are far from having constant Jacobian determinant. Developing constrained optimization under a star-shaped geometry thus not only enables our Riemannian archetypal analysis, but also points to a natural next step for the iso-Riemannian optimization literature.

\paragraph{Star geometry in machine learning and data science}

More broadly, star-shaped data distributions and their deformations have recently begun to play a prominent role both in explaining the success of modern machine learning methods and in suggesting more principled design directions. For instance, non-convex regularizers based on star geometry can be shown to be optimal for certain data models \cite{leong2025optimal}, motivating the development of non-convex regularization schemes in inverse problems and learning \cite{hertrich2025learning,leong2024star}. Related geometric star models also underlie recent analyses of diffusion models, where such structures lead to efficient learning and recovery guarantees \cite{wang2025diffusion,leong2025recovery}, and they appear in white-box representations of deep networks: deformed star-geometry-based models for architectures such as ReduNet \cite{chan2022redunet} and for transformers \cite{yu2023white} reveal explicit links between star-shaped latent structure and classification or sequence modeling performance. All of this suggests that a deformed star-based Riemannian model may simultaneously sharpen the statistical understanding of machine learning methods more broadly and provide a natural geometric foundation for Riemannian neural networks in realistic data regimes.


\subsection{Outline}

\Cref{sec:prelims} recalls the Riemannian and pullback geometry underlying our constructions and highlights limitations of current pullback manifold learning, further motivating the need for our approach. We then develop a Riemannian formulation of archetypal analysis and its data-driven instantiation in a step-by-step manner. First, \Cref{sec:star-geometry} introduces deformed star distributions as the statistical backbone of the assumed data geometry, together with their pullback structure, and shows how these induce stable, high-likelihood geodesics adapted to archetypal structure, while iso-Riemannian geometry restores an interpretable notion of time. Next, \Cref{sec:rams} derives the Riemannian archetypal mapping (RAM) from these statistical and geometric assumptions, develops relaxed and refined optimization schemes for its evaluation, and explains how iso-corrected weights can be used for classification. Building on this, \Cref{sec:learning-stars} proposes a practical three-step learning procedure that produces reasonable, albeit suboptimal, deformed star models from data, thereby making the tools of \Cref{sec:star-geometry,sec:rams} accessible in practice. To assess how well the star-shaped assumption fits real data and to expose practical limitations, \Cref{sec:numerics} reports experiments on MNIST, focusing on interpolation, denoising, and classification under the learned geometry. Finally, \Cref{sec:conclusions} summarizes the main findings and outlines directions for future work.

\section{Preliminaries}
\label{sec:prelims}

For the purposes of this work we will need the following notions and results, which we present in basic notations from differential and Riemannian geometry, see \cite{boothby2003introduction,carmo1992riemannian,lee2013smooth,sakai1996riemannian} for details.

\paragraph{Smooth manifolds and tangent spaces} 
Let $\manifold$ be a \emph{$\dimInd$-dimensional smooth manifold}, i.e., a topological manifold of dimension $\dimInd$ equipped with a \emph{maximal smooth atlas}, meaning a collection of charts whose transition functions are all smooth, making the manifold locally diffeomorphic to $\Real^\dimInd$. We write $C^\infty(\manifold)$ for the space of smooth functions over $\manifold$. The \emph{tangent space} at $\mPoint \in \manifold$, which is defined as the space of all \emph{derivations} at $\mPoint$, is denoted by $\tangent_\mPoint \manifold$ and for \emph{tangent vectors} we write $\tangentVector_\mPoint \in \tangent_\mPoint \manifold$. For the \emph{tangent bundle} we write $\tangent\manifold$ and smooth vector fields, which are defined as \emph{smooth sections} of the tangent bundle, are written as $\vectorfield(\manifold) \subset \tangent\manifold$.

\paragraph{Riemannian manifolds} 
A smooth manifold $\manifold$ becomes a \emph{Riemannian manifold} if it is equipped with a smoothly varying \emph{metric tensor field} $(\cdot, \cdot) : \vectorfield(\manifold) \times \vectorfield(\manifold) \to C^\infty(\manifold)$. This tensor field induces a \emph{(Riemannian) metric} $\distance_{\manifold} : \manifold\times\manifold\to\Real$. The metric tensor can also be used to construct a unique affine connection, the \emph{Levi-Civita connection}, that is denoted by $\nabla_{(\,\cdot\,)}(\,\cdot\,) : \vectorfield(\manifold) \times \vectorfield(\manifold) \to \vectorfield(\manifold)$. 
This connection is in turn the cornerstone of a myriad of manifold mappings.

One is the notion of a \emph{geodesic}, which for two points $\mPoint,\mPointB \in \manifold$ is defined as a curve $\geodesic_{\mPoint,\mPointB} : [0,1] \to \manifold$ with minimal length that connects $\mPoint$ with $\mPointB$ -- that is, if such a curve exists. To ensure existence, we often consider \emph{(geodesically) convex} subsets, i.e., sets $\mathcal{D}\subset \manifold$ such that $\geodesic_{\mPoint,\mPointB}\subset \mathcal{D}$ for any $\mPoint,\mPointB\in \mathcal{D}$. In addition, when geodesics are also unique on $\mathcal{D}$, we call $\mathcal{D}$ \emph{strongly (geodesically) convex}. Another closely related notion to geodesics is the curve $t \mapsto \geodesic_{\mPoint,\tangentVector_\mPoint}(t)$  for a geodesic starting from $\mPoint\in\manifold$ with velocity $\dot{\geodesic}_{\mPoint,\tangentVector_\mPoint} (0) = \tangentVector_\mPoint \in \tangent_\mPoint\manifold$. This can be used to define the \emph{exponential map} $\exp_\mPoint : \mathcal{D}_\mPoint \to \manifold$ at $\mPoint$ as \(\exp_\mPoint(\tangentVector_\mPoint) := \geodesic_{\mPoint,\tangentVector_\mPoint}(1),\) where \(\mathcal{D}_\mPoint \subset \tangent_\mPoint \manifold\) is the set on which \(\geodesic_{\mPoint,\tangentVector_\mPoint}(1)\) is defined. The manifold $\manifold$ is said to be \emph{(geodesically) complete} whenever $\mathcal{D}_{\mPoint}=\tangent_{\mPoint} \manifold$ for all $\mPoint \in \manifold$. Furthermore, the \emph{logarithmic map} $\log_\mPoint : \exp_\mPoint(\mathcal{D}'_\mPoint ) \to \mathcal{D}'_\mPoint$ at $\mPoint$ is defined as the inverse of $\exp_\mPoint$, so it is well-defined on  $\mathcal{D}'_{\mPoint} \subset \mathcal{D}_{\mPoint}$ where $\exp_\mPoint$ is a diffeomorphism. Moreover, for \emph{parallel transport} $\partransport_{\mPointB \leftarrow \mPoint}: \tangent_\mPoint \manifold \to \tangent_\mPointB \manifold$ of a vector $\tangentVector_\mPoint \in \tangent_\mPoint \manifold$ along a geodesic from $\mPoint$ to
$\mPointB$ we write $\partransport_{\mPointB \leftarrow \mPoint} \tangentVector_\mPoint$. 


\paragraph{Pullback manifolds} 
If $(\manifold, (\cdot,\cdot))$ is a $\dimInd$-dimensional Riemannian manifold, $\manifoldB$ is a $\dimInd$-dimensional smooth manifold and $\diffeo:\manifoldB \to \manifold$ is a diffeomorphism, the \emph{pullback metric}
\begin{equation}
    (\tangentVector, \tangentVectorB)_\mPoint^\diffeo := (D_{\mPoint}\diffeo[\tangentVector_{\mPoint}], D_{\mPoint}\diffeo[\tangentVectorB_{\mPoint}])_{\diffeo(\mPoint)}, \quad \mPoint \in \manifoldB, \tangentVector, \tangentVectorB \in \vectorfield(\manifoldB)
    \label{eq:pull-back-metric}
\end{equation}
where $D_{\mPoint}\diffeo: \tangent_\mPoint \manifoldB \to \tangent_{\diffeo(\mPoint)} \manifold$ denotes the differential of $\diffeo$,
defines a Riemannian structure on $\manifoldB$, which we denote by $(\manifoldB, (\cdot,\cdot)^\diffeo)$. 
Pullback mappings are denoted similarly to (\ref{eq:pull-back-metric}) with the diffeomorphism $\diffeo$ as a superscript, i.e., we write $\distance^\diffeo_{\manifoldB}(\mPoint, \mPointB)$, $\geodesic^\diffeo_{\mPoint, \mPointB}$, $\exp^\diffeo_\mPoint (\tangentVector_\mPoint)$, $\log^\diffeo_{\mPoint} \mPointB$, and $\partransport^{\diffeo}_{\mPointB \leftarrow \mPoint}$ for $\mPoint,\mPointB \in \manifoldB$ and $\tangentVector_\mPoint \in \tangent_\mPoint \manifoldB$. Pullback metrics literally pull back all geometric information from the Riemannian structure on $\manifold$. 
In particular, closed-form manifold mappings on $(\manifold, (\cdot,\cdot))$ yield under mild assumptions closed-form manifold mappings on $(\manifoldB, (\cdot,\cdot)^\diffeo)$. 

\paragraph{Data-driven Euclidean pullback manifolds}

Notably, for Euclidean pullback manifolds $(\Real^\dimInd,(\cdot,\cdot)^\diffeo)$ generated by a diffeomorphism $\diffeo:\Real^\dimInd\to \Real^\dimInd$ pulling back the standard Euclidean structure $(\Real^\dimInd, (\cdot, \cdot)_2)$ -- which is how scalable data-driven Riemannian geometry is constructed for high-dimensional data \cite{diepeveen2025scorebased,diepeveen2025manifold} --, we have \cite[Prop~2.1]{diepeveen2024pulling}
\begin{align}
    \distance_{\Real^{\dimInd}}^{\diffeo}(\Vector, \VectorB) &= \|\diffeo(\Vector) - \diffeo(\VectorB)\|_2,
    \label{eq:thm-distance-remetrized}\\
    \geodesic^{\diffeo}_{\Vector, \VectorB}(t) &= \diffeo^{-1}((1 - t)\diffeo(\Vector) + t \diffeo(\VectorB)),
    \label{eq:thm-geodesic-remetrized}\\
    \exp^{\diffeo}_\Vector (\tangentVector_\Vector) &= \diffeo^{-1}(\diffeo(\Vector) + D_{\Vector} \diffeo[\tangentVector_\Vector]),
    \label{eq:thm-exp-remetrized}\\
    \log^{\diffeo}_{\Vector} (\VectorB) &= D_{\diffeo(\Vector)}\diffeo^{-1}[\diffeo(\VectorB) - \diffeo(\Vector)],
    \label{eq:thm-log-remetrized}\\
    \partransport^{\diffeo}_{\VectorB \leftarrow \Vector} \tangentVector_\Vector &= D_{\diffeo(\VectorB)}\diffeo^{-1}[D_{\Vector} \diffeo[\tangentVector_\Vector]],
    \label{eq:thm-exp-remetrized}
\end{align}
where $\Vector, \VectorB\in \Real^\dimInd$ and $\tangentVector_\Vector \in \tangent_\Vector \Real^\dimInd \cong \Real^\dimInd$, and have \cite[Prop~3.7]{diepeveen2024pulling}
\begin{equation}
    \operatorname{argmin}_{\Vector\in \Real^\dimInd} \sum_{\sumIndA=1}^\dataNum \distance^\diffeo_{\Real^\dimInd}(\Vector, \Vector^\sumIndA)^2 = \diffeo^{-1} (\frac{1}{\dataNum} \sum_{\sumIndA=1}^\dataNum \diffeo(\Vector^\sumIndA)),
    \label{eq:thm-bary-remetrized}
\end{equation}
for the Riemannian barycentre \cite{karcher1977riemannian}, where $\Vector^1, \ldots, \Vector^\dataNum\in \Real^\dimInd$. 

In the context of a data-driven pullback structure, the manifold mappings above gain a practical interpretation. A well-trained $\diffeo$ essentially flattens out the data space, i.e., it maps a data set -- residing close to a non-linear data manifold -- into the vicinity of a (low-dimensional) linear subspace of $\Real^\dimInd$. Manifold mappings are essentially computed using Euclidean rules applied to points and tangent vectors mapped into this linear subspace by $\diffeo$ and then mapped back to the original data domain using $\diffeo^{-1}$. As a result, geodesics between two points will always move through regions with large amounts of data -- or probabilistically speaking through regions with high likelihood. For a more detailed discussion and the manifold mapping for the general pullback setting, we refer the reader to \cite{diepeveen2024pulling}.

\paragraph{Learning Euclidean pullback manifolds}
To learn such a diffeomorphism $\diffeo:\Real^\dimInd\to\Real^\dimInd$ that generates geodesics that interpolate through regions of high likelihood, normalizing flow training has shown to be a scalable approach \cite{diepeveen2025scorebased,diepeveen2025manifold}. Following \cite{diepeveen2025manifold}, this boils down to minimizing the negative log likelihood loss
\begin{equation}
    \mathcal{L}(\networkParams):=\mathbb{E}_{\mathbf{\stoVector} \sim \density_{\text{data }}}\left[-\log \density_{\diffeo_\networkParams}(\stoVector)\right],
    \label{eq:loss-nf}
\end{equation}
where $\density_{\diffeo_\networkParams}: \Real^\dimInd \to \Real$ is given by
\begin{equation}
    \density_{\diffeo_\networkParams}(\Vector):=\frac{1}{\sqrt{(2 \pi)^\dimInd}} e^{-\frac{1}{2}\|\diffeo_\networkParams(\Vector)\|_2^2}|\det(D_{\Vector} \diffeo_\networkParams)| ,
    \label{eq:deformed-gaussian-density}
\end{equation}
and where $\diffeo_\networkParams:\Real^\dimInd\to\Real^\dimInd$ is an invertible neural network with parameters $\networkParams$ such that the mapping $\Vector\mapsto |\det(D_{\Vector} \diffeo_\networkParams)|$ is constant. In practice, the latter constraint is typically guaranteed by using additive coupling layers \cite{dinh2014nice} combined with invertible linear channel mixing and normalization strategies \cite{kingma2018glow}.

To get intuition as to why this approach yields a suitable diffeomorphism and pullback geometry by extension, we first note that the function $t \mapsto -\log \bigl(\density_{\diffeo_\networkParams}(\geodesic^{\diffeo_{\networkParams}}_{\Vector, \VectorB}(t))\bigr)$ is strongly convex for any combination of end points $\Vector,\VectorB\in \Real^\dimInd$ and network parametrization $\networkParams$ (\cite[Thm.~3.3]{diepeveen2025scorebased}). Then, if $\density_{\text{data}}$ is feasible, i.e., there exists some $\networkParams^*$ such that $\density_{\diffeo_{\networkParams^*}} = \density_{\text{data}}$, minimizing (\ref{eq:loss-nf}) will find this $\networkParams^*$. In other words, if we have $\density_{\diffeo_{\networkParams^*}}  = \density_{\text{data}}$ this means that geodesics $\geodesic^{\diffeo_{\networkParams^*}}_{\Vector, \VectorB}$ between data points move through regions with higher likelihood than the end points, which is exactly what we set out to do. 

\begin{remark}
\label{rem:deformed-gaussian-shortcomings}
    In practice, we cannot expect the true data density $\density_{\text{data}}$ to belong to the feasible class described above. This mismatch is a key motivation for our work, which extends pullback learning beyond deformed Gaussian models to a richer family of deformed star distributions that more faithfully capture non-linear polytope-type geometry underlying archetypal analysis.
\end{remark}

\section{Riemannian geometry for deformed star distributions}
\label{sec:star-geometry}

Before discussing how to formulate Riemannian archetypal analysis and learn the associated Riemannian structures, we first address \Cref{rem:deformed-gaussian-shortcomings} and introduce a class of data distributions that provides a more suitable starting point for this endeavor and discuss several geometric considerations. In what follows, we formalize this class as a family of deformed star distributions. In the archetypal analysis setting, archetypes naturally concentrate near the “corners” of the deformed star. We then systematically study the associated pullback structures they induce and address interpretability issues of geodesics via the iso-Riemannian geometry generated by these pullback metrics. All proofs are deferred to \Cref{app:star-geometry}.

\subsection{Deformed star distributions}

First, we seek a class of distributions that generalizes \eqref{eq:deformed-gaussian-density}. In particular, we will consider densities $\density_{\diffeoA, \radial}:\Real^\dimInd\to\Real$ of the form

\begin{equation}
    \density_{\diffeoA, \radial} (\Vector):= \frac{1}{2^{\frac{\dimInd}{2}-1} \Gamma\left(\frac{\dimInd}{2}\right)} \frac{e^{-\frac{1}{2}\radial( \frac{\diffeoA(\Vector)}{\|\diffeoA(\Vector)\|_2} )^{-2}\|\diffeoA(\Vector)\|_2^2}}{\int_{\Sphere^{\dimInd-1}} \radial(\sVector)^\dimInd d \sigma(\sVector)} |\det D_\Vector \diffeoA|,
    \label{eq:starflow-density}
\end{equation}
where $\diffeoA:\Real^\dimInd\to\Real^\dimInd$ is a diffeomorphism with constant Jacobian determinant,  $\radial:\Sphere^{\dimInd-1}\to\Real_{>0}$ is a radial function whose role is to modulate the radial scaling in each direction (so that the ``star corners'' correspond to directions $\sVector$ with large values of $\radial(\sVector)$), and $\sigma$ is the standard surface measure on the unit sphere $\Sphere^{\dimInd-1} \subset \Real^\dimInd$, i.e. the $(\dimInd-1)$-dimensional Hausdorff measure restricted to $\Sphere^{\dimInd-1}$ -- so that $\sigma(\Sphere^{\dimInd-1}) = 2 \pi^{\frac{\dimInd}{2} } / \Gamma\left(\frac{\dimInd}{2}\right)$.

Densities of the form \eqref{eq:starflow-density} exhibit a deformed star-shaped structure, as depicted in \Cref{fig:river-ram-schematic-a}, and can be interpreted as the pushforward -- via $\diffeoA^{-1}$ -- of a latent star-shaped distribution whose level sets are completely determined by the radial function $\radial$. It is worth highlighting that \eqref{eq:starflow-density} reduces to a distribution of the form \eqref{eq:deformed-gaussian-density} when $\radial(\sVector) := 1$ for all $\sVector\in \Sphere^{\dimInd-1}$ and that for general $\radial$ this family of densities still defines valid probability distributions.

\begin{proposition}
\label{prop:deformed-star-prob-density}
Let $\diffeoA:\Real^\dimInd\to\Real^\dimInd$ be a smooth diffeomorphism and let $\radial:\Sphere^{\dimInd-1}\to\Real_{>0}$ be a radial function. 

Then, the function $\density_{\diffeoA,\radial}$ defined as in \eqref{eq:starflow-density} is a probability density on
$\mathbb{R}^\dimInd$, i.e., 
\begin{equation}
    \int_{\mathbb{R}^\dimInd} \density_{\diffeoA,\radial}(\Vector) \, d \Vector = 1.
\end{equation}
\end{proposition}


\subsection{Pullback geometry and stability} 
Assuming that the data distribution is of the form \eqref{eq:starflow-density}, our next goal is to construct a diffeomorphism $\diffeo : \Real^{\dimInd} \to \Real^{\dimInd}$ -- presumably depending on $\diffeoA$ and $\radial$ -- such that, for any choice of endpoints $\Vector, \VectorB \in \Real^{\dimInd}$, the mapping $t \mapsto -\log\bigl(\density_{\diffeoA,\radial}(\geodesic^{\diffeo}_{\Vector,\VectorB}(t))\bigr)$ is strongly convex. In direct analogy with the deformed Gaussian setting, this would imply that geodesics between data points pass through regions of higher likelihood than the endpoints, but now with respect to a more realistic data distribution than the deformed Gaussian model.

In contrast to the deformed Gaussian case, however, obtaining stable and meaningful geodesics is somewhat less straightforward here. Before turning to the construction in detail, we first describe the basic building blocks of the pullback geometry: in addition to the diffeomorphism $\diffeoA$, we will introduce two further mappings.



\begin{proposition}
\label{prop:rho-v-diffeos}
Let $\radial:\Sphere^{\dimInd-1}\to\Real_{>0}$ be a radial function and let $\monotone:\Real_{> 0}\to \Real_{> 0}$ be a strictly increasing function. 
    \begin{enumerate}[label=(\roman*)]
        \item The mapping $\diffeoB_\radial:\Real^\dimInd\to\Real^\dimInd$ defined as
        \begin{equation}
            \diffeoB_\rho (\Vector) :=  \left\{\begin{matrix}
        \rho\bigl(\frac{\Vector}{\|\Vector\|_2}\bigr)^{-1} \Vector  & \|\Vector\|_2 > 0 \\
        \mathbf{0} & \|\Vector\|_2 = 0 \\
        \end{matrix}\right. ,
        \label{eq:radial-homeo}
        \end{equation}
        is invertible and its inverse is given by
        \begin{equation}
            \diffeoB_\rho^{-1} (\VectorB) :=  \left\{\begin{matrix}
        \rho\bigl(\frac{\VectorB}{\|\VectorB\|_2}\bigr) \VectorB  & \|\VectorB\|_2 > 0 \\
        \mathbf{0} & \|\VectorB\|_2 = 0 \\
        \end{matrix}\right. .
        \label{eq:radial-homeo-inv}
        \end{equation}
        \item The mapping $\diffeoC_\monotone:\Real^\dimInd\to\Real^\dimInd$ defined as
        \begin{equation}
            \diffeoC_\monotone (\Vector) :=  \left\{\begin{matrix}
        \monotone\bigl(\|\Vector\|_2\bigr) \frac{\Vector}{\|\Vector\|_2}  & \|\Vector\|_2 > 0 \\
        \mathbf{0} & \|\Vector\|_2 = 0 \\
        \end{matrix}\right. ,
        \label{eq:monotone-homeo}
        \end{equation}
        is invertible and its inverse is given by
        \begin{equation}
            \diffeoC_\monotone^{-1} (\VectorB) :=  \left\{\begin{matrix}
        \monotone\bigl(\|\VectorB\|_2\bigr)^{-1} \frac{\VectorB}{\|\VectorB\|_2}   & \|\VectorB\|_2 > 0 \\
        \mathbf{0} & \|\VectorB\|_2 = 0 \\
        \end{matrix}\right. .
        \label{eq:monotone-homeo-inv}
        \end{equation}
    \end{enumerate}
\end{proposition}

The pullback geometry we advocate for in this work is defined by the composition of the form $\diffeo := \diffeoC_{\monotone} \circ \diffeoB_{\radial} \circ \diffeoA$. While the role of $\radial$ in $\diffeoB_{\radial}$ is perhaps somewhat anticipated, the introduction of $\monotone$ and $\diffeoC_{\monotone}$ requires additional motivation. Before presenting the main result showing that this geometry indeed yields high-likelihood geodesics, it is instructive to compare pullback geodesics obtained with and without $\diffeoC_{\monotone}$. \Cref{fig:river-geodesics-schematic-b-curves} shows that pullback geodesics with $\diffeoC_{\monotone}$ contract more strongly towards the center; this is particularly beneficial for geodesics connecting to the $\textcolor{tabtenfour}{\pmb{\blacklozenge}}$-archetype, which otherwise swing out considerably without this adaptation, as illustrated in \Cref{fig:river-geodesics-schematic-a-curves}. More generally, reduced swinging in the figure corresponds to trajectories that pass through regions of even higher likelihood than the default geodesics. Thus, rather than compromising our initial goal of obtaining high-likelihood geodesics, the inclusion of $\diffeoC_{\monotone}$ is expected to enhance it, while yielding more stable geodesics. 

\begin{figure}[h!]
    \centering
    \begin{subfigure}{0.48\linewidth}
        \centering
        \includegraphics[width=\linewidth]{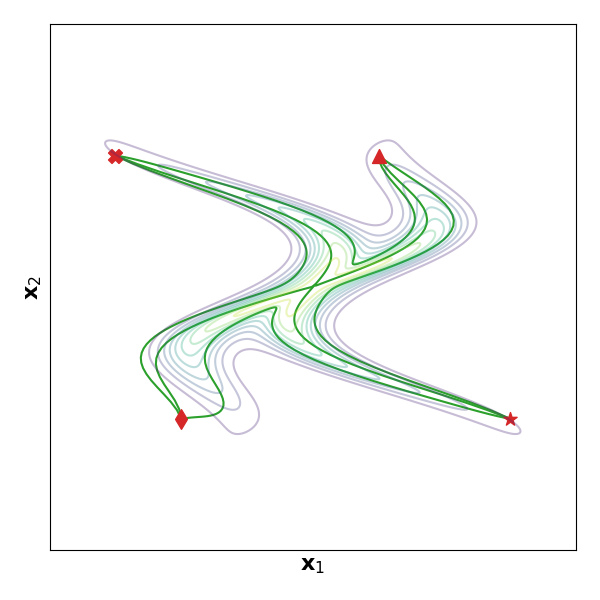}
        \caption{Pullback geodesics under $\diffeo := \diffeoC_{\monotone} \circ \diffeoB_{\radial} \circ \diffeoA$.}
        \label{fig:river-geodesics-schematic-b-curves}
    \end{subfigure}
    \hfill
    \begin{subfigure}{0.48\linewidth}
        \centering
        \includegraphics[width=\linewidth]{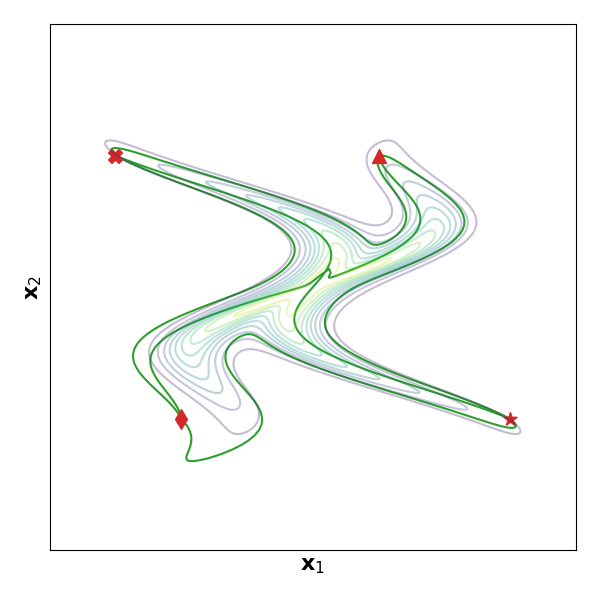}
        \caption{Pullback geodesics under $\diffeo := \diffeoB_{\radial} \circ \diffeoA$.}
        \label{fig:river-geodesics-schematic-a-curves}
    \end{subfigure}
    \caption{Using the default pullback geometry for a data distribution $\density_{\diffeoA, \radial}$ can cause geodesics to swing out in an undesirable way, whereas a modified pullback geometry obtained by composing with $\diffeoC_{\monotone}$ for $\monotone(s) = \log(10s + 1)$ eliminates this behavior.}
    \label{fig:river-geodesics-schematic-curves}
\end{figure}

To make this intuition precise and move toward a general statement, we first require the following lemma, which essentially indicates when incorporating $\diffeoC_{\monotone}$ into the diffeomorphism does not destroy strong convexity.



\begin{lemma}
\label{lem:strong-convexity}
    Let $\diffeoC_\monotone:\Real^\dimInd\to\Real^\dimInd$ be an invertible mapping of the form \eqref{eq:monotone-homeo} generated by a concave strictly increasing function $\monotone:\Real_{> 0}\to \Real_{> 0}$ with $\lim_{s\to 0} \monotone(s)=0$ and $\lim_{s\to 0} \monotone'(s)>0$.
    
    Then, for any vectors $\Vector\neq \VectorB\in \Real^\dimInd$, the mapping
    \begin{equation}
        t\mapsto \|\diffeoC_\monotone^{-1}((1 - t)\diffeoC_\monotone(\Vector) + t \diffeoC_\monotone(\VectorB))\|_2^2
    \end{equation}
    is strongly convex.
\end{lemma}

With all these components in place, we can now state the main result, which formalizes the behavior illustrated in \Cref{fig:river-geodesics-schematic}: geodesics under a pullback geometry that is possibly modified by a diffeomorphism $\diffeoC_{\monotone}$ with additional structure on $\monotone$ indeed pass through high-likelihood regions in general.

\begin{theorem}[high-likelihood geodesics]
\label{thm:geodesic-convexity-starflows}
    Let $\diffeoA: \Real^\dimInd \to \Real^\dimInd$ be a smooth diffeomorphism with constant Jacobian determinant, let $\radial:\Sphere^{\dimInd-1}\to\Real_{>0}$ be a radial function, and let  $\monotone:\Real_{> 0}\to \Real_{> 0}$ be a concave strictly increasing function with $\lim_{s\to 0} \monotone(s)=0$ and $\lim_{s\to 0} \monotone'(s)>0$. In addition, let $\density_{\diffeoA,\radial}$ be a probability density of the form \eqref{eq:starflow-density} generated by $\diffeoA$ and $\radial$, and let $\diffeoB_\radial:\Real^\dimInd\to\Real^\dimInd$ and $\diffeoC_\monotone:\Real^\dimInd\to\Real^\dimInd$ be invertible mappings of the form \eqref{eq:radial-homeo} and \eqref{eq:monotone-homeo} generated by $\radial$ and $\monotone$, respectively.

Then, for any vectors $\Vector\neq \VectorB\in \Real^\dimInd$ and $\diffeo := \diffeoC_\monotone \circ \diffeoB_\radial \circ \diffeoA$, mapping

\begin{equation}
    t \mapsto - \log \bigl( \density_{\diffeoA, \radial} (\gamma_{\mathbf{x}, \mathbf{y}}^{\diffeo}(t) ) \bigr), \quad t \in[0,1]
    \label{eq:thm1-neg-log-likelihood-geo}
\end{equation}

is strongly convex.
    
\end{theorem}

\begin{proof}[Proof sketch]
    We rewrite the mapping \eqref{eq:thm1-neg-log-likelihood-geo} to the form  
\[
t \mapsto \tfrac12\bigl\|\diffeoC_\monotone^{-1}\bigl((1-t)\,\diffeoC_\monotone(\Vector') + t\,\diffeoC_\monotone(\VectorB')\bigr)\bigr\|_2^2 + \text{const},
\]
for suitable \(\Vector',\VectorB'\), and then invoke \Cref{lem:strong-convexity}, which shows that this squared norm is strongly convex in \(t\) on \([0,1]\).
\end{proof}

\begin{remark}
    It is worth noting that \Cref{thm:geodesic-convexity-starflows} strictly generalizes the deformed Gaussian setting of \cite[Thm.~3.3]{diepeveen2025scorebased}, which is recovered when having $\radial(\sVector) = 1$ -- so that the deformed star distribution \eqref{eq:starflow-density} reduces to a deformed Gaussian distribution \eqref{eq:deformed-gaussian-density} -- and choosing $\monotone (s) = s$ -- so that the pullback geometry solely comes from the diffeomorphism $\diffeoA$.
\end{remark}

\subsection{Iso-Riemannian geometry and interpretability}

Before considering the pullback construction complete, it is important to note that, although the geodesics have the desired shape, their speed is not constant in the Euclidean ($\ell^2$) sense, mostly due to the diffeomorphisms $\diffeoB_{\radial}$ and $\diffeoC_{\monotone}$ having non-constant Jacobian determinants. This leads to misleading interpolations -- an effect predicted by theory \cite[Thm.~3.4]{diepeveen2024pulling} and clearly visible in \Cref{fig:river-geodesics-schematic-b,fig:river-geodesics-schematic-a}, where time-equidistant points along the geodesics can cluster towards the center of the distribution, even more so after introducing $\diffeoC_{\monotone}$. As a result, one obtains the erroneous impression that most of the data between the endpoints looks very similar.

\begin{figure}[h!]
    \centering
    \begin{subfigure}{0.48\linewidth}
        \centering
        \includegraphics[width=\linewidth]{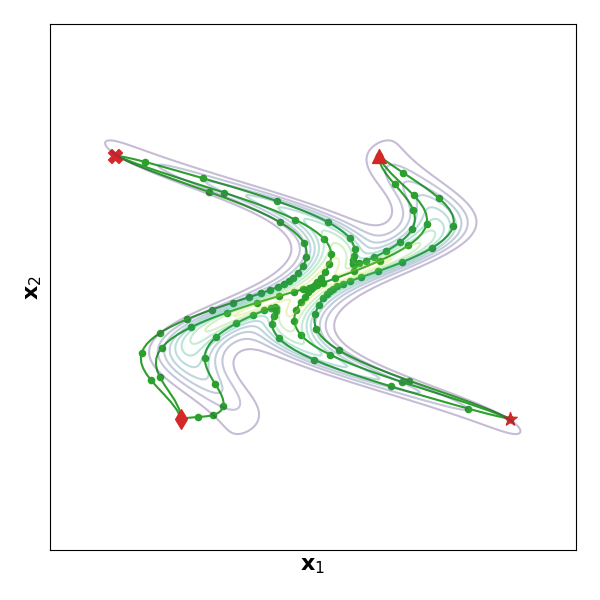}
        \caption{Pullback geodesics under $\diffeo := \diffeoC_{\monotone} \circ \diffeoB_{\radial} \circ \diffeoA$.}
        \label{fig:river-geodesics-schematic-b}
    \end{subfigure}
    \hfill
    \begin{subfigure}{0.48\linewidth}
        \centering
        \includegraphics[width=\linewidth]{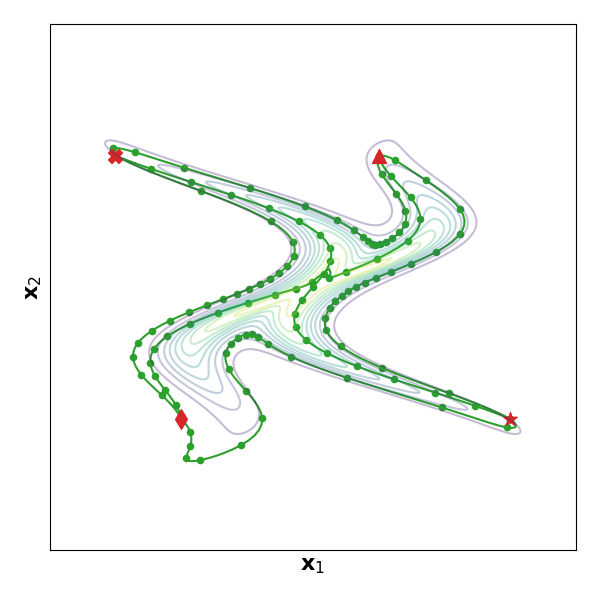}
        \caption{Pullback geodesics under $\diffeo := \diffeoB_{\radial} \circ \diffeoA$.}
        \label{fig:river-geodesics-schematic-a}
    \end{subfigure}\\
    \begin{subfigure}{0.48\linewidth}
        \centering
        \includegraphics[width=\linewidth]{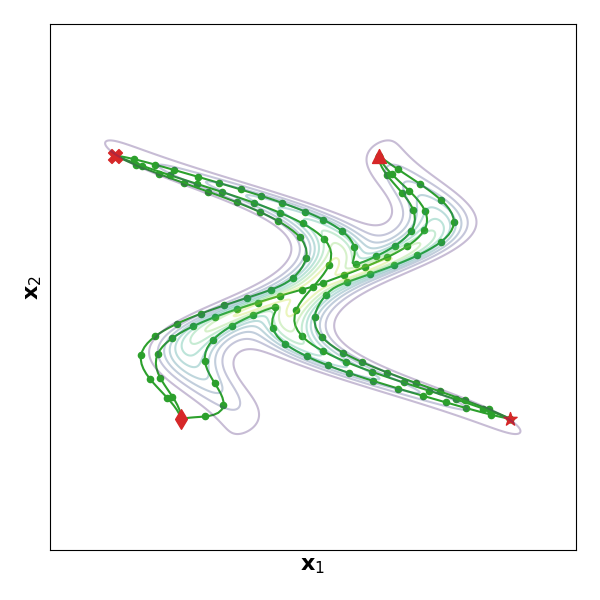}
        \caption{Iso-pullback geodesics under $\diffeo := \diffeoC_{\monotone} \circ \diffeoB_{\radial} \circ \diffeoA$.}
        \label{fig:river-geodesics-schematic-d}
    \end{subfigure}
    \hfill
    \begin{subfigure}{0.48\linewidth}
        \centering
        \includegraphics[width=\linewidth]{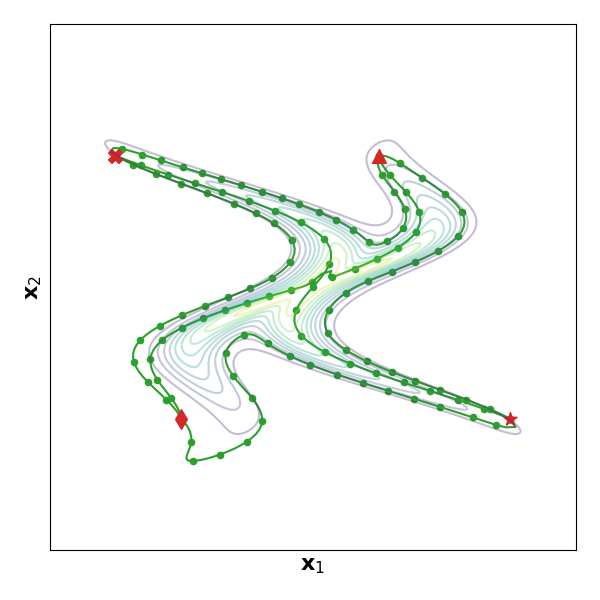}
        \caption{Iso-pullback geodesics under $\diffeo := \diffeoB_{\radial} \circ \diffeoA$.}
        \label{fig:river-geodesics-schematic-c}
    \end{subfigure}
    \caption{Using either the default pullback geometry or its modified variant for a data distribution $\density_{\diffeoA,\radial}$ produces geodesics that can spend most of their time near the center of the deformed star (top), whereas their iso-Riemannian counterparts remove this effect (bottom).}
    \label{fig:river-geodesics-schematic}
\end{figure}

This issue was anticipated by \cite{diepeveen2025manifold}, in which the authors proposed iso-Riemannian geometry as a remedy. For our purposes, it suffices to recall that iso-Riemannian geometry constructs the fundamental manifold mappings using the iso-connection rather than the Levi-Civita connection. In practice, this implies that geodesics retain the same shape as those induced by the Levi-Civita connection but are reparametrized to have constant $\ell^2$-speed, as illustrated in \Cref{fig:river-geodesics-schematic-d,fig:river-geodesics-schematic-c}. In particular, writing $\geodesic^{\diffeo,\iso}_{\Vector,\VectorB}:[0,1]\to \Real^\dimInd$ for iso-geodesics under $(\Real^\dimInd, (\cdot, \cdot)^{\diffeo})$, we have
\begin{equation}
    \geodesic^{\diffeo,\iso}_{\Vector,\VectorB}(t) = \geodesic^{\diffeo}_{\Vector,\VectorB}(\timechange_{\Vector,\VectorB}(t)), \quad t \in [0,1],    
    \label{eq:iso-geodesic}
\end{equation}
where $\geodesic^{\diffeo}_{\Vector,\VectorB}$ is the standard pullback geodesic between $\Vector$ and $\VectorB$ induced by $(\Real^{\dimInd}, (\cdot,\cdot)^{\diffeo})$, and where the reparametrization mapping $\timechange_{\Vector,\VectorB} : [0,1] \to [0,1]$ is a diffeomorphism specified via its inverse
\begin{equation}
    \timechange_{\Vector,\VectorB}^{-1}(t') := \frac{\displaystyle\int_0^{t'} \bigl\|\dot{\geodesic}^{\diffeo}_{\Vector,\VectorB}(s)\bigr\|_2 \, ds} {\displaystyle\int_0^1 \bigl\|\dot{\geodesic}^{\diffeo}_{\Vector,\VectorB}(s)\bigr\|_2 \, ds}.
    \label{eq:iso-geodesic-time-change}
\end{equation}

\begin{remark}
    It is worth highlighting that in practice $\timechange_{\Vector,\VectorB}$ is evaluated approximately by replacing $\geodesic^{\diffeo}_{\Vector,\VectorB}$ with a piecewise linear approximation, for which \eqref{eq:iso-geodesic-time-change} can be computed in closed form. See \cite{diepeveen2025manifold} for details.
\end{remark}

\section{Riemannian archetypal projections for deformed star distributions}
\label{sec:rams}

Now that we have an appropriate statistical and geometric framework for a Riemannian version of archetypal analysis (AA) -- a deformed star distribution together with its associated pullback geometry -- we turn to formulating archetypal analysis itself in this setting. Classical AA can be viewed as comprising two components: learning the archetypes and using these archetypes to project and classify the entire data set; in our framework, it will be more natural to tackle the second part first, since the archetypes and the underlying distribution will ultimately be learned jointly in the next section. Accordingly, we assume the data are already endowed with a star-shaped statistical model and a fixed collection of archetypes situated near the “corners” of the deformed star. In what follows, we introduce the Riemannian archetypal mapping (RAM) as the projection of data points onto the manifold defined by geodesic convex combinations of the archetypes. To address the nonconvexity of the RAM problem, we propose a convex relaxation that provides a strong initialization; solving the resulting RAM then yields a “denoised” approximation on this manifold. Finally, to extract an interpretable weight vector encoding the corresponding archetypal mixture, we discuss the necessity of transitioning to the iso-Riemannian geometry induced by the pullback metric and propose a simple way of doing so. As before, we defer all proofs to \Cref{app:rams}.

\subsection{The Riemannian archetypal mapping}

Before turning to classification, we first seek to project data onto the set of convex combinations of the archetypes. To this end, we define the \emph{Riemannian archetypal mapping} (RAM) generated by archetypes $\emVector^{(1)}, \ldots, \emVector^{(\dimIndB)} \in \Real^\dimInd$ and pullback structure $(\Real^\dimInd,(\cdot,\cdot)^{\diffeo})$ -- in practice arising from a diffeomorphism of the form $\diffeo = \diffeoC_{\monotone} \circ \diffeoB_{\rho} \circ \diffeoA$ -- as the mapping $\RAM^\diffeo:\Real^\dimInd\to\Real^\dimInd$ given by
\begin{equation}
    \RAM^\diffeo(\Vector) := \argmin_{\VectorB\in \bar{\manifold}^\diffeo} \|\Vector - \VectorB\|_2^2,
    \label{eq:RAM}
\end{equation}
where 
\begin{equation}
    \bar{\manifold}^\diffeo := \bigl\{ \Vector \in \Real^\dimInd \; \mid \; \Vector = \argmin_{\VectorB\in \Real^\dimInd} \sum_{\sumIndB=1}^\dimIndB \RAMweight_\sumIndB \distance^{\diffeo}_{\Real^\dimInd} (\VectorB, \emVector^{(\sumIndB)})^2 \text{ for some } \RAMweight \in \Simplex^\dimIndB\bigr\},
    \label{eq:RAM-manifold}
\end{equation}
and where $\Simplex^\dimIndB\subset \Real^\dimIndB$ is the unit simplex in $\Real^\dimIndB$.

The RAM can thus be viewed as the $\ell^2$-projection onto the constraint set $\bar{\manifold}^{\diffeo}$, consisting of all weighted Riemannian barycentres (with respect to $(\Real^{\dimInd}, (\cdot,\cdot)^{\diffeo})$) of the archetypes. This choice is natural: it reduces to the classical interpretation of linear AA in the Euclidean setting, while in our more general framework every geodesic segment between a pair of archetypes is contained in $\bar{\manifold}^{\diffeo}$. In addition, for points that already live in the constraint set, the RAM reduces to the identity mapping, i.e., $\RAM^\diffeo (\Vector) = \Vector$ for any $\Vector \in \bar{\manifold}^\diffeo$.

However, in its present formulation, computing the RAM for general inputs appears intractable. The first step toward making \eqref{eq:RAM} more manageable is to simplify the constraint set, which turns out to be geodesically convex submanifold with corners, i.e., an embedded polytope-like set shaped by $\diffeo$, and 
admits a more convenient representation.

\begin{theorem}[Properties of $\bar{\manifold}^\diffeo$]
\label{thm:ram-manifold-properties}
    Let $\diffeo:\Real^\dimInd\to\Real^\dimInd$ be a smooth diffeomorphism and let $\emVector^{(1)}, \ldots, \emVector^{(\dimIndB)} \in \Real^\dimInd$ be any $\dimIndB$ vectors.
    
    Then, the set $\bar{\manifold}^\diffeo$ defined in \eqref{eq:RAM-manifold}
    \begin{enumerate}[label=(\roman*)]
        \item satisfies
        \begin{equation}
            \bar{\manifold}^\diffeo = \{\Vector \in \Real^\dimInd \; \mid \; \Vector = \diffeo^{-1} \Bigl(\sum_{\sumIndB=1}^\dimIndB \RAMweight_\sumIndB  \diffeo(\emVector^{(\sumIndB)})\Bigr) \text{ for some } \RAMweight \in \Simplex^\dimIndB\}.
        \end{equation}
        \item is a strongly geodesically convex set of $(\mathbb{R}^{\dimInd},(\cdot,\cdot)^{\diffeo})$.
        \item is a smooth manifold with corners, whose interior has dimension
        \begin{equation}
            \dim( \operatorname{int}(\bar{\manifold}^\diffeo)) = \operatorname{rank} ([ \diffeo(\emVector^{(1)}) - \diffeo(\emVector^{(\dimIndB)}), \ldots, \diffeo(\emVector^{(\dimIndB-1)}) - \diffeo(\emVector^{(\dimIndB)})]) \leq \dimIndB-1.
        \end{equation}
    \end{enumerate}
\end{theorem}

\begin{proof}[Proof sketch]

For (i), each term \(\Vector \mapsto \distance_{\Real^\dimInd}^\diffeo(\Vector,\VectorB)^2\) is strongly geodesically convex since it is just \(\|\diffeo(\Vector)-\diffeo(\VectorB)\|_2^2\) in the embedding, so the weighted sum has a unique minimizer, and checking the first-order optimality conditions shows that this minimizer is exactly \(\diffeo^{-1}\bigl(\sum_{\sumIndB} \RAMweight_{\sumIndB}\diffeo(\emVector^{(\sumIndB)})\bigr)\) for any \(\RAMweight\in\Delta_{\dimIndB}\). 

For (ii), using the representation from (i) we write \(\Vector,\VectorB\) as weighted barycenters in the \(\diffeo\)-embedding and interpolate their weights linearly in \(t\); the induced curve is both a geodesic by the explicit formula \(\geodesic^\diffeo_{\Vector,\VectorB}(t)=\diffeo^{-1}((1-t)\diffeo(\Vector)+t\diffeo(\VectorB))\) and remains in \(\bar{\manifold}^\diffeo\), which shows geodesic convexity. 

For (iii), the map \(\diffeo'\colon\bar{\manifold}^\diffeo\to\operatorname{conv}\{\diffeo(\emVector^{(1)}),\ldots,\diffeo(\emVector^{(\dimIndB)})\}\), \(\diffeo'(\Vector)=\diffeo(\Vector)\), is a diffeomorphism, so \(\bar{\manifold}^\diffeo\) is a smooth manifold with corners whose interior dimension equals that of the convex polytope, namely the rank of the difference matrix of the embedded archetypes, which is at most \(\dimIndB-1\).

\end{proof}

Especially the first property in \Cref{thm:ram-manifold-properties} is very useful in terms of evaluating the RAM. That is, it allows us to rephrase the optimization problem \eqref{eq:RAM} as a problem over a simplex.

\begin{theorem}[Equivalent formulation of $\RAM^\diffeo$]
\label{thm:ram-mapping-equivalence}
    Let $\diffeo:\Real^\dimInd\to\Real^\dimInd$ be a smooth diffeomorphism and let $\emVector^{(1)}, \ldots, \emVector^{(\dimIndB)} \in \Real^\dimInd$ be any $\dimIndB$ vectors. 
    
    Then, the Riemannian archetypal mapping $\RAM^\diffeo:\Real^\dimInd\to\Real^\dimInd$ generated by $\emVector^{(1)}, \ldots, \emVector^{(\dimIndB)}$ and $(\Real^\dimInd,(\cdot,\cdot)^{\diffeo})$ satisfies
    \begin{equation}
        \RAM^\diffeo(\Vector) = \diffeo^{-1} \Bigl(\sum_{\sumIndB=1}^\dimIndB \RAMweight_\sumIndB^*  \diffeo(\emVector^{(\sumIndB)})\Bigr), \quad \text{where } \RAMweight^* \in \argmin_{\RAMweight \in \Simplex^\dimIndB} \|\Vector - \diffeo^{-1} \Bigl(\sum_{\sumIndB=1}^\dimIndB \RAMweight_\sumIndB  \diffeo(\emVector^{(\sumIndB)})\Bigr)\|_2^2.
        \label{eq:RAM-rewrite}
    \end{equation}
\end{theorem}

\begin{proof}
    The claim follows directly from (i) in \Cref{thm:ram-manifold-properties}.
\end{proof}

At this point, it is important to note that the optimization problem \eqref{eq:RAM-rewrite} is generally non-convex, so computing the RAM requires some care. Before turning to algorithmic aspects, it is helpful to briefly relate this mapping to the specific setting of interest, namely the deformed star distributions $\density_{\diffeoA,\radial}$ under the choice $\diffeo := \diffeoC_{\monotone} \circ \diffeoB_{\rho} \circ \diffeoA$. Beyond providing a convenient reformulation of the RAM, this perspective also allows us to establish its statistical properties with respect to $\density_{\diffeoA,\radial}$. In particular, any RAM-projected vector can be shown to be at least as likely as the corresponding weighted combination of the archetype likelihoods.

\begin{corollary}
    Let $\diffeoA: \Real^\dimInd \to \Real^\dimInd$ be a smooth diffeomorphism with constant Jacobian determinant, let $\radial:\Sphere^{\dimInd-1}\to\Real_{>0}$ be a radial function, and let  $\monotone:\Real_{> 0}\to \Real_{> 0}$ be a concave strictly increasing function with $\lim_{s\to 0} \monotone(s)=0$ and $\lim_{s\to 0} \monotone'(s)>0$. In addition, let $\density_{\diffeoA,\radial}$ be a probability density of the form \eqref{eq:starflow-density} generated by $\diffeoA$ and $\radial$, and let $\diffeoB_\radial:\Real^\dimInd\to\Real^\dimInd$ and $\diffeoC_\monotone:\Real^\dimInd\to\Real^\dimInd$ be invertible mappings of the form \eqref{eq:radial-homeo} and \eqref{eq:monotone-homeo} generated by $\radial$ and $\monotone$, respectively.
    
    Then, for any $\Vector \in \Real^\dimInd$ and $\diffeo := \diffeoC_\monotone \circ \diffeoB_\rho \circ \diffeoA$, the Riemannian archetypal mapping satisfies
        \begin{equation}
            \density_{\diffeoA, \radial} (\RAM^\diffeo(\Vector)) \geq   \exp \Bigl( \sum_{\sumIndB=1}^\dimIndB \RAMweight_\sumIndB^* \log \bigl( \density_{\diffeoA, \radial} (\emVector^\sumIndB) \bigr) \Bigr) , \quad  \text{for any } \RAMweight^* \in \argmin_{\RAMweight \in \Simplex^\dimIndB} \|\Vector - \diffeo^{-1} \Bigl(\sum_{\sumIndB=1}^\dimIndB \RAMweight_\sumIndB  \diffeo(\emVector^{(\sumIndB)})\Bigr)\|_2^2.
        \end{equation}
\end{corollary}

\begin{remark}
    The reformulation of both the constraint set $\bar{\manifold}^{\diffeo}$ in \Cref{thm:ram-manifold-properties} and the RAM $\RAM^{\diffeo}$ itself in \Cref{thm:ram-mapping-equivalence} crucially relies on working with a Euclidean pullback, for which weighted barycentres admit a closed-form expression -- whereas in general Riemannian settings they must be computed via numerical optimization.
\end{remark}

\subsection{Solving the RAM problem}

The optimization problem \eqref{eq:RAM-rewrite} in \Cref{thm:ram-mapping-equivalence} remains non-convex, so a good initialization strategy is essential. In analogy with the relaxation used for Riemannian autoencoders (RAEs) \cite{diepeveen2024pulling}, we therefore introduce a relaxed variant. Specifically, we define the relaxed Riemannian archetypal mapping (relaxed RAM) generated by archetypes $\emVector^{(1)}, \ldots, \emVector^{(\dimIndB)} \in \Real^\dimInd$ and pullback structure $(\Real^\dimInd,(\cdot,\cdot)^{\diffeo})$ as the mapping $\relRAM^\diffeo:\Real^\dimInd\to\Real^\dimInd$ given by
\begin{equation}
    \relRAM^\diffeo(\Vector) := \argmin_{\VectorB\in \bar{\manifold}^\diffeo} \distance_{\Real^\dimInd}^\diffeo(\Vector, \VectorB)^2.
    \label{eq:relRAM}
\end{equation}
The relaxed RAM can be interpreted as the Riemannian metric projection onto the constraint set $\bar{\manifold}^{\diffeo}$. For diffeomorphisms with non-constant Jacobian determinant -- again due to the presence of $\diffeoB_{\radial}$ and $\diffeoC_{\monotone}$ in a diffeomorphism $\diffeo = \diffeoC_{\monotone} \circ \diffeoB_{\rho} \circ \diffeoA$ we would use in practice -- this generally yields substantially different minimizers than the (original) RAM. Nevertheless, for points already lying in the constraint set the relaxed RAM still reduces to the identity, i.e., $\relRAM^{\diffeo}(\Vector) = \Vector$ for all $\Vector \in \bar{\manifold}^{\diffeo}$.

As before, \Cref{thm:ram-manifold-properties} provides a practical way to evaluate the relaxed RAM, with the key difference that the reformulated optimization problem is now convex.

\begin{theorem}[Equivalent formulation of $\relRAM^\diffeo$]
\label{thm:rel-ram-mapping-equivalence}
    Let $\diffeo:\Real^\dimInd\to\Real^\dimInd$ be a smooth diffeomorphism and let $\emVector^{(1)}, \ldots, \emVector^{(\dimIndB)} \in \Real^\dimInd$ be any $\dimIndB$ vectors. 
    
    Then, the relaxed Riemannian archetypal mapping $\relRAM^\diffeo:\Real^\dimInd\to\Real^\dimInd$ generated by $\emVector^{(1)}, \ldots, \emVector^{(\dimIndB)}$ and $(\Real^\dimInd,(\cdot,\cdot)^{\diffeo})$ satisfies
    \begin{equation}
        \relRAM^\diffeo(\Vector) = \diffeo^{-1} \Bigl(\sum_{\sumIndB=1}^\dimIndB \RAMweight_\sumIndB^*  \diffeo(\emVector^{(\sumIndB)})\Bigr), \quad \text{where } \RAMweight^* \in \argmin_{\RAMweight \in \Simplex^\dimIndB} \|\diffeo(\Vector) - \sum_{\sumIndB=1}^\dimIndB \RAMweight_\sumIndB  \diffeo(\emVector^{(\sumIndB)})\|_2^2,
        \label{eq:relRAM-rewrite}
    \end{equation}
\end{theorem}

\begin{proof}
    The claim follows directly from (i) in \Cref{thm:ram-manifold-properties}.
\end{proof}

The optimization problem \eqref{eq:relRAM-rewrite} can be solved efficiently with proximal gradient descent \cite{parikh2014proximal}, which yields the update scheme
\begin{equation}
    \RAMweight_{\operatorname{rel}}^{(\iterInd+1)} := \Pi_{\Simplex^\dimIndB}\Bigl( \RAMweight_{\operatorname{rel}}^{(\iterInd)} - \alpha \,\diffeo(\emMatrix)^\top \bigl(\diffeo(\emMatrix) \RAMweight_{\operatorname{rel}}^{(\iterInd)} - \diffeo(\Vector) \bigr)\Bigr), \quad \text{with } \diffeo(\emMatrix):= [\diffeo(\emVector^{(1)}), \ldots, \diffeo(\emVector^{(\dimIndB)})] \in \Real^{\dimInd \times \dimIndB},
    \label{eq:pgd-rel-ram}
\end{equation}
where $\Pi_{\Simplex^\dimIndB}:\Real^\dimIndB\to \Simplex^\dimIndB$ is the $\ell^2$-orthogonal projection onto the  simplex and the step size 
\begin{equation}
    \alpha := \frac{1}{\|\diffeo(\emMatrix)^\top \diffeo(\emMatrix)\|} = \frac{1}{\lambda_{\max}( \diffeo(\emMatrix)^\top \diffeo(\emMatrix))}
\end{equation}
is used to ensure convergence. The scheme is initialized at $\RAMweight_{\operatorname{rel}}^{(0)} := \frac{1}{\dimIndB} \mathbf{1}_\dimIndB:= [\frac{1}{\dimIndB}, \ldots, \frac{1}{\dimIndB}]^\top\in \Real^\dimIndB$.

Upon convergence of proximal gradient descent for the relaxed RAM problem, the resulting weights $\RAMweight_{\operatorname{rel}}^*\in \Real^\dimIndB$ serve as initialization for proximal gradient descent on the RAM objective\footnote{In practice, using $\RAMweight_{\operatorname{rel}}^{(0)}$ as initialization can be beneficial if it yields a lower RAM objective value.}, leading to the update scheme
\begin{equation}
    \RAMweight^{(\iterInd+1)} := \Pi_{\Simplex^\dimIndB}\Bigl( \RAMweight^{(\iterInd)} - \alpha_\iterInd \,\diffeo(\emMatrix)^\top (D_{\diffeo(\emMatrix) \RAMweight^{(\iterInd)}}\diffeo^{-1})^\top \bigl(\diffeo^{-1}(\diffeo(\emMatrix) \RAMweight^{(\iterInd)}) - \Vector \bigr)\Bigr),
    \label{eq:pgd-ram}
\end{equation}
where $\alpha_\iterInd$ is determined via line search to promote convergence.

\begin{figure}[h!]
    \centering
    \begin{subfigure}{0.48\linewidth}
        \centering
        \includegraphics[width=\linewidth]{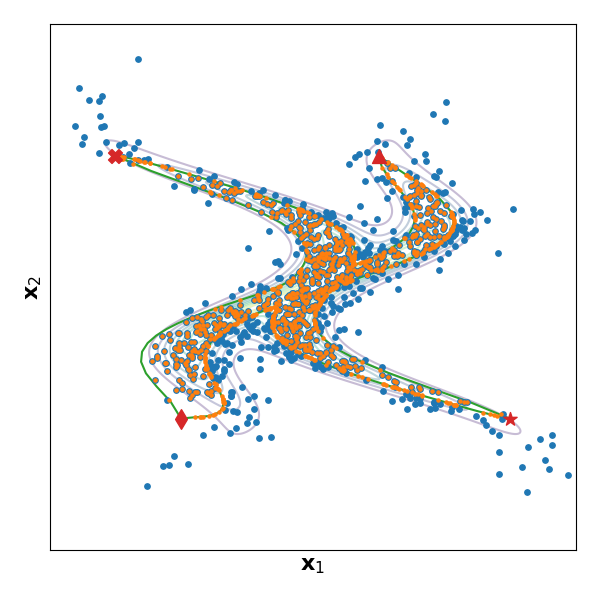}
        \caption{RAM projections (orange) or deformed star data (blue)}
        \label{fig:river-ram-schematic-proj}
    \end{subfigure}
    \hfill
    \begin{subfigure}{0.48\linewidth}
        \centering
        \includegraphics[width=\linewidth]{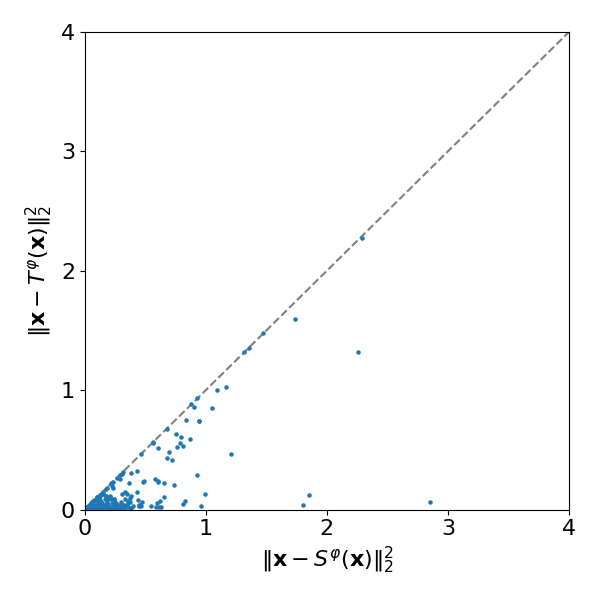}
        \caption{Squared distance of RAM and relaxed RAM to original data}
        \label{fig:river-ram-schematic-errors}
    \end{subfigure}
    \caption{The Riemannian archetypal mapping (RAM) -- initialized from relaxed RAM weights -- projects data points from a deformed star distribution onto the data manifold generated by four archetypes (left). RAM projections achieve substantially smaller distances to the original data compared to relaxed RAM projections, as evidenced by the reduction in reconstruction error (right).}
    \label{fig:river-ram-schematic-resulrs}
\end{figure}

For pullback geometries induced by deformed star distributions, the RAM optimization scheme \eqref{eq:pgd-ram} produces substantial improvements over the relaxed RAM scheme \eqref{eq:pgd-rel-ram}. \Cref{fig:river-ram-schematic-resulrs} illustrates this effect: RAM projections achieve significantly lower reconstruction errors by properly accounting for the pullback metric structure, whereas relaxed RAM incurs systematic projection errors due to neglecting the distance distortion introduced by $\diffeoB_{\rho}$ and $\diffeoC_{\monotone}$.

\begin{remark}
\label{rem:ram-optimization}
    While the proposed proximal gradient scheme performs well empirically, it also highlights clear room for algorithmic improvement. The effective step sizes for solving the RAM problem \(\eqref{eq:pgd-ram}\) are forced to be extremely small, as the objective suffers from rapidly vanishing convexity and very large Lipschitz constants induced by \(\diffeoB_{\rho}\) and \(\diffeoC_{\monotone}\). In the example of \Cref{fig:river-ram-schematic-resulrs}, this manifests in practice: after solving the relaxed RAM problem\footnote{for all data points in parallel} in 19 iterations using \(\eqref{eq:pgd-rel-ram}\), an additional 124 “fine-tuning” iterations of \(\eqref{eq:pgd-ram}\) are required before the maximal change between successive iterates falls below \(10^{-3}\).

    In the Riemannian autoencoder (RAE) setting \cite{diepeveen2024pulling}, iso-Riemannian optimization methods are able to exploit the geometry of the projection problem much more effectively: geodesic convexity of the constraint manifold, combined with (Euclidean) strong convexity of the objective, can yield strong iso-convexity together with a Lipschitz gradient field with substantially better constants \cite{diepeveen2025isoriemanopt}, even when the constraint set is non-convex in the Euclidean sense and the objective lacks classical geodesic convexity. For the RAM, an intrinsic iso-Riemannian approach would not only permit significantly larger step sizes, but would also be especially attractive because the manifold \(\bar{\manifold}^\diffeo\) can have lower dimension than the simplex: working intrinsically on \(\bar{\manifold}^\diffeo\) would remove the multiplicity of minimizers caused by mere convexity on the simplex and instead yield strong iso-convexity on this lower-dimensional space. At present, however, such methods are not available for our problem, as iso-Riemannian optimization theory for constrained settings is still largely undeveloped.
\end{remark}

\subsection{Iso-Riemannian geometry and classification}
Given the strong empirical performance of RAM on the deformed star data in \Cref{fig:river-ram-schematic-resulrs}, one might be inclined to believe that the converged weights $\RAMweight^* \in \Real^\dimIndB$ from proximal gradient descent yield a trustworthy classification of a projected point, and that the magnitude of each entry in $\RAMweight^*$ is directly interpretable. 

However, such an interpretation is premature: due to $\diffeoB_{\rho}$ and $\diffeoC_{\monotone}$ not having a constant Jacobian determinant, the pullback metric generally distorts distances, which affects the geometry experienced by the logarithmic map and hence the meaning of the weights. To make this precise, observe first that  
\begin{multline}
\sum_{\sumIndB=1}^\dimIndB \RAMweight^*_{\sumIndB} \log^\diffeo_{\RAM^\diffeo(\Vector)} (\emVector^{(\sumIndB)}) \overset{\eqref{eq:thm-log-remetrized}}{=} \sum_{\sumIndB=1}^\dimIndB \RAMweight^*_{\sumIndB}  D_{\diffeo(\RAM^\diffeo(\Vector))}\diffeo^{-1}[\diffeo(\emVector^{(\sumIndB)}) - \diffeo(\RAM^\diffeo(\Vector))] \\
=  D_{\diffeo(\RAM^\diffeo(\Vector))}\diffeo^{-1}\Bigl[\sum_{\sumIndB=1}^\dimIndB \RAMweight^*_{\sumIndB} \diffeo(\emVector^{(\sumIndB)}) - \diffeo(\RAM^\diffeo(\Vector))\Bigr]\\
\overset{\text{\Cref{thm:ram-mapping-equivalence}}}{=} D_{\diffeo(\RAM^\diffeo(\Vector))}\diffeo^{-1}\Bigl[\diffeo(\RAM^\diffeo(\Vector)) - \diffeo(\RAM^\diffeo(\Vector))\Bigr] = \mathbf{0}.
\end{multline}
Thus, the weights determine an affine combination of the tangent vectors aimed at the archetypes such that their weighted sum is zero. Conceptually, if $\RAM^\diffeo(\Vector)$ is close to a given archetype, the associated logarithmic vector is short relative to those pointing toward more distant archetypes. To compensate for this disparity, the optimal solution assigns a larger weight to the nearby archetype and smaller weights to archetypes that are farther away.  

The subtlety is that “close”, “far”, “short” and “long” here are measured in the pullback geometry, not in the ambient $\ell^2$-sense.  Because $\diffeo$ is not an isometry, neither the distances nor the lengths of these tangent vectors reflect Euclidean arc length along the underlying geodesics. Consequently, a naive $\ell^2$-based reading of $\RAMweight^*$ can lead to misleading interpretations, analogous to the interpolation pathologies discussed earlier for pullback geodesics. 

This issue was also anticipated in the iso-Riemannian geometry framework of \cite{diepeveen2025manifold}. For the logarithmic map, this means keeping the direction of $\log^\diffeo_{\RAM^\diffeo(\Vector)} (\emVector^{(\sumIndB)})$ but rescaling its length so that it equals the $\ell^2$-arc length of the geodesic connecting the endpoints. Specifically, writing $\log_{\Vector}^{\diffeo,\iso} : \Real^\dimInd \to \tangent_{\Vector}\Real^\dimInd$ for the iso-logarithmic map under $(\Real^\dimInd, (\cdot,\cdot)^\diffeo)$, we have
\begin{equation}
    \log_{\Vector}^{\diffeo,\iso}(\VectorB) 
= \frac{\displaystyle\int_0^1 \bigl\|\dot{\geodesic}^{\diffeo}_{\Vector,\VectorB}(s)\bigr\|_2 \, ds}{\bigl \| \log_{\Vector}^\diffeo(\VectorB)\bigr\|} \, \log_{\Vector}^\diffeo(\VectorB),
\end{equation}
where $\log_{\Vector}^\diffeo(\VectorB)$ is the standard pullback log between $\Vector$ and $\VectorB\in \Real^\dimInd$. 

Rather than using $\RAMweight^*$ to weight the pullback logs, we therefore seek weights that balance the iso-logs to the archetypes, i.e., weights $\RAMweightB^*$ on the simplex satisfying  
\begin{equation}
    \sum_{\sumIndB=1}^\dimIndB \RAMweightB^{*}_{\sumIndB} \log^{\diffeo,\iso}_{\RAM^\diffeo(\Vector)} (\emVector^{(\sumIndB)}) = \mathbf{0}.
    \label{eq:iso-log-weighted}
\end{equation}
Such a solution always exists; in particular, one can verify that  
\begin{equation}
    \RAMweightB^{*}_{\sumIndB} := \frac{c_{\sumIndB}\,\RAMweight^{*}_{\sumIndB}}{\sum_{\sumIndB'=1}^\dimIndB c_{\sumIndB'} \RAMweight^{*}_{\sumIndB'}}, \quad \text{with} \quad 
    c_{\sumIndB} := 
    \begin{cases}
    \displaystyle \frac{\bigl\| \log_{\RAM^\diffeo(\Vector)}^\diffeo(\emVector^{(\sumIndB)})\bigr\|}{\displaystyle\int_0^1 \bigl\|\dot{\geodesic}^{\diffeo}_{\RAM^\diffeo(\Vector),\emVector^{(\sumIndB)}}(s)\bigr\|_2 \, ds} & \text{if } \RAM^\diffeo(\Vector) \neq \emVector^{(\sumIndB)},\\\\
\qquad \quad \quad 1 & \text{otherwise},
\end{cases}
\label{eq:ram-reweighted}
\end{equation}
indeed satisfies \eqref{eq:iso-log-weighted}.  Importantly, we never need to evaluate the iso-logs themselves -- only the scaling coefficients $c_{\sumIndB}$, which suffice to obtain the corrected classification -- and always have $\RAMweightB^{*} = \RAMweight^{*}$ in the classical Euclidean AA setting. 

\Cref{fig:river-logs-schematic-results} illustrates this effect for the deformed star distribution we have been considering. Here, the pullback logarithms (pink) do not encode the true $\ell^2$-distance to the archetypes, whereas the iso-logarithms (purple) do. As a consequence, naive classification overemphasizes the $\textcolor{tabtenfour}{\pmb{\star}}$-archetype for a point lying relatively close to the geodesic connecting the $\textcolor{tabtenfour}{\pmb{\star}}$- and $\textcolor{tabtenfour}{\pmb{\blacktriangle}}$-archetype, while the corrected weights sketch a more tempered, and geometrically faithful, picture.

\begin{remark}
    In the example of \Cref{fig:river-logs-schematic-results}, there are in fact multiple weight vectors $\RAMweight$ that fit the data equally well, because the latent archetypes are collinear in the pullback geometry, i.e., $\dim( \operatorname{int}(\bar{\manifold}^\diffeo)) = 2 \neq  3 = \dimIndB-1$ for this example. This ambiguity can be controlled via regularization, in contrast to methods such as AAnet \cite{van2019finding}, where an overly high-dimensional latent space -- with many points encoding the same information -- can arise without a built-in mechanism to correct for it. 
\end{remark}

\begin{remark}
    As with iso-geodesics, the integral in \eqref{eq:ram-reweighted} is approximated numerically to compute the coefficients $c_{\sumIndB}$.
\end{remark}

\begin{figure}[h!]
    \centering
    \begin{subfigure}{0.48\linewidth}
        \centering
        \includegraphics[width=\linewidth]{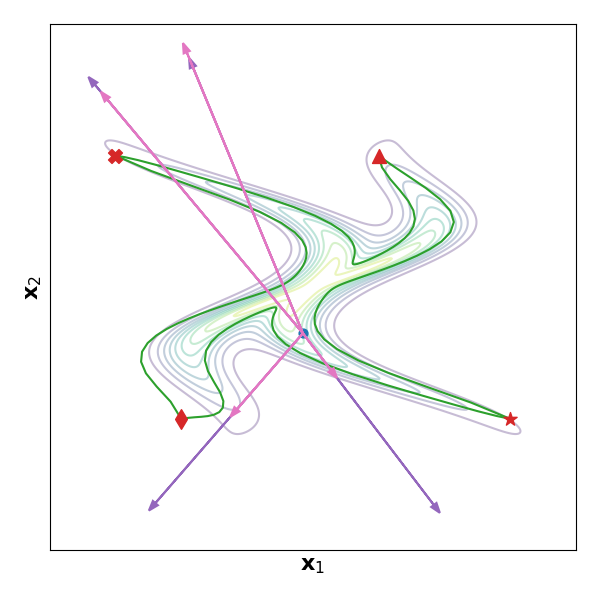}
        \caption{Pullback logs (pink) and iso-pullback logs (purple).}
        \label{fig:river-ram-schematic-proj}
    \end{subfigure}
    \hfill
    \begin{subfigure}{0.48\linewidth}
        \centering
        \includegraphics[width=\linewidth]{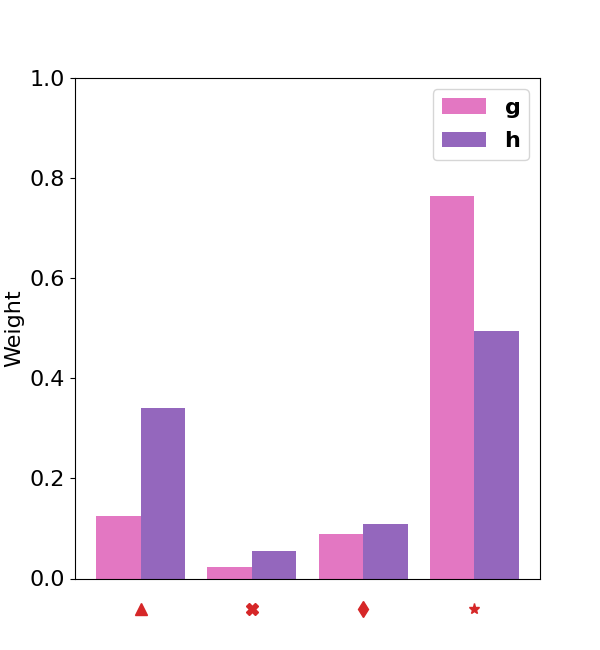}
        \caption{Naive and iso-corrected weights}
        \label{fig:river-ram-schematic-errors}
    \end{subfigure}
    \caption{The pullback logarithmic mappings (pink) do not carry information on how far away the archetypes are in an $\ell^2$-sense -- unlike the iso-logarithmic mappings (purple) --, which causes classification on the $\textcolor{tabtenfour}{\pmb{\star}}$-archetype to be unreasonably high for a data point that lives relatively close to the geodesic connecting the $\textcolor{tabtenfour}{\pmb{\star}}$- and $\textcolor{tabtenfour}{\pmb{\blacktriangle}}$-archetype. When using the iso-corrected weights, the classification is more modest and geometrically faithful.}
    \label{fig:river-logs-schematic-results}
\end{figure}




\section{Learning deformed star distributions}
\label{sec:learning-stars}

Finally, the remaining task is to (jointly) learn a distribution $\density_{\diffeoA,\radial} : \Real^{\dimInd} \to \Real$ of the form \eqref{eq:starflow-density} -- where $\diffeoA : \Real^{\dimInd} \to \Real^{\dimInd}$ is a diffeomorphism with constant Jacobian determinant and $\radial : \Sphere^{\dimInd-1} \to \Real_{>0}$ is a radial function -- and archetypes $\emVector^{(1)}, \ldots, \emVector^{(\dimIndB)} \in \Real^{\dimInd}$ from data $\Vector^{(1)}, \ldots, \Vector^{(\dataNum)}\sim \density_{\text{data}}$. Assuming that the true data distribution $\density_{\text{data}}$ is well-approximated by such a distribution and well-represented by such archetypes, the resulting pullback structure and Riemannian archetypal mapping (RAM) enable the data analysis tasks introduced earlier, i.e., interpolation (\Cref{sec:star-geometry}) and denoising and classification (\Cref{sec:rams}).  In what follows, we first discuss why classical negative log-likelihood training for learning $\density_{\diffeoA,\radial}$ is likely to encounter difficulties, and how a three-step procedure can nevertheless yield a reasonable -- albeit suboptimal -- approximation. We show that existing methods, after suitable adaptations, address all but one component in our three-step approach: learning the radial function, to which we devote a more detailed treatment. As before, all proofs are deferred to \Cref{app:learning-stars}.



\subsection{A three-step approach}

Focusing on the distribution $ \density_{\diffeoA,\radial} $, one would ideally seek a model that is closest to the true data distribution in Kullback-Leibler divergence. In particular, if the true distribution were of this form, it could in principle be recovered by solving
\begin{equation}
    (\diffeoA^*, \radial^*) \in \argmin_{\diffeoA, \radial} \, \mathbb{E}_{\mathbf{\stoVector} \sim \density_{\text{data}}} \left[ -\log \density_{\diffeoA,\radial}(\stoVector) \right].
    \label{eq:star-log-likelihood-loss}
\end{equation}
To understand why this objective is challenging, consider the expanded form of the negative log-likelihood:
\begin{equation}
    -\log \density_{\diffeoA,\radial}(\Vector)
    = \frac{1}{2}\,\radial\Bigl( \frac{\diffeoA(\Vector)}{\|\diffeoA(\Vector)\|_2} \Bigr)^{-2} \|\diffeoA(\Vector)\|_2^2
    - \log\bigl(|\det D_\Vector \diffeoA|\bigr)
    + \log\Bigl( \int_{\Sphere^{\dimInd-1}} \radial(\sVector)^{\dimInd} \, d\sigma(\sVector) \Bigr)
    + \log \Bigl( 2^{\frac{\dimInd}{2}-1} \Gamma\bigl(\tfrac{\dimInd}{2}\bigr) \Bigr).
\end{equation}
This expression shows that optimizing the likelihood requires evaluating an integral over the high-dimensional sphere. For general radial functions, such integrals are not tractable in closed form, while Monte Carlo approximations become numerically unstable due to the exponent $\dimInd$, especially in the high-dimensional regimes of interest.

Motivated by these difficulties, we instead adopt an alternative approach that allows us to learn the diffeomorphism $ \diffeoA $, the archetypes $ \emVector^{(1)}, \ldots, \emVector^{(\dimIndB)} $, and the radial function $ \radial $ separately. 

\paragraph{Step 1: finding a diffeomorphism}
We begin by observing that any diffeomorphism with constant Jacobian determinant cannot alter the topology of the level sets of a distribution. In particular, if the target distribution has the structure of a deformed star, such a diffeomorphism will map it to another deformed star distribution. This observation provides guidance for selecting a suitable diffeomorphism.

In particular, suppose we aim for the pushforward distribution to be as isotropic as possible. Then we expect the latent distribution to exhibit predominantly isotropic behavior, while retaining star-shaped structure in certain directions. Enforcing an isotropic latent distribution corresponds precisely to standard normalizing flow training \cite{dinh2014nice}, which we know to scale well. That is, we seek $\diffeoA_{\networkParams^*}$, for a suitable parametrization $\networkParams^*$ that ensures a constant Jacobian determinant, solving
\begin{equation}
    \networkParams^* \in \argmin_{\networkParams} \, \mathbb{E}_{\mathbf{\stoVector} \sim \density_{\text{data}}} \left[-\log \density_{\diffeoA_{\networkParams}}(\stoVector)\right], 
    \quad \text{where} \quad
    \density_{\diffeoA_{\networkParams}}(\Vector)
    := \frac{1}{\sqrt{(2 \pi)^\dimInd}} \exp\Bigl(-\tfrac{1}{2}\|\diffeoA_{\networkParams}(\Vector)\|_2^2\Bigr)\, \bigl|\det(D_{\Vector} \diffeoA_{\networkParams})\bigr|.
    \label{eq:nf-loss}
\end{equation}

\begin{remark}
    In geometric terms, the first step toward learning a deformed star pullback structure is to proceed along the lines of approaches such as \cite{diepeveen2025scorebased}. Put differently, our method genuinely extends previously proposed schemes.
\end{remark}

\paragraph{Step 2: finding a radial decomposition via archetypes}


Next, rather than attempting to fit all branches simultaneously, it might be better to treat each branch separately and consider radial functions of the form
\begin{equation}
    \radial(\sVector) := \operatorname{softmax}\bigl(\radial_1(\sVector), \ldots, \radial_{\dimIndB}(\sVector)\bigr),
    \label{eq:radial-outer-construction}
\end{equation}
where each $\radial_{\sumIndB} : \Sphere^{\dimInd-1} \to \Real_{>0}$ describes the radial behavior of a single “branch’’ of the data. In \Cref{fig:tree-projections} -- representing latent data $\VectorB^{(\sumIndA)} := \diffeoA_{\networkParams^*}(\Vector^{(\sumIndA)})$ for $\sumIndA=1, \ldots \dataNum$ --, this corresponds to the four individual arms of the distribution. 

There are two natural design choices for how to incorporate archetypes:
\begin{enumerate}[label=(\roman*)]
    \item each branch gets a single archetype, 
    \item each branch gets multiple archetypes.
\end{enumerate}
The first option is most appropriate when we lack prior information about which parts of the data belong to which branch, while the second fits labeled data where, within each class, we aim to identify several archetypes that represent it. In practice the situation need not be this binary, but for the purposes of this work it is useful to analyze both scenarios separately.


\begin{figure}[h!]
    \centering
    \begin{subfigure}{0.48\linewidth}
        \centering
        \includegraphics[width=\linewidth]{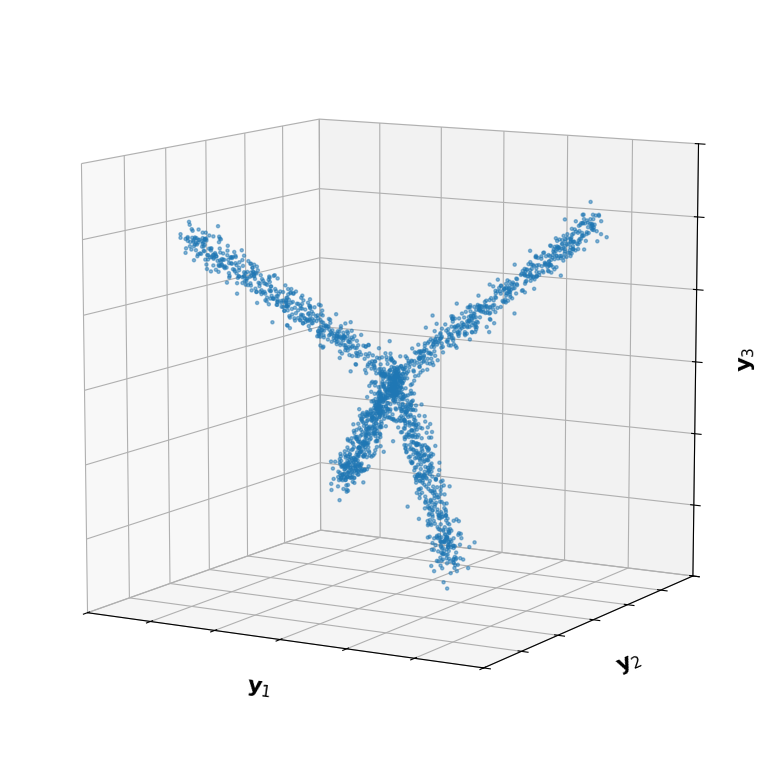}
        \caption{Samples of a star distribution}
        \label{fig:tree-projections}
    \end{subfigure}
    \hfill
    \begin{subfigure}{0.48\linewidth}
        \centering
        \includegraphics[width=\linewidth]{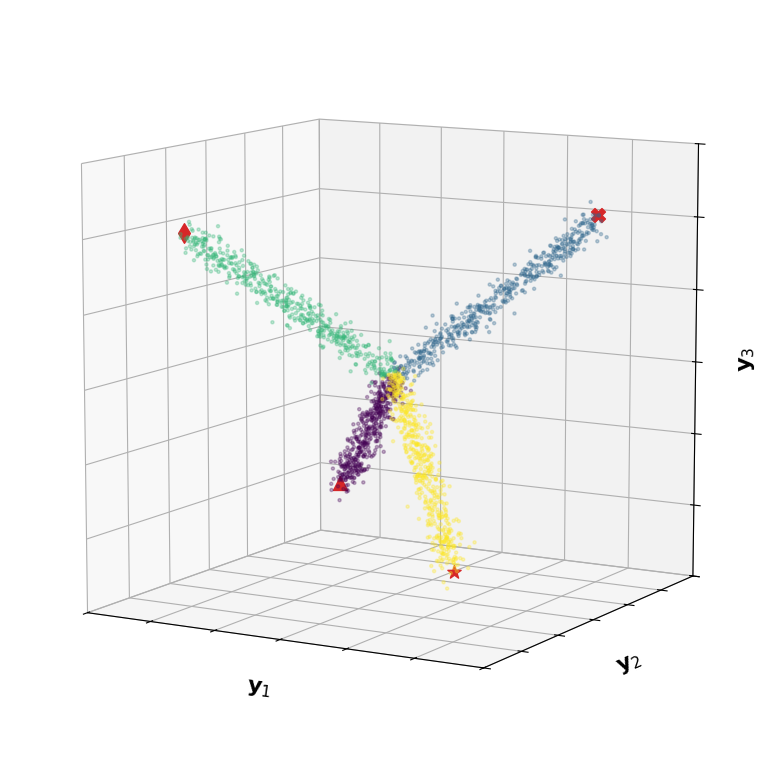}
        \caption{Sample archetypes and corresponding classification}
        \label{fig:tree-aa}
    \end{subfigure}
    \caption{Instead of fitting a single radial function to the entire dataset, we first decompose the latent data (left) using classical archetypal analysis (right) and then model each resulting branch individually.}
    \label{fig:tree-example}
\end{figure}

To make (i) work, we must determine which points belong to which branch, after which it remains to find a radial function $\radial_{\sumIndB}$ that captures the geometry of those data points alone. This can be achieved by applying classical archetypal analysis (AA) \cite{cutler1994archetypal} in the latent space, i.e., choose $\dimIndB \in \Natural$ and solve
\begin{equation}
    \inf_{\simMatrix \in \Real^{\dimInd \times \dimIndB},\; \simMatrixB \in \Real^{\dimIndB \times \dataNum}}
        \bigl\|\MatrixB - \MatrixB\,\simMatrix \simMatrixB\bigr\|_F^2,
    \quad \text{s.t.} \quad
    \simVector^{(\sumIndB)} \in \Simplex^{\dataNum},\;
    \simVectorB^{(\sumIndA)} \in \Simplex^{\dimIndB},
    \label{eq:latent-aa}
\end{equation}
where $\MatrixB := [\VectorB^{(1)}, \ldots, \VectorB^{(\dataNum)}] \in \Real^{\dimInd \times \dataNum}$ collects the mapped data points, and $\simMatrix := [\simVector^{(1)}, \ldots, \simVector^{(\dimIndB)}]$ and $\simMatrixB := [\simVectorB^{(1)}, \ldots, \simVectorB^{(\dataNum)}]$ are the factor matrices whose columns lie in the corresponding unit simplices.

This yields archetypes
\begin{equation}
    \emVector_{\networkParams^*}^{(\sumIndB)}
    := \diffeoA_{\networkParams^*}^{-1}\bigl(\MatrixB \simVector^{(\sumIndB)}\bigr)
    = \diffeoA_{\networkParams^*}^{-1}\Bigl(\sum_{\sumIndA=1}^{\dataNum} \simVector_{\sumIndA}^{(\sumIndB)}\,\diffeoA_{\networkParams^*}(\Vector^{(\sumIndA)})\Bigr)
    \in \Real^{\dimInd},
    \label{eq:latent-aa-aarchetypes}
\end{equation}
and assigns each point $\Vector^{(\sumIndA)}$ a class label
\begin{equation}
    c_{\networkParams^*}^{(\sumIndA)} := \argmax_{\sumIndB = 1,\ldots,\dimIndB} \simVectorB^{(\sumIndA)}_{\sumIndB},
    \label{eq:class-label}
\end{equation}
and is illustrated for $\dimIndB=4$ in \Cref{fig:tree-aa}.

For case (ii), we can proceed in an analogous way by solving \eqref{eq:latent-aa} separately on each labeled subset of the data. We then compute the corresponding archetypes as in \eqref{eq:latent-aa-aarchetypes}, but now there is no need to perform an additional classification step within an already labeled data set.

\begin{remark}
    It is worth noting that the problem in \eqref{eq:latent-aa} is in fact equivalent to an abstract Riemannian formulation of archetypal analysis in the tangent space. Specifically, archetypal analysis can be expressed intrinsically as
\begin{equation}
    \inf_{\simMatrix \in \Real^{\dimInd \times \dimIndB},\; \simMatrixB \in \Real^{\dimIndB \times \dataNum}}
    \Bigl(\bigl\|\tangentVector - \tangentVector\,\simMatrix \simMatrixB\bigr\|_\mPoint^\diffeoA \Bigr)^2,
\quad \text{s.t.} \quad
    \simVector^{(\sumIndB)} \in \Simplex^{\dataNum},\;
    \simVectorB^{(\sumIndA)} \in \Simplex^{\dimIndB},
    \label{eq:abstract-AA}
\end{equation}
where $ \tangentVector := [\log^{\diffeoA_{\networkParams^*}}_{\mPoint} (\Vector^{(1)}), \ldots, \log^{\diffeoA_{\networkParams^*}}_{\mPoint} (\Vector^{(\dataNum)})] \in (\tangent_{\mPoint} \Real^{\dimInd})^{\dataNum} \cong \Real^{\dimInd \times \dataNum} $ collects the logarithmic tangent vectors, and $ \mPoint \in \Real^\dimInd $ is an arbitrary reference point, for instance the Riemannian barycentre \eqref{eq:thm-bary-remetrized} of the data.

To establish this equivalence, we rewrite the objective as
\begin{multline}
    \Bigl(\bigl\|\tangentVector - \tangentVector\,\simMatrix \simMatrixB\bigr\|_\mPoint^\diffeoA \Bigr)^2
    \overset{\eqref{eq:pull-back-metric}}{=}
    \bigl\|D_{\mPoint} \diffeoA_{\networkParams^*}[\tangentVector - \tangentVector\,\simMatrix \simMatrixB]\bigr\|_F^2 \\\\
    \overset{\eqref{eq:thm-log-remetrized}}{=}
    \bigl\|D_{\mPoint} \diffeoA_{\networkParams^*} \circ D_{\diffeoA_{\networkParams^*}(\mPoint)} \diffeoA_{\networkParams^*}^{-1}
    [(\MatrixB - \diffeoA_{\networkParams^*}(\mPoint) \mathbf{1}_\dataNum^\top)
    - (\MatrixB - \diffeoA_{\networkParams^*}(\mPoint) \mathbf{1}_\dataNum^\top)\,\simMatrix \simMatrixB]\bigr\|_F^2 \\\\
    =
    \bigl\| (\MatrixB - \diffeoA_{\networkParams^*}(\mPoint) \mathbf{1}_\dataNum^\top)
    - (\MatrixB - \diffeoA_{\networkParams^*}(\mPoint) \mathbf{1}_\dataNum^\top)\,\simMatrix \simMatrixB \bigr\|_F^2
    =
    \bigl\|\MatrixB - \MatrixB\,\simMatrix \simMatrixB\bigr\|_F^2,
\end{multline}
which coincides with the objective in \eqref{eq:latent-aa} and is notably independent of the reference point $\mPoint$. 

That said, the formulation \eqref{eq:abstract-AA} is more naturally amenable to incorporating curvature corrections 
in more general non-flat Riemannian geometries.
While unnecessary in the present setting, such corrections are expected to have a more pronounced effect in the presence of nontrivial curvature.
\end{remark}

\begin{remark}
    In \Cref{sec:rams}, we emphasized the importance of measuring errors using the $\ell^2$-norm. This naturally raises the question of why we do not instead solve a formulation of the form \eqref{eq:abstract-AA} under a different norm, namely
    \begin{equation}
    \inf_{\simMatrix \in \Real^{\dimInd \times \dimIndB}, \simMatrixB \in \Real^{\dimIndB \times \dataNum}}
    \bigl\|\tangentVector - \tangentVector,\simMatrix \simMatrixB\bigr\|_2^2,
    \quad \text{s.t.} \quad
    \simVector^{(\sumIndB)} \in \Simplex^{\dataNum},\;
    \simVectorB^{(\sumIndA)} \in \Simplex^{\dimIndB}.
    \label{eq:l2-abstract-AA}
    \end{equation}
    One key distinction is that the diffeomorphism $\diffeoA_{\networkParams^*}$ has a constant Jacobian determinant, in contrast to the full diffeomorphism underlying Riemannian archetypal mapping, which does not satisfy this property.
    
    Despite this structural difference, the formulations \eqref{eq:abstract-AA} and \eqref{eq:l2-abstract-AA} produce noticeably different results on real data. In practice, however, the former yields more interpretable archetypes, motivating our choice to adopt it. A more detailed investigation of these differences is left for future work.
\end{remark}

\paragraph{Step 3: finding radial functions}
For each class, the remaining task is to define an appropriate radial function $\rho_{\sumIndB}$. We propose to construct this function by assuming that the branch associated with $\rho_{\sumIndB}$ forms a convex set that tightly encloses the data while containing the origin. Given such a convex set $\StarBody \subset \Real^{\dimInd}$ -- containing the origin -- we can compute its radial function via \cite{leong2025optimal}
\begin{equation}
\radial_{\StarBody}(\sVector) := \sup \{ t > 0 \,\mid\, t \cdot \sVector \in \StarBody \}.
\label{eq:star-radial}
\end{equation}
Although this construction may not be optimal with respect to log-likelihood \eqref{eq:star-log-likelihood-loss}, it yields closed-form expressions for the radial function for certain classes of convex sets and guarantees that all data points are assigned high likelihood.

In contrast to Steps 1 and 2, we can no longer rely on existing methods and must instead design a dedicated procedure to identify suitable convex sets and their corresponding radial functions.

\subsection{Learning ellipsoidal radial functions} 
We propose an ellipsoid-based construction of radial functions for each branch $ \radial_{\sumIndB} $. Specifically, we model each branch using ellipsoidal sets of the form
\begin{equation}
    \Ellipsoid_{\SpdMatrix}(\centroid) := \{\VectorB \in \Real^\dimInd \; \mid \; (\VectorB - \centroid)^\top \SpdMatrix^{-1} (\VectorB - \centroid) \leq 1\},
    \label{eq:ellipsoidal-set}
\end{equation}
where $ \SpdMatrix \in \Real^{\dimInd \times \dimInd} $ is symmetric positive definite and $ \centroid \in \Real^\dimInd $ denotes the center. To ensure that these sets admit a well-defined radial function, we require $ \mathbf{0} \in \operatorname{int}(\Ellipsoid_{\SpdMatrix}(\centroid)) $, which is equivalent to the condition $ \centroid^\top \SpdMatrix^{-1} \centroid < 1 $.

Rather than relying on a single ellipsoid, we instead consider the (soft) intersection of two ellipsoids. This leads to a radial function for each branch of the form
\begin{equation}
    \radial_{\sumIndB} (\sVector) := \operatorname{softmin} \bigl( \radial_{\SpdMatrix^{(\sumIndB)}_o, \centroid^{(\sumIndB)}}(\sVector), \; \radial_{\SpdMatrix^{(\sumIndB)}_c, \mathbf{0}}(\sVector) \bigr),
    \label{eq:softmin-ellipsoidal-radial}
\end{equation}
where each component is given by
\begin{equation}
    \radial_{\SpdMatrix,\centroid}(\sVector) := \radial_{\Ellipsoid_{\SpdMatrix}(\centroid)} (\sVector) = \sup \{ t > 0 \,\mid\, t \cdot \sVector \in \Ellipsoid_{\SpdMatrix}(\centroid) \}.
    \label{eq:ellipsoidal-radial}
\end{equation}

\begin{remark}
    The use of the softmin in \eqref{eq:softmin-ellipsoidal-radial} ensures differentiability. In contrast, the exact radial function of the intersection of two ellipsoids would involve the pointwise minimum instead.
\end{remark}

\paragraph{Ellipsoidal radial functions}
Before proceeding to the construction of suitable ellipsoids, we first note that \eqref{eq:ellipsoidal-radial} admits a more convenient closed-form expression.

\begin{proposition}
\label{prop:ellipsoidal-radial-functions}
    Let $\SpdMatrix\in \Real^{\dimInd\times \dimInd}$ be a symmetric positive definite matrix and let $\centroid \in \Real^\dimInd$ be a vector. Furthermore, assume that $\centroid^\top \SpdMatrix^{-1} \centroid < 1$.

    Then, the radial function \eqref{eq:ellipsoidal-radial} of the ellipsoid generated by $\SpdMatrix$ and $\centroid$ satisfies
    \begin{equation}
        \radial_{\SpdMatrix,\centroid}(\sVector) = \frac{\sVector^{\top}\SpdMatrix^{-1} \centroid + \sqrt{ \left(\sVector^{\top}\SpdMatrix^{-1} \centroid\right)^2 + (\sVector^{\top}\SpdMatrix^{-1} \sVector)\left(1 - \centroid^{\top}\SpdMatrix^{-1} \centroid\right)}} {\sVector^{\top}\SpdMatrix^{-1} \sVector} 
    \end{equation}
    In particular, for $\centroid = \mathbf 0$ the radial function reduces to
    \begin{equation}
        \radial_{\SpdMatrix,\centroid}(\sVector) =(\sVector^\top \SpdMatrix^{-1} \sVector)^{-\frac{1}{2}}
    \end{equation}
\end{proposition}

\begin{remark}
    As a quick sanity check of the above result, it is worth highlighting that the case of $\centroid = \mathbf{0}$ yields
    \begin{equation}
        \radial_{\SpdMatrix,\centroid}\Bigl( \frac{\diffeoA_{\networkParams^*}(\Vector)}{\|\diffeoA_{\networkParams^*}(\Vector)\|_2} \Bigr)^{-2}\|\diffeoA_{\networkParams^*}(\Vector)\|_2^2 = \Bigl(\frac{\diffeoA_{\networkParams^*}(\Vector)^\top  \SpdMatrix^{-1}\diffeoA_{\networkParams^*}(\Vector)}{\|\diffeoA_{\networkParams^*}(\Vector)\|_2^2} \Bigr) \|\diffeoA_{\networkParams^*}(\Vector)\|_2^2 = \diffeoA_{\networkParams^*}(\Vector)^\top  \SpdMatrix^{-1}\diffeoA_{\networkParams^*}(\Vector).
    \end{equation}
    In other words, the distribution $\density_{\diffeoA,\radial_{\SpdMatrix,\centroid}}$ with $\centroid = \mathbf{0}$ reduces to a deformed Gaussian with latent distribution $\mathcal{N}(\mathbf{0}, \SpdMatrix)$.
\end{remark}

\paragraph{Data enclosing ellipsoids}
Selecting an ellipsoid of the form \eqref{eq:ellipsoidal-set} amounts to deciding how tightly the latent data $ \VectorB^{(1)}, \ldots, \VectorB^{(\dataNum)} $ associated with a given branch\footnote{Strictly speaking, one should write $ \tilde{\VectorB}^{(1)}, \ldots, \tilde{\VectorB}^{(\dataNumB)} \in \{\VectorB^{(\sumIndA)} \,\mid\, c_{\networkParams^*}^{(\sumIndA)} = \sumIndB \} $, but we omit this for notational simplicity.} should be enclosed. In the presence of outliers, enclosing every data point may be undesirable. Instead, we aim to construct an off-centered ellipsoid with a meaningful center $ \centroid $ and a symmetric positive definite matrix $ \SpdMatrix_o $, such that the origin is guaranteed to lie in its interior while the data are captured in an average sense. These requirements can be satisfied through a careful construction.

\begin{proposition}
\label{prop:off-centered-ellipsoidal}
    Let $\VectorB^{(1)}, \ldots, \VectorB^{(\dataNum)} \in \Real^\dimInd$ be any $\dataNum$ vectors and let $\alpha >1$ and $\beta \in (0,\alpha)$ be positive real numbers. Furthermore, let $\centroid:= \frac{1}{\dataNum} \sum_{\sumIndA=1}^\dataNum \VectorB^{(\sumIndA)}$ be the mean of the vectors and let $\ProjMatrix_\centroid:= \frac{\centroid  \centroid^\top}{\|\centroid\|_2^2}\in \Real^{\dimInd\times\dimInd}$ be the projection matrix onto the subspace generated by $\centroid$. Finally, consider the singular value decomposition 
    \begin{equation}
        (\mathbf{I} - \ProjMatrix_\centroid) [\VectorB^{(1)}, \ldots, \VectorB^{(\dataNum)}] = \OrthMatrix \mathbf{S} \OrthMatrixB^\top, \quad \OrthMatrix\in \Real^{\dimInd\times (\dimInd-1)}, \mathbf{S} \in \Real^{(\dimInd-1)\times (\dimInd-1)}, \OrthMatrixB\in \Real^{\dataNum \times(\dimInd-1)},
    \end{equation}
    where we write $\varsigma_{\sumIndC} \geq 0$ for the singular values on the diagonal of the matrix $\mathbf{S}$.

    Then, the symmetric positive definite matrix $\SpdMatrix_o := \OrthMatrixC \SpdMatrixB_o \OrthMatrixC^\top$, defined through the orthogonal matrix
    \begin{equation}
        \OrthMatrixC := [\frac{1}{\|\centroid\|_2} \centroid, \OrthMatrix] = [\frac{1}{\|\centroid\|_2} \centroid, \OrthVector^{(1)}, \ldots \OrthVector^{(\dimInd-1)}] \in \Real^{\dimInd\times \dimInd}
    \end{equation}
    and the diagonal matrix
    \begin{equation}
        \SpdMatrixB_o := \operatorname{diag}(\lambda_1, \ldots, \lambda_\dimInd), \quad \lambda_\sumIndC := \begin{cases}
        \max \Bigl\{\frac{\dimInd}{\dataNum} \sum_{\sumIndA=1}^\dataNum \frac{(\centroid^\top (\VectorB^{(\sumIndA)} - \centroid))^2}{\|\centroid\|_2^2}, \alpha \Bigr\} & \sumIndC = 1, \\
        \quad \quad \max\Bigl\{\frac{\dimInd}{\dataNum} \varsigma_{\sumIndC-1}^2, \beta \Bigr\} & \sumIndC = 2, \ldots, \dimInd,
        \end{cases}
    \end{equation}
    satisfies 
    \begin{equation}
     \frac{1}{\dataNum} \sum_{\sumIndA=1}^\dataNum (\VectorB^{(\sumIndA)} - \centroid)^\top \SpdMatrix_o^{-1} (\VectorB^{(\sumIndA)} - \centroid) \leq 1, \quad \text{and} \quad \centroid^\top \SpdMatrix_o^{-1} \centroid < 1.
     \label{eq:prop-off-centered-ellipsoid}
\end{equation}
\end{proposition}



To motivate the inclusion of a second, centered ellipsoid in the radial function \eqref{eq:softmin-ellipsoidal-radial}, observe that geodesics under a radial function from the previously constructed ellipsoid can be numerically unstable, when used in isolation. In particular, if its center lies far from the origin, the resulting set can extend excessively outward, which is undesirable for interpolation in terms of geodesics swinging out undesirably. While this effect can be partially mitigated through a concave transformation, it is generally preferable to work with enclosing sets that remain as compact as possible. Consequently, our primary interest lies in the portion of the set that is close to the origin, while still capturing the data in an average sense and ensuring that the data mean is contained. This motivates the introduction of a second ellipsoid centered at the origin, which can be constructed in an analogous manner.

\begin{proposition}
\label{prop:centered-ellipsoidal}
    Let $\VectorB^{(1)}, \ldots, \VectorB^{(\dataNum)} \in \Real^\dimInd$ be any $\dataNum$ vectors and let $\alpha >1$ and $\beta \in (0,\alpha)$ be positive real numbers. Furthermore, let $\centroid:= \frac{1}{\dataNum} \sum_{\sumIndA=1}^\dataNum \VectorB^{(\sumIndA)}$ be the mean of the vectors and let $\ProjMatrix_\centroid:= \frac{\centroid  \centroid^\top}{\|\centroid\|_2^2}\in \Real^{\dimInd\times\dimInd}$ be the projection matrix onto the subspace generated by $\centroid$. Finally, consider the singular value decomposition 
    \begin{equation}
        (\mathbf{I} - \ProjMatrix_\centroid) [\VectorB^{(1)}, \ldots, \VectorB^{(\dataNum)}] = \OrthMatrix \mathbf{S} \OrthMatrixB^\top, \quad \OrthMatrix\in \Real^{\dimInd\times (\dimInd-1)}, \mathbf{S} \in \Real^{(\dimInd-1)\times (\dimInd-1)}, \OrthMatrixB\in \Real^{\dataNum \times(\dimInd-1)},
    \end{equation}
    where we write $\varsigma_{\sumIndC} \geq 0$ for the singular values on the diagonal of the matrix $\mathbf{S}$.

    Then, the symmetric positive definite matrix $\SpdMatrix_c := \OrthMatrixC \SpdMatrixB_c \OrthMatrixC^\top$, defined through the orthogonal matrix
    \begin{equation}
        \OrthMatrixC := [\frac{1}{\|\centroid\|_2} \centroid, \OrthMatrix] = [\frac{1}{\|\centroid\|_2} \centroid, \OrthVector^{(1)}, \ldots \OrthVector^{(\dimInd-1)}] \in \Real^{\dimInd\times \dimInd}
    \end{equation}
    and the diagonal matrix
    \begin{equation}
        \SpdMatrixB_c := \operatorname{diag}(\lambda_1, \ldots, \lambda_\dimInd), \quad \lambda_\sumIndC := \begin{cases}
        \max \Bigl\{\frac{\dimInd}{\dataNum} \sum_{\sumIndA=1}^\dataNum \frac{(\centroid^\top \VectorB^{(\sumIndA)})^2}{\|\centroid\|_2^2}, \alpha \Bigr\} & \sumIndC = 1, \\
        \quad \quad \max\Bigl\{\frac{\dimInd}{\dataNum} \varsigma_{\sumIndC-1}^2, \beta \Bigr\} & \sumIndC = 2, \ldots, \dimInd,
        \end{cases}
    \end{equation}
    satisfies 
    \begin{equation}
         \frac{1}{\dataNum} \sum_{\sumIndA=1}^\dataNum (\VectorB^{(\sumIndA)} )^\top \SpdMatrix_c^{-1} \VectorB^{(\sumIndA)}  \leq 1, \quad \text{and} \quad \centroid^\top \SpdMatrix_c^{-1} \centroid < 1.
         \label{eq:prop-centered-ellipsoid}
    \end{equation}
    
\end{proposition}


Returning to the example in \Cref{fig:tree-example}, we observe that this construction produces geodesics with a well-behaved shape -- along with its variant incorporating a strongly concave function $ \monotone(s) := \log(5s + 1) $ --, as illustrated in \Cref{fig:tree-geodesics}.

\begin{figure}[h!]
    \centering
        \begin{subfigure}{0.48\linewidth}
        \centering
        \includegraphics[width=\linewidth]{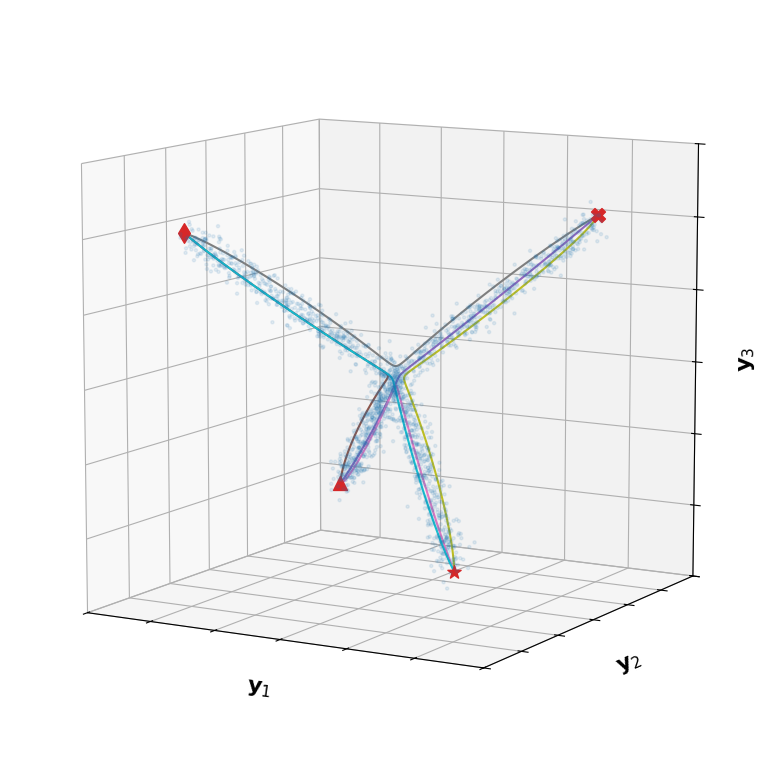}
        \caption{Pullback geodesics under $\diffeo := \diffeoC_{\monotone} \circ \diffeoB_{\radial}$.}
        \label{fig:river-geodesics-schematic-b}
    \end{subfigure}
    \hfill
    \begin{subfigure}{0.48\linewidth}
        \centering
        \includegraphics[width=\linewidth]{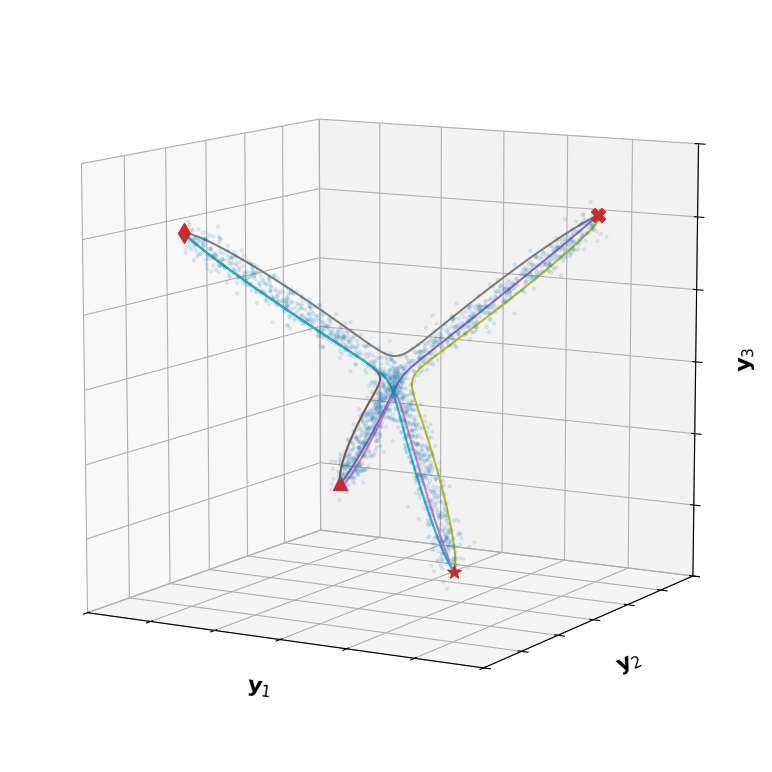}
        \caption{Pullback geodesics under $\diffeo := \diffeoB_{\radial}$.}
        \label{fig:river-geodesics-schematic-a}
    \end{subfigure}
    \caption{The proposed construction \eqref{eq:radial-outer-construction} combined with \eqref{eq:softmin-ellipsoidal-radial} yields geodesics between archetypes that remain close to the data support under the naive construction (right). These can be further drawn inward by incorporating a second diffeomorphism (left).}
    \label{fig:tree-geodesics}
\end{figure}

\section{Numerical Experiments}
\label{sec:numerics}

The proposed framework has already shown to performs well on the low-dimensional examples in previous sections, but a number of questions remain open. In particular, it is not yet clear (i) how effectively the three-step procedure in \Cref{sec:learning-stars} can learn deformed star densities (in both labeled and unlabeled settings), (ii) how well the star geometry captures high-dimensional data sets, for example in terms of the quality of interpolation, and (iii) whether our RAM-efficient optimization scheme requires further refinement, either for optimization itself or for downstream classification performance. Moreover, we would like to understand how iso-Riemannian geometry influences the various data-analysis tasks within this framework.

Since no ground truth is available in high dimensions, we require a setting where we can visually assess whether geodesics meaningfully traverse the data cloud, and which naturally admits both labeled and unlabeled variants. To this end, we consider the MNIST data set in two ways: first, we select a single digit class and apply our learning pipeline in the unlabeled setting, using one archetype per mode, and second, we use all digits together with their labels, grouping the data into ten class-specific subsets. In all experiments we further compose the learned star diffeomorphism $\diffeoB_{\radial} \circ \diffeoA_{\networkParams^*}$ with the additional diffeomorphism $\diffeoC_{\monotone}$, where $\monotone(s) := \log(10s + 1)$, in order to regularize the resulting geodesics. Consequently, all pullback geodesics are generated with respect to the composite map $\diffeo := \diffeoC_{\monotone} \circ \diffeoB_{\radial} \circ \diffeoA_{\networkParams^*}$.

For a fair comparison with existing manifold-learning approaches, we note that the only framework currently shown to learn meaningful (and computationally tractable) geodesics on MNIST is the method of \cite{diepeveen2025scorebased}, whose loss and parametrization were later refined in \cite{diepeveen2025manifold}. This construction reduces to normalizing flow training under a diffeomorphism with constant Jacobian determinant, which coincides with the first step in our star-learning procedure and therefore provides a natural baseline. To make the comparison as fair as possible, we adopt the same training routine for the normalizing flow component—using identical training parameters and network architectures in all settings—and we emphasize that these choices are not numerically optimized, but rather selected so that both the improvements over existing methods and the most pressing directions for further work are clearly visible. All training details are deferred to \Cref{app:numerics}.


\subsection{Single digit MNIST}
When considering the digit 3, \Cref{fig:mnist-three-stars} shows that step 1 of the three-step approach indeed yields latent stars. Here it should be noted that such projections can be misleading: the toy data set in \Cref{fig:tree-example} also yields cross-shaped two-dimensional views, yet the full higher-dimensional structure is not a clean multidimensional cross with symmetric radial branches. Similarly to the toy example, however, this ambiguity in the projections does not interfere with our automated procedure for extracting both archetypes and radial functions.

\begin{figure}[h!]
    \centering
    \includegraphics[width=\linewidth]{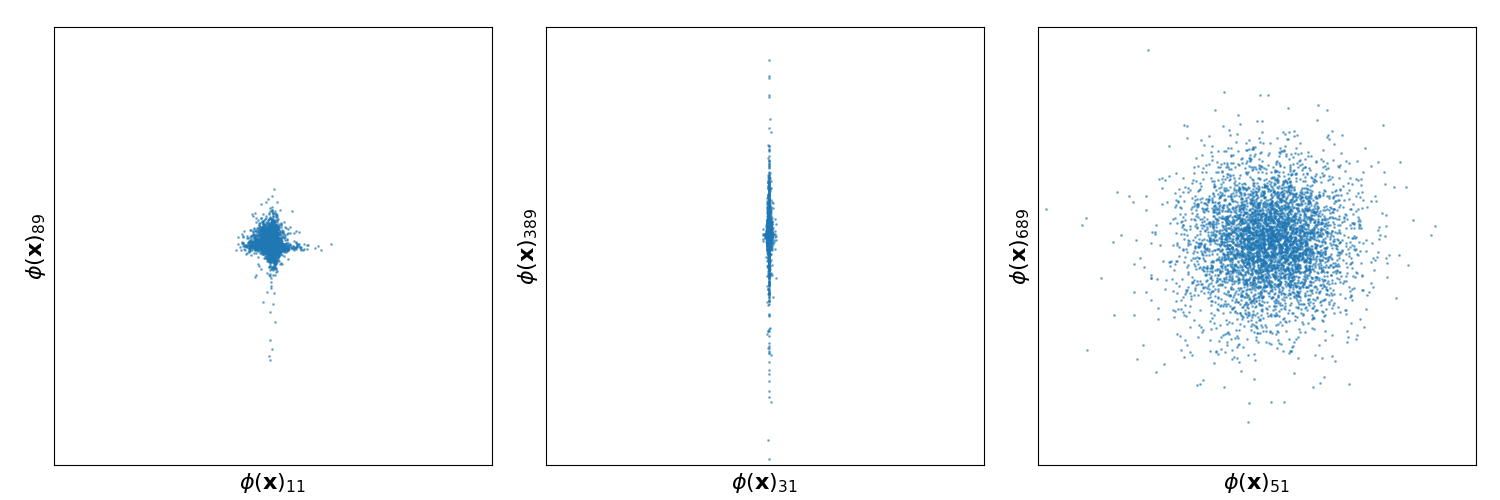}
    \caption{After training the normalizing flow with loss \eqref{eq:nf-loss} on the digit 3 from MNIST, the latent representations $\VectorB^{(\sumIndA)} := \diffeoA_{\networkParams^*}(\Vector^{(\sumIndA)})$ for $\sumIndA = 1,\ldots,\dataNum$ exhibit a pronounced star-like structure in several two-dimensional projections (left and middle), while appearing approximately isotropic in others (right).}
    \label{fig:mnist-three-stars}
\end{figure}

Advancing to step 2, selecting \(\dimIndB = 10\) archetypes produces the digits shown in \Cref{fig:mnist-threes-archetypes}, which provide a reasonable summary of the main ways in which threes appear in the data. Using the labels \eqref{eq:class-label} induced by these archetypes, we then estimate a radial function \eqref{eq:softmin-ellipsoidal-radial} for each class, thereby completing step 3: this yields an estimated deformed star distribution and, together with the archetypes, equips us with a Riemannian structure for subsequent data analysis.

\begin{figure}[h!]
    \centering
    \includegraphics[width=\linewidth]{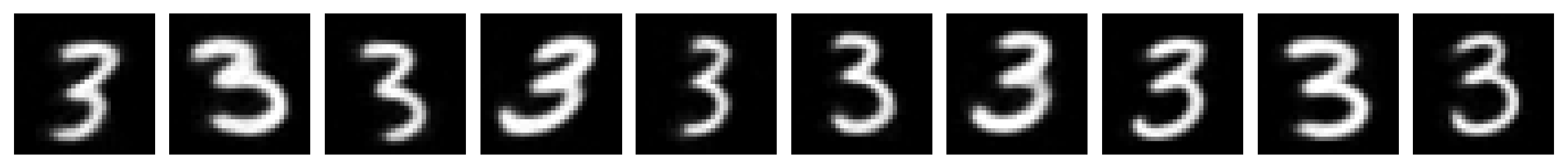}
    \caption{The archetypes $\emVector_{\networkParams^*}^{(\sumIndB)}$ \eqref{eq:latent-aa-aarchetypes} for $\sumIndB = 1, \ldots, 10$ obtained from Riemannian Archetypal analysis \eqref{eq:latent-aa} -- under the pullback geometry generated by $\diffeoA_{\networkParams^*}$ -- sketch a reasonable picture of the different types of threes in the data set.}
    \label{fig:mnist-threes-archetypes}
\end{figure}

Next, \Cref{fig:mnist-three-geodesics} shows that interpolation \eqref{eq:thm-geodesic-remetrized} under the reference diffeomorphism \(\diffeoA_{\networkParams^*}\) already improves substantially on linear interpolation, while the star-based geodesic associated with \(\diffeo\) yields paths that more closely follow the underlying data geometry, and the iso-corrected variant \eqref{eq:iso-geodesic} further restores a more interpretable notion of time along the trajectory.

\begin{figure}[h!]
  \centering
  \begin{minipage}{0.08\textwidth}
    \centering
    $\gamma^{\ell^2}_{\Vector, \VectorB} (t)$\\[1cm]
    $\gamma^{\diffeoA_{\networkParams^*}}_{\Vector, \VectorB} (t)$\\[1cm]
    $\gamma^{\diffeo}_{\Vector, \VectorB} (t)$\\[1cm]
    $\gamma^{\diffeo, \iso}_{\Vector, \VectorB} (t)$\\[1cm]
  \end{minipage}%
  \hfill
  \begin{minipage}{0.9\textwidth}
    \centering
    \includegraphics[width=\linewidth]{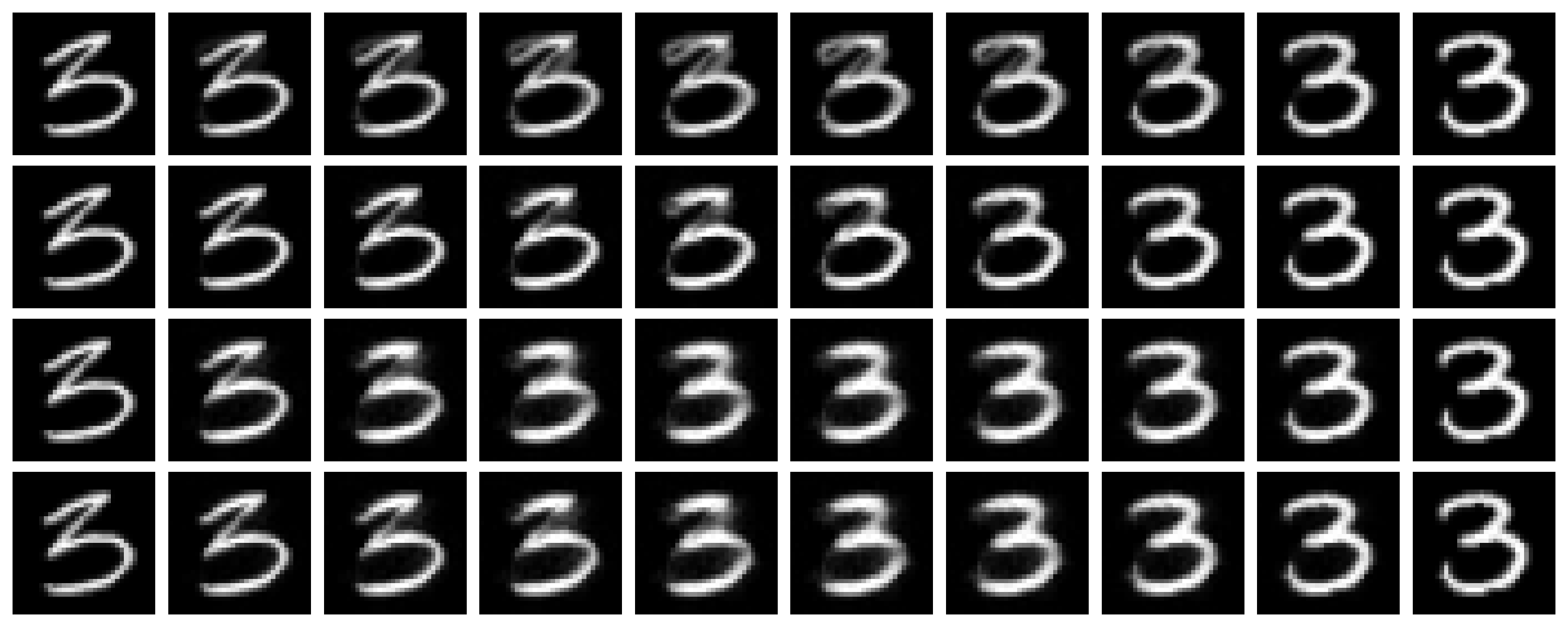}
  \end{minipage}
  \caption{When interpolating between two validation images of the digit three (left- and right-most columns), the linear path (first row) simply causes one image to fade out while the other fades in. In contrast, the baseline pullback geodesic (second row) produces a more nonlinear transformation, but the lower part of the three almost closes into a loop -- an artifact not supported by the data. The proposed geodesic (third row), which respects the learned star structure, first moves the left three toward a more “reference-like” three (closer to the star center) before continuing to the target image. The corresponding iso-geodesic (fourth row) further mitigates the naive impression that the trajectory spends an inordinate amount of time near this reference three.}
  \label{fig:mnist-three-geodesics}
\end{figure}

Finally, the Riemannian archetypal mapping (RAM) projections in \Cref{fig:mnist-threes-ram} highlight the importance of refining the relaxed RAM initialization and demonstrate that this procedure yields very reasonable “denoised” instances of the digit three.

\begin{figure}[h!]
  \centering
  \begin{minipage}{0.08\textwidth}
    \centering
    $\Vector_{\operatorname{val}}^{(\sumIndA)}$\\[1cm]
    $\relRAM^\diffeo(\Vector_{\operatorname{val}}^{(\sumIndA)})$\\[1cm]
    $\RAM^\diffeo(\Vector_{\operatorname{val}}^{(\sumIndA)})$\\[1cm]
  \end{minipage}%
  \hfill
  \begin{minipage}{0.9\textwidth}
    \centering
    \includegraphics[width=\linewidth]{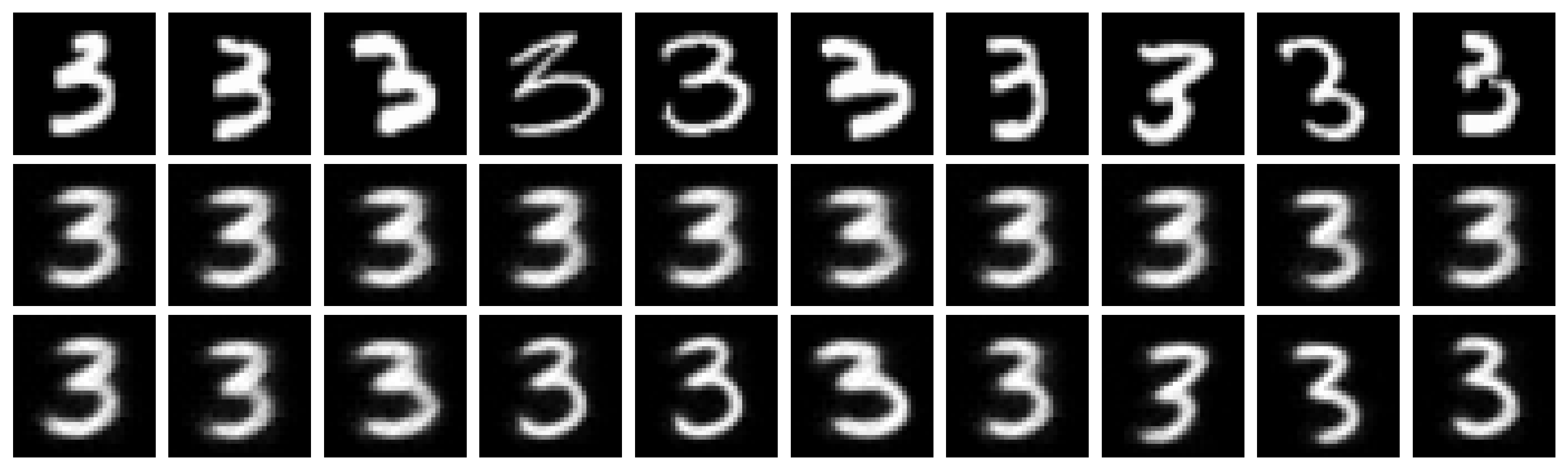}
  \end{minipage}
  \caption{Validation images (top row) are first mapped using the relaxed RAM (middle row), which then serves as an initialization for the RAM (bottom row). Without this additional refinement step, the relaxed RAM tends to produce highly similar projections.}
  \label{fig:mnist-threes-ram}
\end{figure}

\clearpage
\subsection{All digit MNIST}

Similarly to the digit 3 case, when we consider the full MNIST data set, \Cref{fig:mnist-stars} illustrates that step 1 of the three-step procedure again produces latent star structures. As before, however, such projections should be interpreted with care, for the reasons discussed above.

\begin{figure}[h!]
    \centering
    \includegraphics[width=\linewidth]{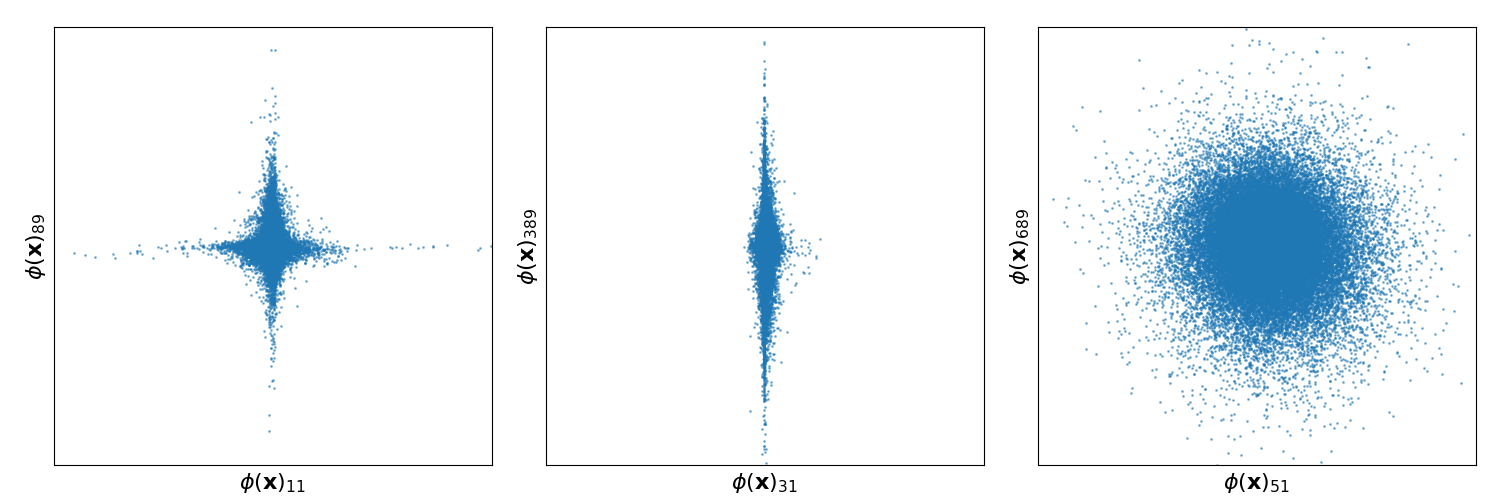}
    \caption{After training the normalizing flow with loss \eqref{eq:nf-loss} on MNIST, the latent representations $\VectorB^{(\sumIndA)} := \diffeoA_{\networkParams^*}(\Vector^{(\sumIndA)})$ for $\sumIndA = 1,\ldots,\dataNum$ exhibit a pronounced star-like structure in several two-dimensional projections (left and middle), while appearing approximately isotropic in others (right).}
    \label{fig:mnist-stars}
\end{figure}

Advancing to step 2, we select 10 archetypes per digit class, giving a total of \(\dimIndB = 100\) archetypes and yielding the representative digits shown in \Cref{fig:mnist-archetypes}, which form a reasonable summary of the data. Using the induced class labels, we then estimate a radial function \eqref{eq:softmin-ellipsoidal-radial} for each class, thereby completing step 3: this provides an estimated deformed star distribution and, together with the archetypes -- now multiple per class -- endows the data with a Riemannian structure suitable for subsequent analysis.

\begin{figure}[h!]
    \centering
    \includegraphics[width=\linewidth]{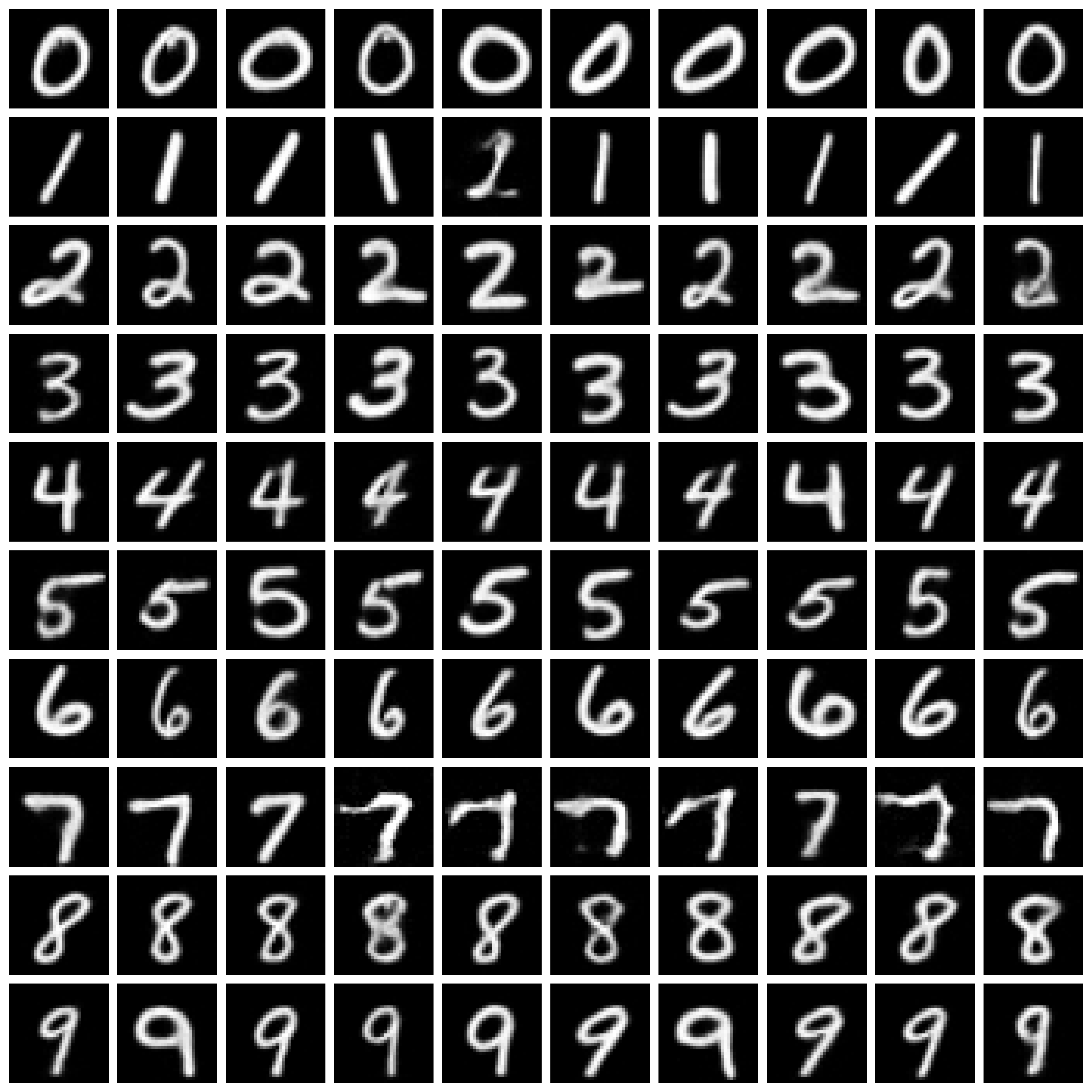}
    \caption{The archetypes $\emVector_{\networkParams^*}^{(\sumIndB)}$ \eqref{eq:latent-aa-aarchetypes} for $\sumIndB = 1, \ldots, 100$ obtained from Riemannian Archetypal analysis \eqref{eq:latent-aa} -- under the pullback geometry generated by $\diffeoA_{\networkParams^*}$ -- on each of the labeled subsets of the data set sketch a reasonable picture of the different types of digits in the data set.}
    \label{fig:mnist-archetypes}
\end{figure}

Next, \Cref{fig:mnist-geodesics} demonstrates that -- much as before -- interpolation \eqref{eq:thm-geodesic-remetrized} under the reference diffeomorphism \(\diffeoA_{\networkParams^*}\) already represents a substantial improvement over linear interpolation, while the star-based geodesic associated with \(\diffeo\) produces trajectories that more faithfully follow the data geometry, and the iso-corrected version \eqref{eq:iso-geodesic} further recovers a more interpretable notion of time along the path.

\begin{figure}[h!]
  \centering
  \begin{minipage}{0.08\textwidth}
    \centering
    $\gamma^{\ell^2}_{\Vector, \VectorB} (t)$\\[1cm]
    $\gamma^{\diffeoA_{\networkParams^*}}_{\Vector, \VectorB} (t)$\\[1cm]
    $\gamma^{\diffeo}_{\Vector, \VectorB} (t)$\\[1cm]
    $\gamma^{\diffeo, \iso}_{\Vector, \VectorB} (t)$\\[1cm]
  \end{minipage}%
  \hfill
  \begin{minipage}{0.9\textwidth}
    \centering
    \includegraphics[width=\linewidth]{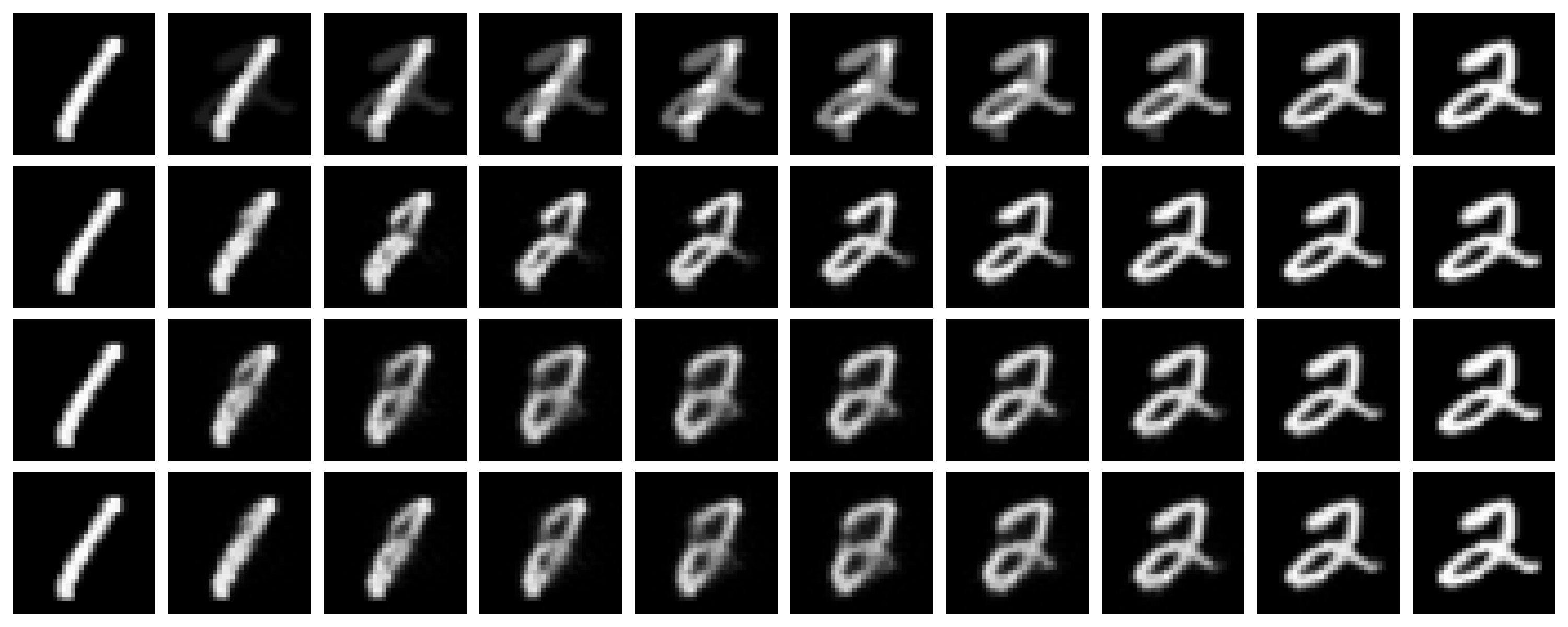}
  \end{minipage}\\
  \begin{minipage}{0.08\textwidth}
    \centering
    $\gamma^{\ell^2}_{\Vector, \VectorB} (t)$\\[1cm]
    $\gamma^{\diffeoA_{\networkParams^*}}_{\Vector, \VectorB} (t)$\\[1cm]
    $\gamma^{\diffeo}_{\Vector, \VectorB} (t)$\\[1cm]
    $\gamma^{\diffeo, \iso}_{\Vector, \VectorB} (t)$\\[1cm]
  \end{minipage}%
  \hfill
  \begin{minipage}{0.9\textwidth}
    \centering
    \includegraphics[width=\linewidth]{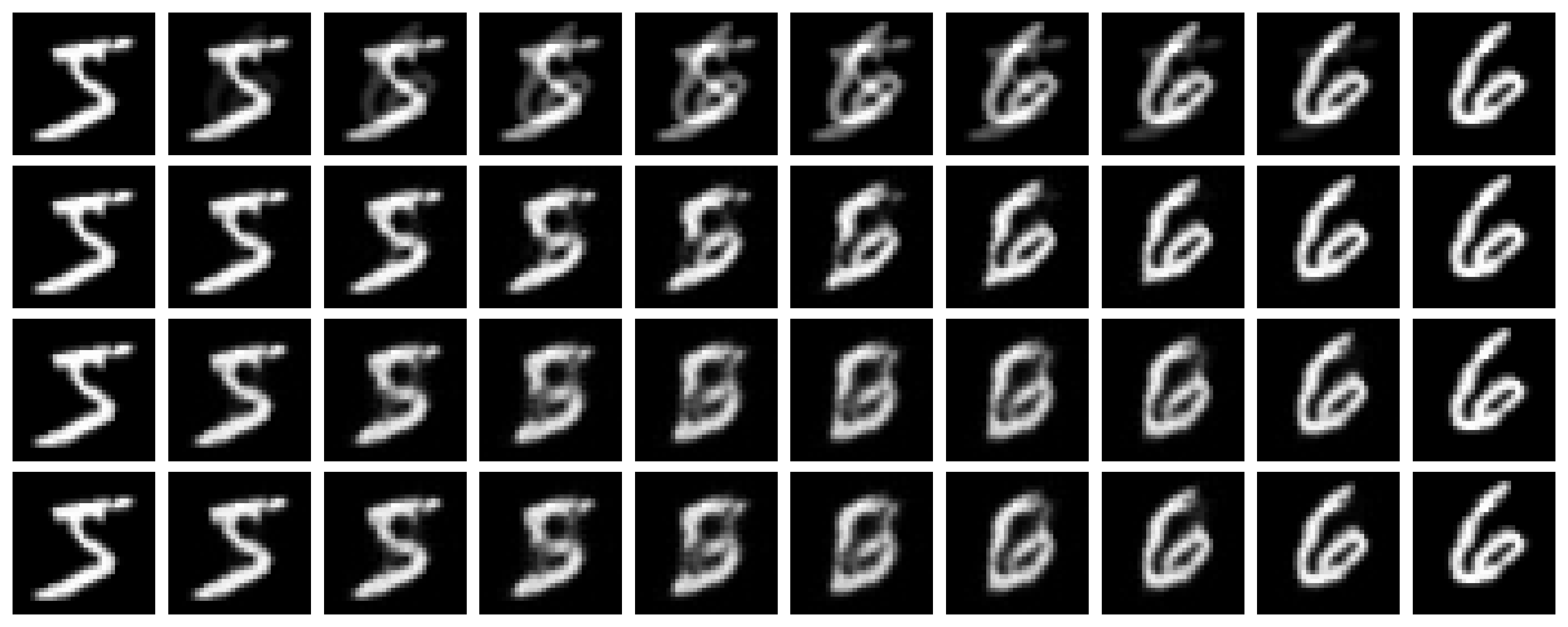}
  \end{minipage}
  \caption{When interpolating between different digits (left- and right-most columns), the linear path (first rows) simply causes one image to fade out while the other fades in. In contrast, the baseline pullback geodesic (second rows) produces a more nonlinear transformation, but it might move through images that are not always supported by the data. The proposed geodesic (third rows), which respects the learned star structure, first moves the left digit toward a more “reference-like” digit (closer to the star center) before continuing to the target digit. In particular, these reference-like digits seem to be eights. The corresponding iso-geodesic (fourth rows) further mitigates the naive impression that the trajectory spends an inordinate amount of time near this reference points.}
  \label{fig:mnist-geodesics}
\end{figure}

Then, the Riemannian archetypal mapping (RAM) projections in \Cref{fig:mnist-ram} expose current limitations of the method. In this setting, the algorithm fails to converge within a reasonable time, as the line search terminates once admissible step sizes fall below numerical precision. In other words, optimizing through the diffeomorphism -- already flagged in \Cref{rem:ram-optimization} as a potential bottleneck -- appears to have reached its limits, and further progress will realistically require a more developed iso-Riemannian optimization theory.

Finally, it is worth emphasizing what these limitations imply for classification. Starting from the vector \(\RAMweight^* \in \Real^{100}\) obtained by solving \eqref{eq:RAM-rewrite}, and its isometry-corrected counterpart \(\RAMweightB^* \in \Real^{100}\) from \eqref{eq:iso-log-weighted}, we aggregate the components associated with each digit class (for example, \(\RAMweight^*_1 + \cdots + \RAMweight^*_{10}\) gives the net weight for the digit 0 class). This allows us to examine more directly how the optimization issues impact downstream classification performance. As shown in \Cref{fig:mnist-classifications}, points that are visibly well-projected in \Cref{fig:mnist-ram}, i.e., digits 0, 1, 2, 3, 4, 6, and 7, are classified reliably -- often with additional improvement after the isometry correction -- whereas digits 5, 8, and 9 are classified less accurately. We stress once more that this shortcoming stems from optimization limits rather than from any intrinsic deficiency of the overall methodology.

\begin{figure}[h!]
  \centering
  \begin{minipage}{0.08\textwidth}
    \centering
    $\Vector_{\operatorname{val}}^{(\sumIndA)}$\\[1cm]
    $\relRAM^\diffeo(\Vector_{\operatorname{val}}^{(\sumIndA)})$\\[1cm]
    $\RAM^\diffeo(\Vector_{\operatorname{val}}^{(\sumIndA)})$\\[1cm]
  \end{minipage}%
  \hfill
  \begin{minipage}{0.9\textwidth}
    \centering
    \includegraphics[width=\linewidth]{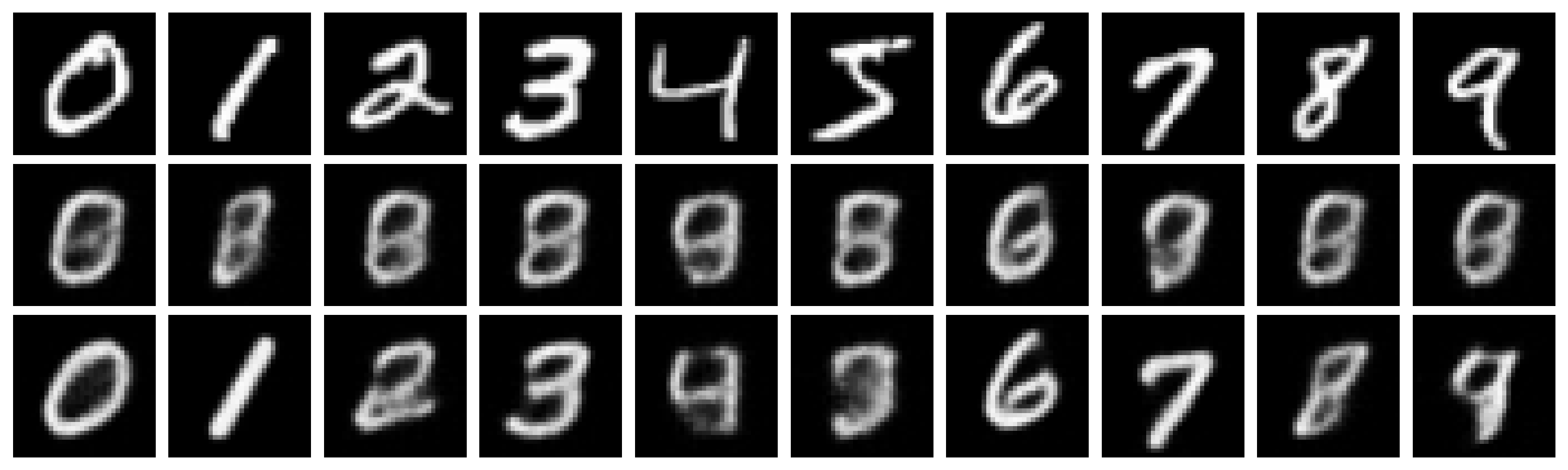}
  \end{minipage}
  \caption{Validation data points (top row) are first mapped using the relaxed RAM (middle row), which then serves as an initialization for the RAM (bottom row). Without this additional refinement step, the relaxed RAM tends to produce highly similar projections. Nevertheless, the RAM does not fully converge for all data points, because the line search step size eventually falls below numerical precision and the procedure can no longer make progress.}
  \label{fig:mnist-ram}
\end{figure}

\begin{figure}[h!]
\vspace{-1cm}
  \centering
  \begin{minipage}{0.48\textwidth}
    \includegraphics[width=\linewidth]{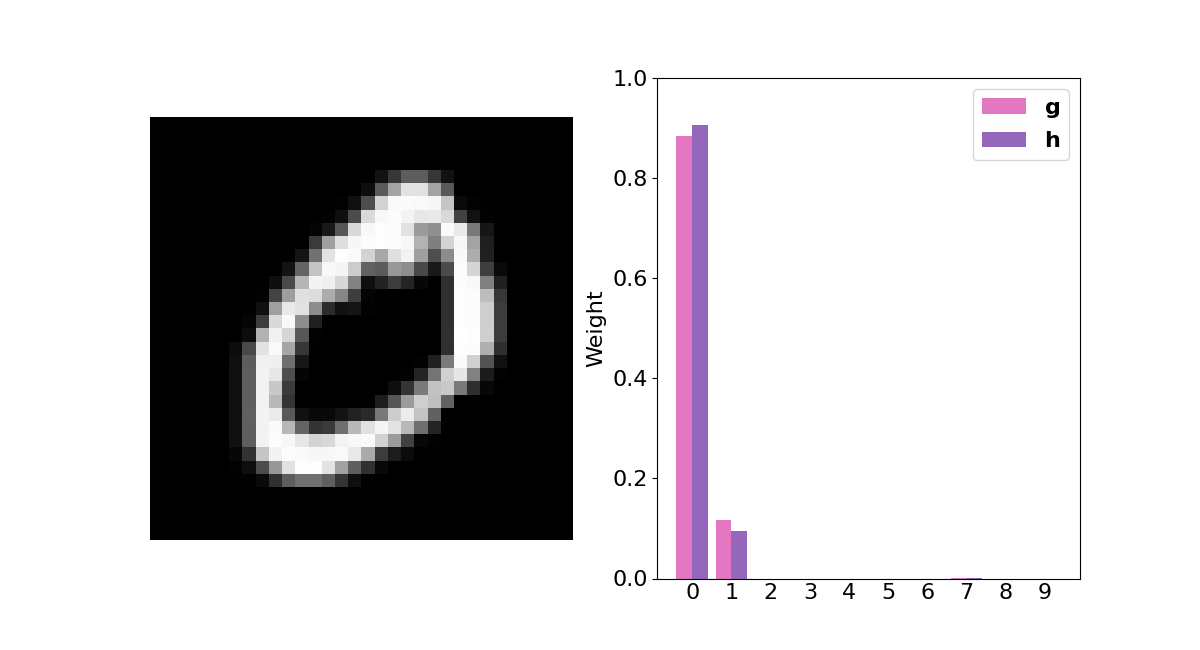}
  \end{minipage}%
  \hfill
  \begin{minipage}{0.48\textwidth}
    \includegraphics[width=\linewidth]{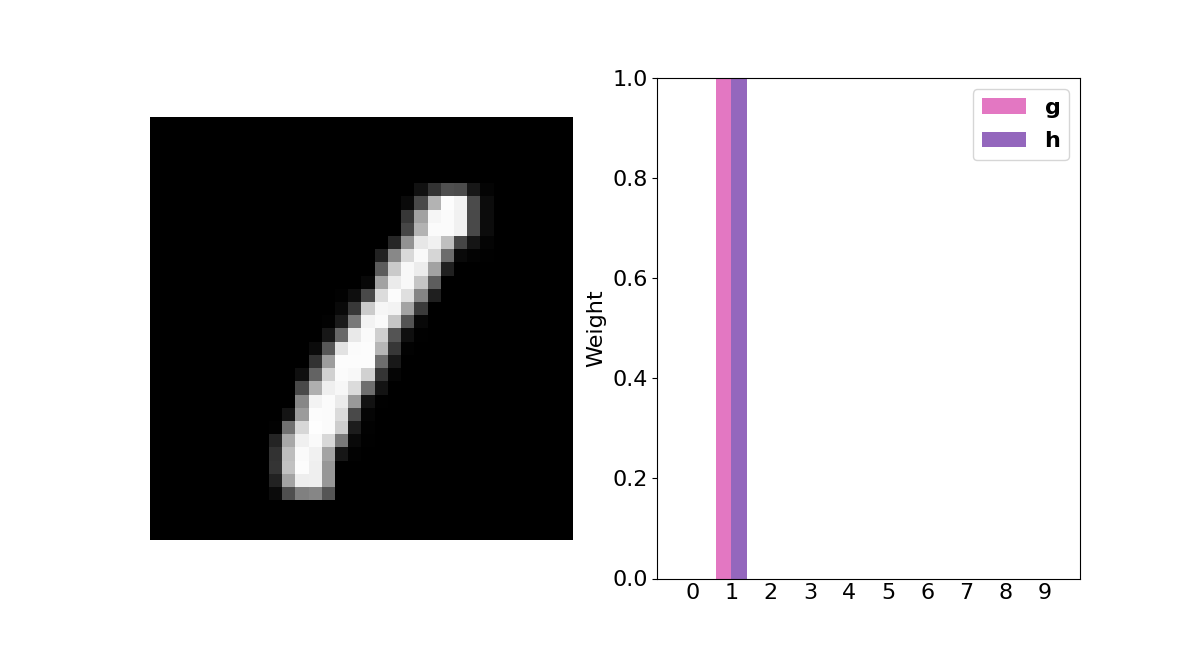}
  \end{minipage}%
  \\
  \begin{minipage}{0.48\textwidth}
    \includegraphics[width=\linewidth]{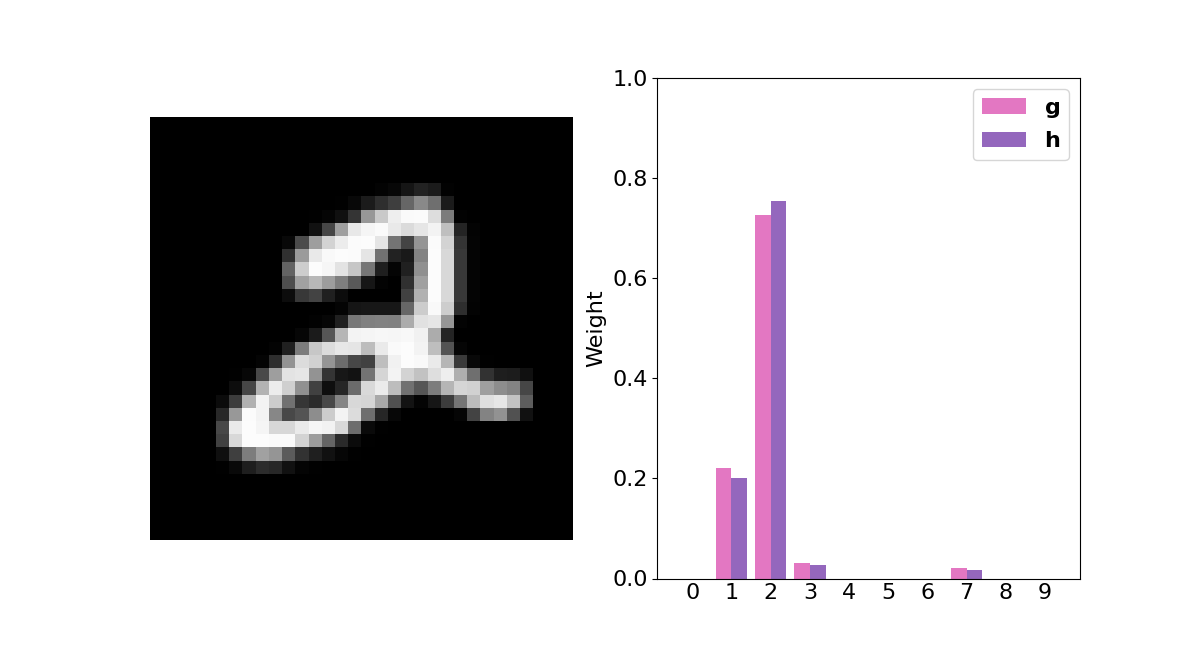}
  \end{minipage}%
  \hfill
  \begin{minipage}{0.48\textwidth}
    \includegraphics[width=\linewidth]{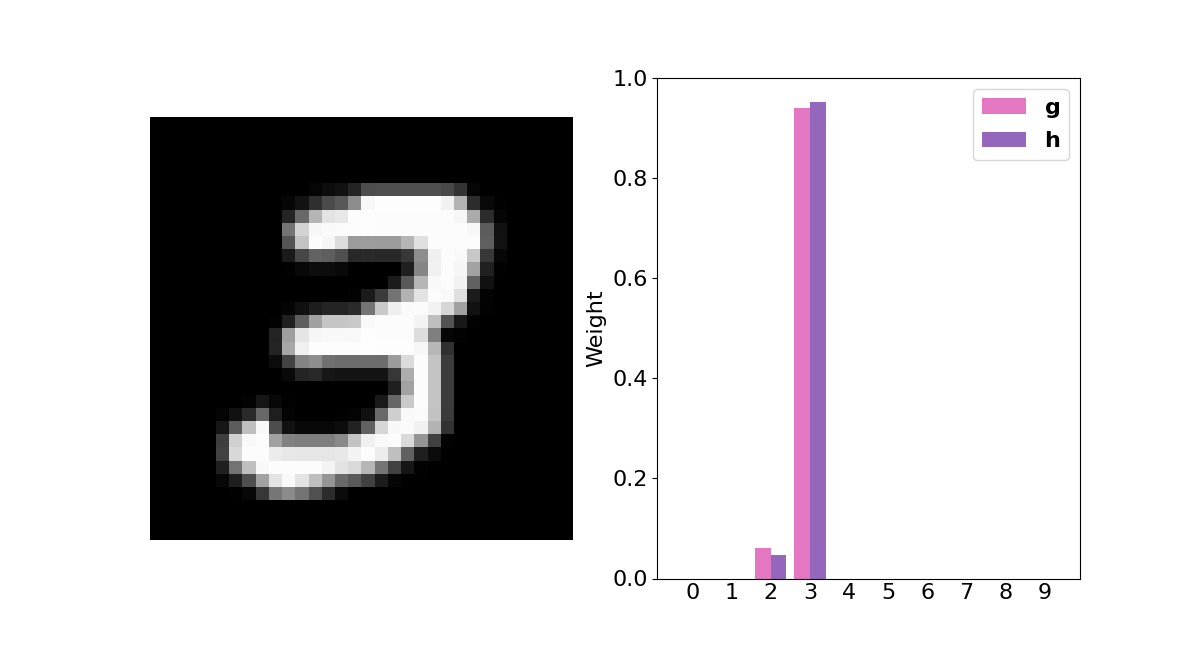}
  \end{minipage}%
  \\
  \begin{minipage}{0.48\textwidth}
    \includegraphics[width=\linewidth]{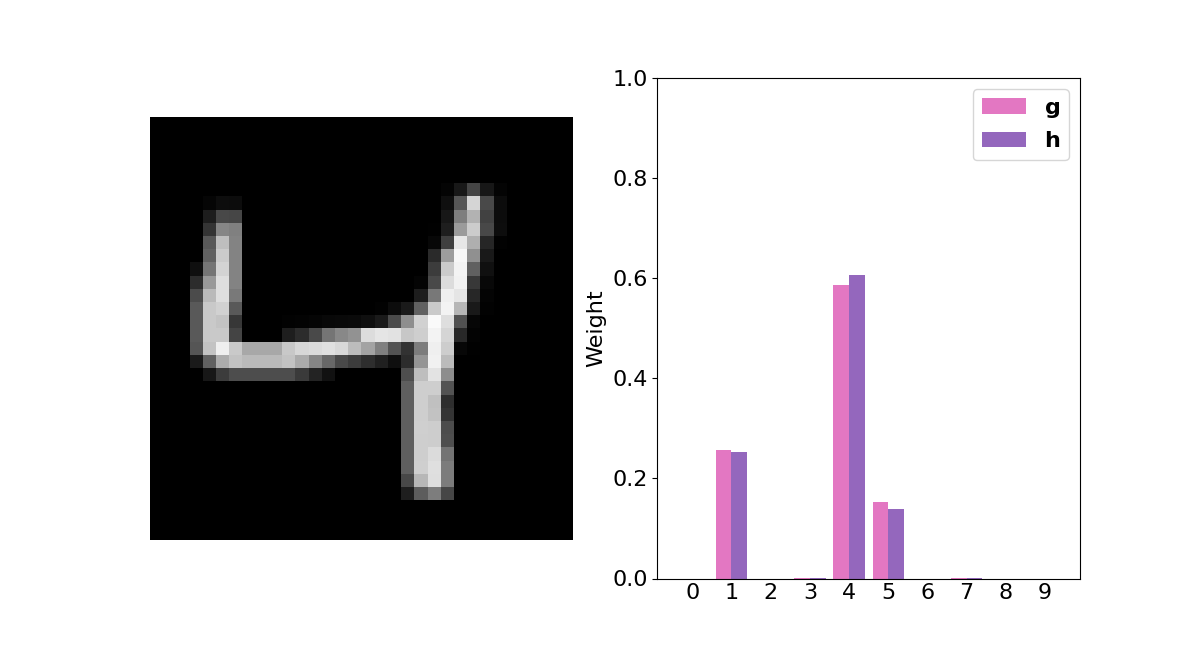}
  \end{minipage}%
  \hfill
  \begin{minipage}{0.48\textwidth}
    \includegraphics[width=\linewidth]{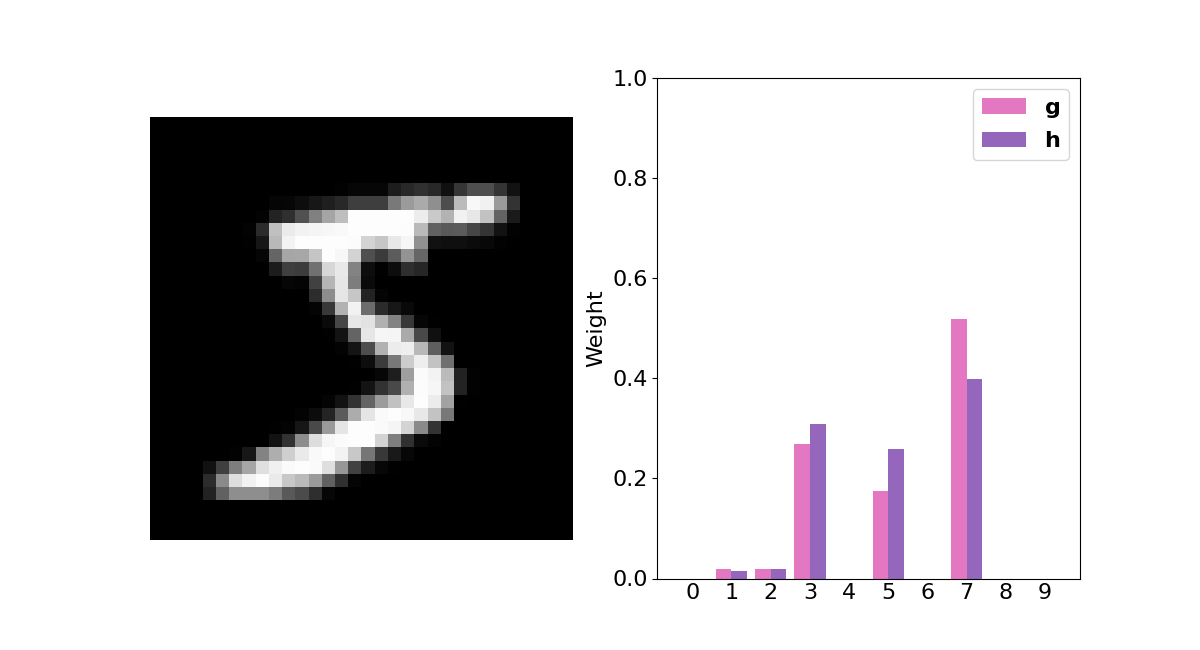}
  \end{minipage}%
  \\
  \begin{minipage}{0.48\textwidth}
    \includegraphics[width=\linewidth]{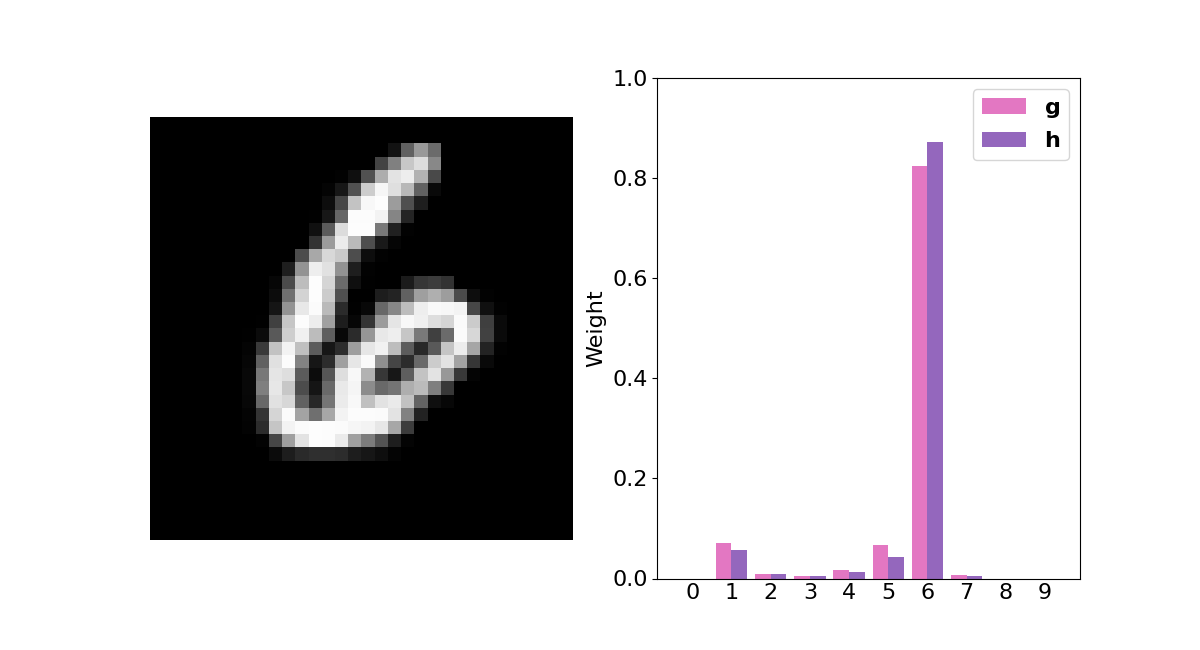}
  \end{minipage}%
  \hfill
  \begin{minipage}{0.48\textwidth}
    \includegraphics[width=\linewidth]{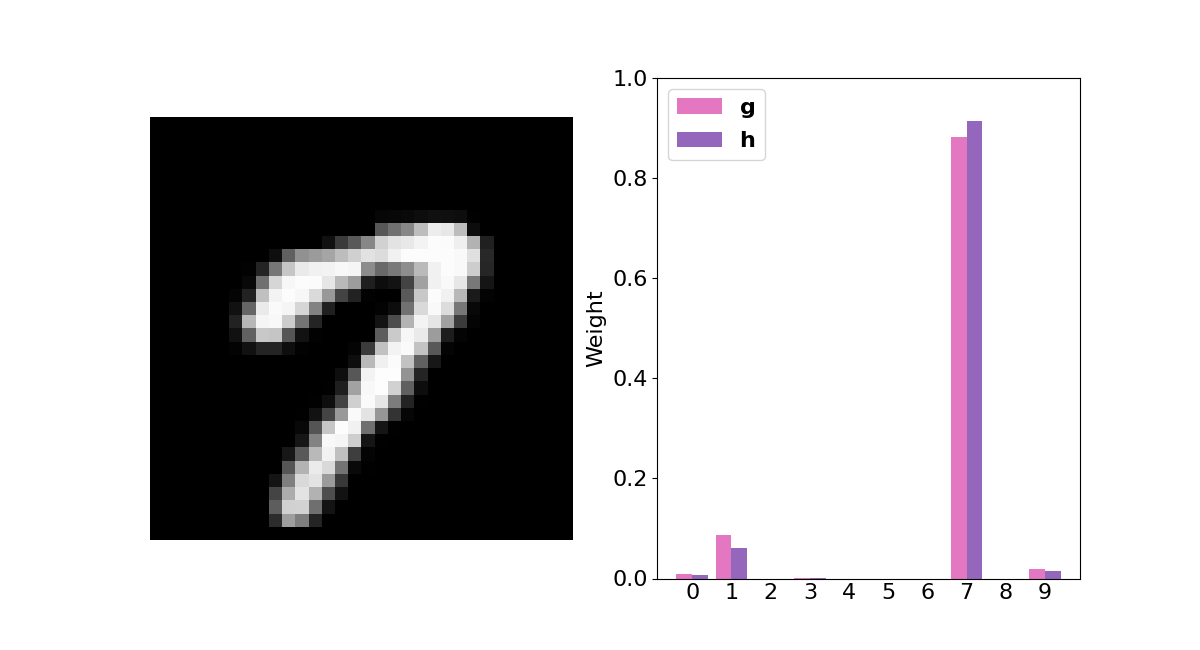}
  \end{minipage}%
  \\
  \begin{minipage}{0.48\textwidth}
    \includegraphics[width=\linewidth]{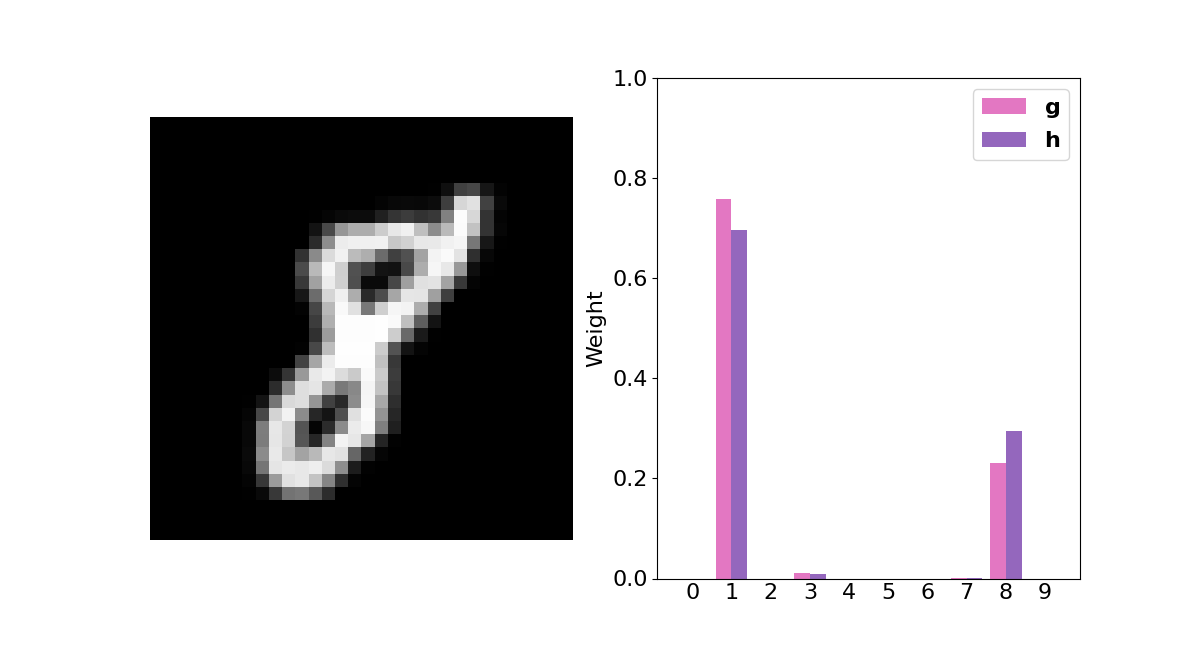}
  \end{minipage}%
  \hfill
  \begin{minipage}{0.48\textwidth}
    \includegraphics[width=\linewidth]{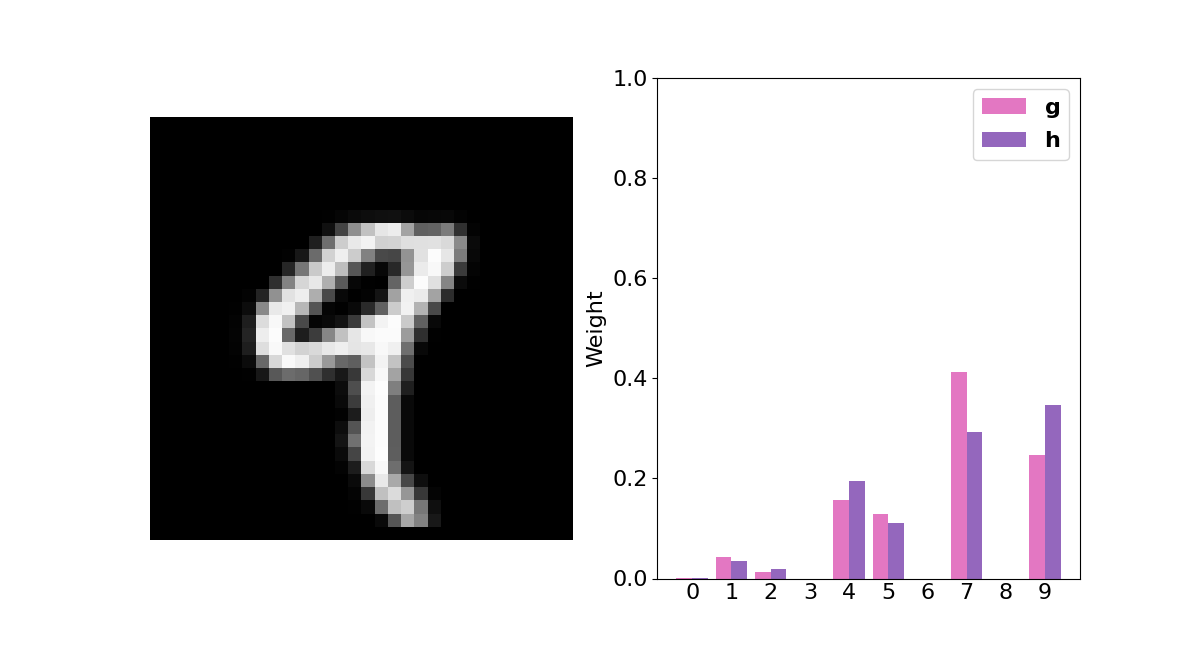}
  \end{minipage}%
  \caption{For each of the validation points in \Cref{fig:mnist-ram}, we look at how its mass is distributed over the archetypes, add up the contributions belonging to the same digit, and read off a predicted class -- before and after the isometry correction, i.e., considering $\RAMweight$ and $\RAMweightB$. For the digits 0, 1, 2, 3, 4, 6, and 7, where the RAM projection is visually clean, this simple RAM-based classifier performs well, often even better after isometry correction, while digits 5, 8, and 9 expose the current optimization limits of the approach rather than a failure of the underlying geometric model.}
  \label{fig:mnist-classifications}
\end{figure}

\clearpage
\section{Conclusions}
\label{sec:conclusions}

This work introduces Riemannian archetypal analysis as a geometrically grounded non-linear extension of classical archetypal analysis for real-valued data whose generating distribution can be approximated by a deformed star distribution. The framework preserves the central interpretability goal of archetypal methods while endowing the ambient space with a data-driven pullback Riemannian structure that supports fast interpolation, projection, denoising, and classification in a unified way. Our experiments show that this perspective is effective on both synthetic and real data, but they also make clear that optimization through the learned diffeomorphism remains the main computational bottleneck, especially for Riemannian archetypal mappings (RAMs) in higher-dimensional settings.

Several natural directions follow from this work. The most immediate is the development of constrained iso-Riemannian optimization methods tailored to RAM-type problems, since the current theory mainly covers first-order unconstrained settings and the numerical results indicate that substantially better optimization schemes should be possible. A second direction is to improve the learning of deformed star models beyond the present constructive scheme, ideally with stronger statistical guarantees and more expressive radial parametrizations. More broadly, the present work suggests that star-shaped geometric models may provide a useful foundation for other interpretable non-linear factorization methods, and that iso-Riemannian ideas may play an important role in making such models computationally viable.

\subsubsection*{Acknowledgments}
WD and DN were partially funded by  Defense Health Agency CDMRP TB240022 and NSF DMS 2408912

\bibliographystyle{plain}
\bibliography{bibliography}

\appendix
\crefalias{section}{appendix}
\crefalias{subsection}{appendix}

\section{Supplementary material to \cref{sec:star-geometry}}
\label{app:star-geometry}

\paragraph{Proof of \Cref{prop:deformed-star-prob-density}}
\begin{proof}

The statement follows from direct computation. To see this, first note that since $\diffeoA$ is a diffeomorphism, the substitution $\VectorB = \diffeoA(\Vector)$ is valid globally, and the
Jacobian factor satisfies $|\det D_\Vector \diffeoA| d\Vector =  d\VectorB$. So
\begin{equation}\label{eq:after-cov}
    \int_{\mathbb{R}^\dimInd} p_{\diffeoA,\radial}(\Vector) \, d\Vector
    =
    \frac{1}{2^{\frac{d}{2}-1}\,\Gamma\!\left(\frac{d}{2}\right)}
    \,
    \frac{1}{
        \displaystyle\int_{\Sphere^{d-1}} \radial(\sVector)^\dimInd \, d\sigma(\sVector)
    }
    \int_{\mathbb{R}^\dimInd}
    \exp\!\left(
        -\frac{1}{2}\,
        \radial\!\left(\frac{\VectorB}{\|\VectorB\|_2}\right)^{-2}
        \|\VectorB\|_2^2
    \right)
    d\VectorB.
\end{equation}

Next, we can write $\VectorB = r\sVector$ with $r \in (0, \infty)$ and $\sVector \in \Sphere^{d-1}$.
Under polar coordinates the Lebesgue measure decomposes as
\[
    d\VectorB = r^{d-1} \, dr \, d\sigma(\sVector).
\]
Since $\|\VectorB\|_2 = r$ and $\VectorB / \|\VectorB\|_2 = \sVector$, the inner integral in \eqref{eq:after-cov} becomes
\begin{equation}\label{eq:polar}
    \int_{\mathbb{R}^\dimInd}
    \exp\!\left(
        -\frac{1}{2}\,\radial\!\left(\frac{\VectorB}{\|\VectorB\|_2}\right)^{-2}\|\VectorB\|_2^2
    \right) d\VectorB
    =
    \int_{\Sphere^{d-1}}
    \int_0^{\infty}
    \exp\!\left(
        -\frac{1}{2}\,\radial(\sVector)^{-2}\, r^2
    \right)
    r^{d-1} \, dr \, d\sigma(\sVector).
\end{equation}

Then, defining $a(\sVector) := \tfrac{1}{2}\,\radial(\sVector)^{-2} > 0$, we can
Use the standard Gamma-function identity
\[
    \int_0^{\infty} e^{-a r^2} r^{d-1} \, dr
    = \frac{1}{2}\, a^{-d/2}\, \Gamma\!\left(\frac{d}{2}\right),
    \qquad a > 0,
\]
to obtain
\begin{multline}
    \int_0^{\infty}
    \exp\!\left(-\frac{1}{2}\,\radial(\sVector)^{-2}\, r^2\right) r^{d-1} \, dr
    = \frac{1}{2}\,
       \left(\frac{1}{2}\,\radial(\sVector)^{-2}\right)^{-d/2}
       \Gamma\!\left(\frac{d}{2}\right)\\
    = \frac{1}{2}\, 2^{d/2}\, \radial(\sVector)^{d}\, \Gamma\!\left(\frac{d}{2}\right) 
    = 2^{\frac{d}{2}-1}\, \Gamma\!\left(\frac{d}{2}\right)\, \radial(\sVector)^{d}.
    \label{eq:polar-worked-out}
\end{multline}

Substituting \eqref{eq:polar-worked-out} back into \eqref{eq:polar} gives
\begin{equation}\label{eq:radial-done}
    \int_{\mathbb{R}^\dimInd}
    \exp\!\left(
        -\frac{1}{2}\,\radial\!\left(\frac{\VectorB}{\|\VectorB\|_2}\right)^{-2}\|\VectorB\|_2^2
    \right) d\VectorB
    =
    2^{\frac{d}{2}-1}\,\Gamma\!\left(\frac{d}{2}\right)
    \int_{\Sphere^{d-1}} \radial(\sVector)^\dimInd \, d\sigma(\sVector).
\end{equation}

Finally, inserting \eqref{eq:radial-done} into \eqref{eq:after-cov} gives
\begin{equation}
    \int_{\mathbb{R}^\dimInd} p_{\diffeoA,\radial}(\Vector) \, d\Vector
    =
    \frac{1}{2^{\frac{d}{2}-1}\,\Gamma\!\left(\frac{d}{2}\right)}
    \cdot
    \frac{
        2^{\frac{d}{2}-1}\,\Gamma\!\left(\frac{d}{2}\right)
        \displaystyle\int_{\Sphere^{d-1}} \radial(\sVector)^\dimInd \, d\sigma(\sVector)
    }{
        \displaystyle\int_{\Sphere^{d-1}} \radial(\sVector)^\dimInd \, d\sigma(\sVector)
    }
    = 1. \qedhere
\end{equation}
\end{proof}

\paragraph{Proof of \Cref{prop:rho-v-diffeos}}
\begin{proof}
(i)
Let $\Vector\in\Real^\dimInd$.
If $\Vector=\mathbf{0}$, the claim is immediate from the definitions.

If $\Vector\neq\mathbf{0}$, set $\VectorB := \diffeoB_\rho(\Vector)
= \rho\!\left(\frac{\Vector}{\|\Vector\|_2}\right)^{-1}\Vector$.
Then $\VectorB\neq\mathbf{0}$, and
\begin{equation}
\frac{\VectorB}{\|\VectorB\|_2}
= \frac{\rho\left(\frac{\Vector}{\|\Vector\|_2}\right)^{-1}\Vector}
       {\rho\left(\frac{\Vector}{\|\Vector\|_2}\right)^{-1}\|\Vector\|_2}
= \frac{\Vector}{\|\Vector\|_2}.
\end{equation}
Hence
\begin{equation}
\diffeoB_\rho^{-1}(\VectorB)
= \rho\!\left(\frac{\VectorB}{\|\VectorB\|_2}\right)\VectorB
= \rho\!\left(\frac{\Vector}{\|\Vector\|_2}\right)\,
  \rho\!\left(\frac{\Vector}{\|\Vector\|_2}\right)^{-1}\Vector
= \Vector.
\end{equation}
Thus $\diffeoB_\rho^{-1}\circ\diffeoB_\rho=\operatorname{Id}_{\Real^\dimInd}$.
The identity $\diffeoB_\rho\circ\diffeoB_\rho^{-1}=\operatorname{Id}_{\Real^\dimInd}$ is proved in exactly the same way, exchanging the roles of $\Vector$ and $\VectorB$.

(ii)
Let $\Vector\in\Real^\dimInd$. If $\Vector=\mathbf{0}$, the claim is immediate from the definitions.

If $\Vector\neq\mathbf{0}$, set $\VectorB := \diffeoC_\monotone(\Vector)
= \monotone(\|\Vector\|_2)\,\frac{\Vector}{\|\Vector\|_2}$.
Then $\|\VectorB\|_2=\monotone(\|\Vector\|_2) >0$ and
\begin{equation}
\frac{\VectorB}{\|\VectorB\|_2}
= \frac{\Vector}{\|\Vector\|_2}.
\end{equation}
Using the inverse function $\monotone^{-1}:\monotone(\Real_{>0})\to\Real_{>0}$, we have
\begin{equation}
\diffeoC_\monotone^{-1}(\VectorB)
= \monotone^{-1}(\|\VectorB\|_2)\,\frac{\VectorB}{\|\VectorB\|_2}
= \monotone^{-1}(\monotone(\|\Vector\|_2))\,\frac{\Vector}{\|\Vector\|_2}
= \|\Vector\|_2\,\frac{\Vector}{\|\Vector\|_2}
= \Vector.
\end{equation}
Thus $\diffeoC_\monotone^{-1}\circ\diffeoC_\monotone=\operatorname{Id}_{\Real^\dimInd}$. The converse composition $\diffeoC_\monotone\circ\diffeoC_\monotone^{-1}=\operatorname{Id}_{\Real^\dimInd}$
is proved in exactly the same way, exchanging the roles of $\Vector$ and $\VectorB$.
\end{proof}

\paragraph{Proof of \Cref{lem:strong-convexity}}
\begin{proof}
Fix $\Vector\neq \VectorB\in\Real^\dimInd$ and write
\begin{equation}
    f(t) := \|\diffeoC_\monotone^{-1}((1 - t)\diffeoC_\monotone(\Vector) + t \diffeoC_\monotone(\VectorB))\|_2^2.
\end{equation}
By the definition of $\diffeoC_\monotone$ and its inverse, for any $\Vector',\VectorB'\in\Real^\dimInd$ we have
\begin{equation}
    \|\diffeoC_\monotone(\Vector')\|_2 = \monotone(\|\Vector'\|_2),
    \qquad
    \|\diffeoC_\monotone^{-1}(\VectorB')\|_2 = \monotone^{-1}(\|\VectorB'\|_2),
\end{equation}
with the convention $\monotone(0)=0$ and $\monotone^{-1}(0)=0$ using $\lim_{s\to 0}\monotone(s)=0$ and strict monotonicity. Hence, for all $t$,
\begin{equation}
    f(t)
    = \|\diffeoC_\monotone^{-1}((1 - t)\diffeoC_\monotone(\Vector) + t \diffeoC_\monotone(\VectorB))\|_2^2
    = \bigl(\monotone^{-1}(\|(1 - t)\diffeoC_\monotone(\Vector) + t \diffeoC_\monotone(\VectorB)\|_2)\bigr)^2.
\end{equation}

Set
\begin{equation}
    r(t) := \|(1 - t)\diffeoC_\monotone(\Vector) + t \diffeoC_\monotone(\VectorB)\|_2,
    \qquad
    u(s) := \monotone^{-1}(s),
    \qquad
    h(s) := u(s)^2,
\end{equation}
so that $f(t) = h(r(t))$.
Since $\monotone$ is strictly increasing and concave, $u$ is strictly increasing and convex, with $u'\!>0$ and $u''\!\ge 0$ on $(0,\infty)$ by the inverse function theorem. 

We compute the second derivative of $f$ using the chain rule. On any interval where $(1 - t)\diffeoC_\monotone(\Vector) + t \diffeoC_\monotone(\VectorB)\neq\mathbf{0}$ for all $t$, $r$ is smooth, and
\begin{equation}
    f'(t) = h'(r(t))\,r'(t),
    \qquad
    f''(t) = h''(r(t))\,(r'(t))^2 + h'(r(t))\,r''(t).
\end{equation}
We now compute $h'$ and $h''$ in terms of $\monotone$. For $s>0$,
\begin{equation}
    h'(s) = 2\,u(s)\,u'(s),
    \qquad
    h''(s) = 2\bigl(u'(s)^2 + u(s)\,u''(s)\bigr),
\end{equation}
which immediately tells us that $h'(s) >0$ for all $s>0$.
By the inverse function theorem,
\begin{equation}
    u'(s) = \frac{1}{\monotone'\bigl(u(s)\bigr)}, 
    \qquad
    u''(s) = -\,\frac{\monotone''(u(s))}{\monotone'\bigl(u(s)\bigr)^3}. 
\end{equation}
Thus
\begin{equation}
    u'(s)^2 + u(s)\,u''(s)
    = \frac{1}{\monotone'\bigl(u(s)\bigr)^2}
      - \frac{u(s)\,\monotone''\bigl(u(s)\bigr)}{\monotone'\bigl(u(s)\bigr)^3}
    = \frac{\monotone'\bigl(u(s)\bigr) - u(s)\,\monotone''\bigl(u(s)\bigr)}
           {\monotone'\bigl(u(s)\bigr)^3}.
\end{equation}
Since $\monotone$ is strictly increasing, $\monotone'>0$, and since it is concave, $\monotone''\le 0$. Therefore
\begin{equation}
    \monotone'\bigl(u(s)\bigr) - u(s)\,\monotone''\bigl(u(s)\bigr)
    \ge \monotone'\bigl(u(s)\bigr) > 0,
\end{equation}
so
\begin{equation}
    h''(s) = 2\bigl(u'(s)^2 + u(s)\,u''(s)\bigr) > 0
    \quad\text{for all }s>0.
\end{equation}

Next, observe 
\begin{equation}
    r(t)^2 = \|(1 - t)\diffeoC_\monotone(\Vector) + t \diffeoC_\monotone(\VectorB)\|_2^2 = \|t (\diffeoC_\monotone(\VectorB) - \diffeoC_\monotone(\Vector)) + \diffeoC_\monotone(\Vector)\|_2^2
\end{equation}
is a quadratic polynomial in $t$ with strictly positive leading coefficient, i.e., $r''(t)$ is a positive constant in that case, and $r'(t)$ is affine. 

Putting things together and going back to
\begin{equation}
    f''(t) = h''(r(t))\,(r'(t))^2 + h'(r(t))\,r''(t),
\end{equation}
we have $h''(r(t))>0$, $r''(t) > 0$, and $h'(r(t)) > 0$. Hence
\begin{equation}
    f''(t) \ge h'(r(t))\,r''(t) > 0.
\end{equation}

If $(1 - t)\diffeoC_\monotone(\Vector) + t \diffeoC_\monotone(\VectorB)=\mathbf{0}$ for some $t$, the limit $\lim_{s\to 0}\monotone(s)=0$ ensures that $f$ extends smoothly through that point with the same second-derivative bound.

Thus $t\mapsto \|\diffeoC_\monotone^{-1}((1 - t)\diffeoC_\monotone(\Vector) + t \diffeoC_\monotone(\VectorB))\|_2^2$ is strongly convex on $[0,1]$.
\end{proof}

\paragraph{Proof of \Cref{thm:geodesic-convexity-starflows}}

\begin{proof}
First, we will rewrite $-\log \density_{\diffeoA,\radial}(\gamma_{\Vector,\VectorB}^{\diffeo}(t))$ in a more convenient way.
From~\eqref{eq:starflow-density}, using that $|\det D_\Vector\diffeoA|$ is constant, we can write
\begin{equation}
    \density_{\diffeoA,\radial}(\gamma_{\Vector,\VectorB}^{\diffeo}(t))
    = \const \cdot
      \exp\!\left(
        -\frac{1}{2}\,
        \radial\!\left(
            \frac{\diffeoA(\gamma_{\Vector,\VectorB}^{\diffeo}(t))}{\|\diffeoA(\gamma_{\Vector,\VectorB}^{\diffeo}(t))\|_2}
        \right)^{-2}
        \|\diffeoA(\gamma_{\Vector,\VectorB}^{\diffeo}(t))\|_2^2
      \right),
\end{equation}
for some constant $\const>0$ independent of $t$.
Hence,
\begin{equation}
    -\log \density_{\diffeoA,\radial}(\gamma_{\Vector,\VectorB}^{\diffeo}(t))
    = \frac{1}{2}\,
      \radial\!\left(
          \frac{\diffeoA(\gamma_{\Vector,\VectorB}^{\diffeo}(t))}{\|\diffeoA(\gamma_{\Vector,\VectorB}^{\diffeo}(t))\|_2}
      \right)^{-2}
      \|\diffeoA(\gamma_{\Vector,\VectorB}^{\diffeo}(t))\|_2^2 + \const',
\end{equation}
for some constant $\const'$ independent of $t$. 
By definition of $\diffeoB_\radial$,
\begin{equation}
    \diffeoB_\radial(\diffeoA(\gamma_{\Vector,\VectorB}^{\diffeo}(t)))
    = \radial\!\left(
        \frac{\diffeoA(\gamma_{\Vector,\VectorB}^{\diffeo}(t))}{\|\diffeoA(\gamma_{\Vector,\VectorB}^{\diffeo}(t))\|_2}
    \right)^{-1}\diffeoA(\gamma_{\Vector,\VectorB}^{\diffeo}(t)),
\end{equation}
and therefore
\begin{equation}
    \|\diffeoB_\radial(\diffeoA(\gamma_{\Vector,\VectorB}^{\diffeo}(t)))\|_2^2
    = \radial\!\left(
        \frac{\diffeoA(\gamma_{\Vector,\VectorB}^{\diffeo}(t))}{\|\diffeoA(\gamma_{\Vector,\VectorB}^{\diffeo}(t))\|_2}
    \right)^{-2}
      \|\diffeoA(\gamma_{\Vector,\VectorB}^{\diffeo}(t))\|_2^2.
\end{equation}
So we can rewrite
\begin{equation}
    -\log \density_{\diffeoA,\radial}(\gamma_{\Vector,\VectorB}^{\diffeo}(t))
    = \frac{1}{2}\,\|\diffeoB_\radial(\diffeoA(\gamma_{\Vector,\VectorB}^{\diffeo}(t)))\|_2^2 + \const'.
\end{equation}

Next, expanding $\diffeoB_\radial(\diffeoA(\gamma_{\Vector,\VectorB}^{\diffeo}(t)))$ gives
\begin{equation}
    \diffeoB_\radial(\diffeoA(\gamma_{\Vector,\VectorB}^{\diffeo}(t)))
    \overset{\eqref{eq:thm-geodesic-remetrized}}{=} \diffeo^{-1}\bigl((1-t)\,\diffeo(\Vector) + t\,\diffeo(\VectorB) \bigr)
    \overset{\diffeo=\diffeoC_\monotone \circ \diffeoB_\radial \circ \diffeoA}{=}\diffeoC_\monotone^{-1}\bigl((1-t)\,\diffeoC_\monotone (\diffeoB_\radial(\diffeoA(\Vector))) +  t\, \diffeoC_\monotone(\diffeoB_\radial(\diffeoA(\VectorB)))
      \bigr),
\end{equation}
from which follows that
\begin{equation}
    \|\diffeoB_\radial (\diffeoA(\gamma_{\Vector,\VectorB}^{\diffeo}(t)))\|_2^2
    = \Bigl\|
        \diffeoC_\monotone^{-1}\bigl((1-t)\,\diffeoC_\monotone (\diffeoB_\radial(\diffeoA(\Vector))) +  t\, \diffeoC_\monotone(\diffeoB_\radial(\diffeoA(\VectorB)))
      \bigr) 
      \Bigr\|_2^2.
\end{equation}
The above equation is exactly of the form in \Cref{lem:strong-convexity}, with the two vectors
\begin{equation}
    \Vector' := \diffeoB_\radial(\diffeoA(\Vector)),
    \qquad
    \VectorB' := \diffeoB_\radial(\diffeoA(\VectorB)),
\end{equation}
and the mapping $\diffeoC_\monotone$ generated by the concave strictly increasing $\monotone$ satisfying $\lim_{s\to 0}\monotone(s)=0$ and $\lim_{s\to 0}\monotone'(s)>0$.

Hence, we must have that the function
\begin{equation}
    t \mapsto - \log \bigl( \density_{\diffeoA, \radial} (\gamma_{\mathbf{x}, \mathbf{y}}^{\diffeo}(t) ) \bigr)
\end{equation}
is strongly convex on $[0,1]$, because
\begin{equation}
    t \mapsto
    \Bigl\|
        \diffeoC_\monotone^{-1}\bigl(
            t\,\diffeoC_\monotone(\VectorB')
            + (1-t)\,\diffeoC_\monotone(\Vector')
        \bigr)
    \Bigr\|_2^2
\end{equation}
is strongly convex on $[0,1]$ by \Cref{lem:strong-convexity}, which proves the claim.
\end{proof}

\section{Supplementary material to \cref{sec:rams}}
\label{app:rams}

\subsection{Proof of \Cref{thm:ram-manifold-properties}}

\begin{proof}
\textbf{(i):}
    To prove that
    \begin{multline}
        \bar{\manifold}^\diffeo:=\bigl\{ \Vector \in \Real^\dimInd \; \mid \; \Vector = \argmin_{\VectorB\in \Real^\dimInd} \sum_{\sumIndB=1}^\dimIndB \RAMweight_\sumIndB \distance^{\diffeo}_{\Real^\dimInd} (\VectorB, \emVector^{(\sumIndB)})^2 \text{ for some } \RAMweight \in \Delta_\dimIndB\bigr\}\\
        = \{\Vector \in \Real^\dimInd \; \mid \; \Vector = \diffeo^{-1} \Bigl(\sum_{\sumIndB=1}^\dimIndB \RAMweight_\sumIndB  \diffeo(\emVector^{(\sumIndB)})\Bigr) \text{ for some } \RAMweight \in \Delta_\dimIndB\},
    \end{multline}
    we will show that every minimizer of $\Vector \mapsto \sum_{\sumIndB=1}^\dimIndB \RAMweight_\sumIndB \distance^{\diffeo}_{\Real^\dimInd} (\Vector, \emVector^{(\sumIndB)})^2$ is of the form $\diffeo^{-1} \Bigl(\sum_{\sumIndB=1}^\dimIndB \RAMweight_\sumIndB  \diffeo(\emVector^{(\sumIndB)})\Bigr)$.
    
    For that, we first note that the mapping $\Vector \mapsto \distance_{\Real^\dimInd}^\diffeo(\Vector, \VectorB)^2$ is strongly geodesically convex for every $\VectorB\in \Real^\dimInd$. Indeed, this directly follows from \cite[Lem.~3.5]{diepeveen2024pulling} after rewriting the function into the composition form, i.e.,  $\distance_{\Real^\dimInd}^\diffeo(\cdot, \VectorB)^2 \overset{\eqref{eq:thm-distance-remetrized}}{=} \|\diffeo(\cdot) - \diffeo(\VectorB)\|_2^2$, and realizing that $\Vector'\mapsto \|\Vector' - \diffeo(\VectorB)\|^2$ is strongly convex. Since strong geodesic convexity is closed under addition and multiplication with non-negative scalars, we conclude that the mapping
    \begin{equation}
        \Vector\mapsto \sum_{\sumIndB=1}^\dimIndB \RAMweight_\sumIndB \distance^{\diffeo}_{\Real^\dimInd} (\Vector, \emVector^{(\sumIndB)})^2
        \label{eq:barycentre-loss}
    \end{equation}
    is strongly geodesically convex for any $\RAMweight \in \Delta_\dimIndB$.
    
    Next, it is then easily checked that $\Vector^* := \diffeo^{-1}(\sum_{\sumIndB=1}^\dimIndB \RAMweight_\sumIndB  \diffeo(\emVector^{(\sumIndB)}))$ satisfies the first-order optimality conditions:
    \begin{multline}
        \operatorname{grad}  \sum_{\sumIndB=1}^\dimIndB \RAMweight_\sumIndB \distance^{\diffeo}_{\Real^\dimInd} (\cdot, \emVector^{(\sumIndB)})^2 \mid_{\Vector^*} = -  2 \sum_{\sumIndB=1}^\dimIndB \RAMweight_\sumIndB \log^\diffeo_{\Vector^*} (\emVector^{(\sumIndB)}) \\
        \overset{\eqref{eq:thm-log-remetrized}}{=} - 2 \sum_{\sumIndB=1}^\dimIndB \RAMweight_\sumIndB D_{\Vector^*}\diffeo^{-1}[\diffeo(\emVector^{(\sumIndB)}) - \diffeo(\Vector^*)] 
        =  - 2 D_{\Vector^*}\diffeo^{-1} [ \sum_{\sumIndB=1}^\dimIndB \RAMweight_\sumIndB  \diffeo(\emVector^{(\sumIndB)}) - \diffeo(\Vector^*)] \\
        = - 2 D_{\Vector^*}\diffeo^{-1} [\diffeo(\Vector^*) - \diffeo(\Vector^*) ]= \mathbf{0}.
    \end{multline}
    So we conclude that for any $\RAMweight \in \Delta_\dimIndB$ the unique minimizer of \eqref{eq:barycentre-loss} is $\diffeo^{-1}(\sum_{\sumIndB=1}^\dimIndB \RAMweight_\sumIndB  \diffeo(\emVector^{(\sumIndB)}))$, which yields the claim

\medskip

    \textbf{(ii):} To prove that $\bar{\manifold}^\diffeo$ is a geodesically convex set, we will show that for any $\Vector,\VectorB\in \bar{\manifold}^\diffeo$ we have that $\geodesic^\diffeo_{\Vector,\VectorB}\in \bar{\manifold}^\diffeo$.

First, we note that by (i), that there exist
$\RAMweight^{\Vector},\RAMweight^{\VectorB} \in \Delta_\dimIndB$ such that
\begin{equation}
    \Vector = \diffeo^{-1}\Bigl(\sum_{\sumIndB=1}^{\dimIndB} 
    \RAMweight^{\Vector}_{\sumIndB}\,\diffeo(\emVector^{(\sumIndB)})\Bigr),
    \qquad
    \VectorB = \diffeo^{-1}\Bigl(\sum_{\sumIndB=1}^{\dimIndB} 
    \RAMweight^{\VectorB}_{\sumIndB}\,\diffeo(\emVector^{(\sumIndB)})\Bigr).
    \label{eq:thm2-xy-rewrite}
\end{equation}
Next, for $t\in[0,1]$ define the interpolating weights
\begin{equation}
    \RAMweight^t := (1-t)\RAMweight^{\Vector} + t\RAMweight^{\VectorB} \in \Delta_{\dimIndB},
    \label{eq:thm2-gt}
\end{equation}
and the curve
\begin{equation}
    \geodesic(t) := 
    \diffeo^{-1}\Bigl(\sum_{\sumIndB=1}^{\dimIndB} \RAMweight^t_{\sumIndB}\,\diffeo(\emVector^{(\sumIndB)})\Bigr).
    \label{eq:thm2-gamma}
\end{equation}
Then, by construction, $\geodesic(0)=\Vector$, $\geodesic(1)=\VectorB$, and 
$\geodesic(t)\in\bar{\manifold}^\diffeo$ for all $t\in[0,1]$.

To prove the claim it suffices to show that $\geodesic^\diffeo_{\Vector,\VectorB}(t) = \geodesic(t)$ for all $t\in[0,1]$. This follows from direct computation:
\begin{multline}
    \geodesic^\diffeo_{\Vector,\VectorB}(t) \overset{\eqref{eq:thm-geodesic-remetrized}}{=} \diffeo^{-1}((1 - t)\diffeo(\Vector) + t \diffeo(\VectorB)) \overset{\eqref{eq:thm2-xy-rewrite}}{=} \diffeo^{-1} \Bigl( (1 - t)\sum_{\sumIndB=1}^{\dimIndB} 
    \RAMweight^{\Vector}_{\sumIndB}\,\diffeo(\emVector^{(\sumIndB)}) + t \sum_{\sumIndB=1}^{\dimIndB} 
    \RAMweight^{\VectorB}_{\sumIndB}\,\diffeo(\emVector^{(\sumIndB)}) \Bigr)\\
    = \diffeo^{-1} \Bigl( \sum_{\sumIndB=1}^{\dimIndB} 
    \bigl( (1 - t)\RAMweight^{\Vector}_{\sumIndB} + t \RAMweight^{\VectorB}_{\sumIndB}\bigr)\,\diffeo(\emVector^{(\sumIndB)}) \Bigr) \overset{\eqref{eq:thm2-gt}}{=} \diffeo^{-1} \Bigl( \sum_{\sumIndB=1}^{\dimIndB} 
    \RAMweight^t \,\diffeo(\emVector^{(\sumIndB)}) \Bigr) \overset{\eqref{eq:thm2-gamma}}{=} \geodesic(t) .
\end{multline}

\medskip

\textbf{(iii):} Finally, the claim that $\bar{\manifold}^\diffeo$ is a smooth manifold with corners, whose interior has  dimension
\begin{equation}
    \dim( \operatorname{int}(\bar{\manifold}^\diffeo)) = \operatorname{rank} ([ \diffeo(\emVector^{(1)}) - \diffeo(\emVector^{(\dimIndB)}), \ldots, \diffeo(\emVector^{(\dimIndB-1)}) - \diffeo(\emVector^{(\dimIndB)})]) \leq \dimIndB-1,
\end{equation}
follows directly, when seeing that the mapping $\diffeo':\bar{\manifold}^\diffeo \to \operatorname{conv}\{\diffeo(\emVector^{(1)}),\ldots,
    \diffeo(\emVector^{(\dimIndB)})\}$ between the constraint set and the convex hull of the embedded archetypes given by
    \begin{equation}
        \diffeo'(\Vector):= \diffeo(\Vector), \quad \Vector\in \bar{\manifold}^\diffeo
    \end{equation}
    is a diffeomorphism between the two sets, i.e., $\diffeo'$ is a chart. In other words, since the polytope $\operatorname{conv}\{\diffeo(\emVector^{(1)}),\ldots,
    \diffeo(\emVector^{(\dimIndB)})\}$ is a smooth manifold with corners, so is $\bar{\manifold}^\diffeo$ and
    \begin{multline}
        \dim( \operatorname{int}(\bar{\manifold}^\diffeo)) = \dim( \operatorname{int}(\operatorname{conv}\{\diffeo(\emVector^{(1)}),\ldots,
    \diffeo(\emVector^{(\dimIndB)})\}) \\
    = \operatorname{rank} ([ \diffeo(\emVector^{(1)}) - \diffeo(\emVector^{(\dimIndB)}), \ldots, \diffeo(\emVector^{(\dimIndB-1)}) - \diffeo(\emVector^{(\dimIndB)})]) \leq \dimIndB-1,
    \end{multline}
    which yields the claim.

\end{proof}

\section{Supplementary material to \cref{sec:learning-stars}}
\label{app:learning-stars}

\subsection{Proof of \Cref{prop:ellipsoidal-radial-functions}}

\begin{proof}
    We prove the statement by direct evaluation of \eqref{eq:ellipsoidal-radial}, i.e., we will solve for the supremum in
    \begin{equation}
        \radial_{\SpdMatrix,\centroid}(\sVector) = \sup \{ t > 0 \,\mid\, t \cdot \sVector \in \Ellipsoid_{\SpdMatrix}(\centroid) \}.
    \end{equation}
    First, note that the supremum lives on the boundary, i.e., we want to find
    \begin{equation}
        t \cdot \sVector \in \partial \Ellipsoid_{\SpdMatrix}(\centroid) \quad \Leftrightarrow \quad ( t \cdot \sVector - \centroid)^\top \SpdMatrix( t \cdot \sVector - \centroid) = 1 \quad \Leftrightarrow \quad t^2 \sVector^\top \SpdMatrix^{-1} \sVector - 2 t \sVector^\top \SpdMatrix^{-1} \centroid + \centroid^\top \SpdMatrix^{-1} \centroid - 1 = 0.
    \end{equation}
    
 The right-hand side is a quadratic equation, which has solutions

    \begin{equation}
        t = \frac{\sVector^{\top}\SpdMatrix^{-1} \centroid \pm \sqrt{ \left(\sVector^{\top}\SpdMatrix^{-1} \centroid\right)^2 + (\sVector^{\top}\SpdMatrix^{-1} \sVector)\left(1 - \centroid^{\top}\SpdMatrix^{-1} \centroid\right)}} {\sVector^{\top}\SpdMatrix^{-1} \sVector} .
    \end{equation}
    Since we want positive $t$, we conclude that
    \begin{equation}
        \radial_{\SpdMatrix,\centroid}(\sVector) = \frac{\sVector^{\top}\SpdMatrix^{-1} \centroid + \sqrt{ \left(\sVector^{\top}\SpdMatrix^{-1} \centroid\right)^2 + (\sVector^{\top}\SpdMatrix^{-1} \sVector)\left(1 - \centroid^{\top}\SpdMatrix^{-1} \centroid\right)}} {\sVector^{\top}\SpdMatrix^{-1} \sVector},
    \end{equation}
    as claimed. Then, for $\centroid = \mathbf{0}$, the above expression reduces to
    \begin{equation}
        \radial_{\SpdMatrix,\centroid}(\sVector) = (\sVector^\top \SpdMatrix^{-1} \sVector)^{-\frac{1}{2}}.
    \end{equation}
\end{proof}

\subsection{Proof of \Cref{prop:off-centered-ellipsoidal}}

\begin{proof}
    Both claims \eqref{eq:prop-off-centered-ellipsoid} follow from direct evaluation. 

    Starting with the first claim, we first rewrite
    \begin{multline}
        \frac{1}{\dataNum} \sum_{\sumIndA=1}^\dataNum (\VectorB^{(\sumIndA)} - \centroid)^\top \SpdMatrix_o^{-1} (\VectorB^{(\sumIndA)} - \centroid)  = \frac{1}{\dataNum}\trace \bigl( (\MatrixB - \centroid \mathbf{1}_\dataNum^\top)^\top \SpdMatrix_{o}^{-1} (\MatrixB - \centroid \mathbf{1}_\dataNum^\top) \bigr) \\
        = \frac{1}{\dataNum}\trace \bigl( \SpdMatrix_{o}^{-1} (\MatrixB - \centroid \mathbf{1}_\dataNum^\top) (\MatrixB - \centroid \mathbf{1}_\dataNum^\top)^\top  \bigr) \\
        = \frac{1}{\dataNum}\trace \bigl( \SpdMatrix_{o}^{-1} \OrthMatrixC \operatorname{diag} (\varsigma_0^2, \varsigma_1^2, \ldots, \varsigma_{\dimInd-1}^2) \OrthMatrixC^\top \bigr).
    \end{multline}
    where $\varsigma_0 := \sum_{\sumIndA=1}^\dataNum \frac{(\centroid^\top (\VectorB^{(\sumIndA)} - \centroid))^2}{\|\centroid\|_2^2}$. Further rewriting yields
    \begin{multline}
        \frac{1}{\dataNum}\trace \bigl( \SpdMatrix_{o}^{-1} \OrthMatrixC \operatorname{diag} (\varsigma_0^2, \varsigma_1^2, \ldots, \varsigma_{\dimInd-1}^2) \OrthMatrixC^\top \bigr)
        = \frac{1}{\dataNum}\trace \bigl( \OrthMatrixC  \SpdMatrixB_{o}^{-1} \OrthMatrixC^\top \OrthMatrixC \operatorname{diag} (\varsigma_0^2, \varsigma_1^2, \ldots, \varsigma_{\dimInd-1}^2) \OrthMatrixC^\top \bigr)\\
        = \frac{1}{\dataNum}\trace \bigl( \SpdMatrixB_{o}^{-1}  \operatorname{diag} (\varsigma_0^2, \varsigma_1^2, \ldots, \varsigma_{\dimInd-1}^2)  \bigr) = \frac{1}{\dataNum} \sum_{\sumIndC=1}^{\dimInd} \frac{\varsigma_{\sumIndC-1}^2}{\lambda_k},
    \end{multline}
    which gives us an expression we can now bound:
    \begin{equation}
        \frac{1}{\dataNum} \sum_{\sumIndC=1}^{\dimInd} \frac{\varsigma_{\sumIndC-1}^2}{\lambda_k} = \frac{1}{\dimInd}\sum_{\sumIndC=1}^{\dimInd} \frac{\frac{\dimInd}{\dataNum}\varsigma_{\sumIndC-1}^2}{\lambda_k} \leq \frac{1}{\dimInd} \sum_{\sumIndC=1}^{\dimInd} \frac{\lambda_k}{\lambda_k} = 1.
    \end{equation}
    So we conclude that 
    \begin{equation}
        \frac{1}{\dataNum} \sum_{\sumIndA=1}^\dataNum (\VectorB^{(\sumIndA)} - \centroid)^\top \SpdMatrix_o^{-1} (\VectorB^{(\sumIndA)} - \centroid) \leq 1,
    \end{equation}
    as claimed.

    For the second claim, we rewrite
    \begin{equation}
        \centroid^\top \SpdMatrix_o^{-1}\centroid = \centroid^\top \OrthMatrixC  \SpdMatrixB_{o}^{-1} \OrthMatrixC^\top \centroid = \frac{\|\frac{1}{\|\centroid\|_2^2}\centroid^\top \centroid\|_2^2}{\lambda_1}  = \frac{1}{\lambda_1} < 1,
    \end{equation}
    as claimed.
\end{proof}

\subsection{Proof of \Cref{prop:centered-ellipsoidal}}

\begin{proof}
    The proof is analogous to the proof of \Cref{prop:off-centered-ellipsoidal} above.
\end{proof}
\section{Supplementary material to \cref{sec:numerics}}
\label{app:numerics}

\subsection{Normalizing flow construction}

\paragraph{Architecture}

In all experiments, the diffeomorphism \( \diffeoA_{\networkParams} \) is a multi-scale convolutional image flow acting on images in \( \mathbb{R}^{1\times 32\times 32} \) with a standard normal base distribution. The flow has \(L\) levels; at each level it:
\begin{enumerate}
    \item Applies a fixed squeeze transform with factor 2 (halving height and width, quadrupling channels).
    \item Applies a learned image transform \(T_\ell\).
    \item Splits the channels in half, passing one half to the next level and storing the other for the inverse pass.
\end{enumerate}

The dimensions are chosen so that the first level sees \(4\) channels of size \(16\times 16\), and for each \(\ell = 2,\dots,L\),
\[
C_\ell = 2 C_{\ell-1}, \quad H_\ell = H_{\ell-1} / 2, \quad W_\ell = W_{\ell-1} / 2,
\]
i.e., spatial dimensions shrink by a factor of 2 while the number of channels doubles at every level.

Each level transform \(T_\ell\) is a composite of several invertible convolutional blocks. For an input with shape \((C,H,W)\), the transform consists of \(n_{\text{flows}}\) repeated “flow steps” followed by a final linear block. Every flow step applies:
\begin{itemize}
    \item channel-wise affine normalization (ActNorm-type layer),
    \item an invertible \(1\times 1\) convolution (channel mixing),
    \item two linear parity-based convolutions with kernel size \(k\) and alternating parities,
    \item a nonlinear parity-based coupling layer with \(k\times k\) convolutions and a hidden width of 64 channels.
\end{itemize}

After these \(n_{\text{flows}}\) steps the transform finishes with one more normalization layer, another invertible \(1\times 1\) convolution, and two additional parity-based linear convolutions. All components are exactly invertible, and the log-determinant contributions from squeezing, normalization, \(1\times 1\) convolutions, and coupling layers are accumulated in the usual way.

The forward pass of the multi-level flow thus proceeds by repeatedly squeezing, applying a level transform, and splitting channels until all levels have been applied. The inverse pass reverses these operations (recombining stored channel splits and “unsqueezing”) and is used both for likelihood evaluation and for sampling.

\paragraph{Training setup}

For both the single-digit and full MNIST experiments, images are resized to \(32\times 32\), normalized using the standard MNIST mean 0.1307 and variance 0.3081, and split into training and validation subsets with an 80/20 ratio. Mini-batches of size 128 are used throughout.

The flow is parameterized with:
\begin{itemize}
    \item kernel size \(k = 3\),
    \item hidden width 64 channels in the nonlinear coupling networks,
    \item \(n_{\text{flows}} = 3\) flow steps per scale,
    \item \(n_{\text{scales}} = 3\) multi-scale levels.
\end{itemize}

Training minimizes the negative log-likelihood of the data under the flow using Adam with learning rate \(10^{-3}\) for 50 epochs. 


\subsection{Radial construction}

In both settings, radial parameters (\( \alpha = 1.1 \) and \( \beta = 1\)) are shared across branches and govern the star-shaped radial deformation used in the deformed star distributions.

\end{document}